%% file: Paper_Final.tex
\documentclass[mnsc,nonblindrev]{informs3}

\OneAndAHalfSpacedXI 

\usepackage{natbib}
 \bibpunct[, ]{(}{)}{,}{a}{}{,}%
 \def\bibfont{\small}%
\TheoremsNumberedThrough     
\ECRepeatTheorems

\EquationsNumberedThrough    

\usepackage{algpseudocode}
\usepackage{xspace}
\algnewcommand{\LineComment}[1]{\State \(\triangleright\) {#1}}

\usepackage{subcaption}
\usepackage{booktabs}
\usepackage{blkarray, bigstrut} %

\usepackage{float}

\usepackage{enumitem}
\setlist[enumerate]{wide}

\input{commands}

\usepackage[en-US]{datetime2}

\newtheorem{fact}{Fact}

\usepackage{chngcntr}
\usepackage{apptools}
\AtAppendix{\counterwithin{lemma}{section}}
\AtAppendix{\counterwithin{proposition}{section}}
\AtAppendix{\counterwithin{remark}{section}}

\begin{document}
\RUNAUTHOR{Johari et al.}

\RUNTITLE{Matching While Learning}

\TITLE{Matching While Learning\footnote{To appear in {\it Operations Research}.}}

\ARTICLEAUTHORS{%
\AUTHOR{Ramesh Johari}
\AFF{Department of Management Science and Engineering\\ Stanford University\\ \EMAIL{rjohari@stanford.edu}} 
\AUTHOR{Vijay Kamble}
\AFF{Department of Information and Decision Sciences\\ The University of Illinois at Chicago\\ \EMAIL{kamble@uic.edu}}
\AUTHOR{Yash Kanoria}
\AFF{Decision, Risk, and Operations Division\\ Columbia Business School\\ \EMAIL{ykanoria@gsb.columbia.edu}}
\vspace{0.1in}\today
} 


\input{abstract}

\maketitle
\input{intro}

\input{relatedwork}

\input{model}
\input{DEEM}
\input{results}

\input{proofsketch}

\input{difficulty-with-many-skills}
\input{simulations}

\input{conclusion}

\bibliographystyle{informs2014}
\bibliography{banknap,references_from_dan,matchingmarket,information_disclosure,revenue_management,Fluid_and_mean_field,directed_search,platform,prior_johari,dynmech}

\newpage


\input{appendix}



\end{document}

%% file: commands.tex
\newcommand{\mexp}{{m_{\tiny {\rm xplr}}}}

\newcommand{\yd}{y^{D}}

\newcommand{\cH}{{\mathcal{H}}}
\newcommand{\cF}{{\mathcal{F}}}
\newcommand{\cG}{{\mathcal{G}}}

\newcommand{\cX}{\mathcal{X}}

\newcommand{\cB}{\mathcal{B}}

\newcommand{\cP}{{\mathcal{P}}}
\newcommand{\cPdiff}{{\mathcal{P}_{\tiny {\rm diff}}}}

\newcommand{\Weak}{\textup{Weak}}
\newcommand{\Str}{\textup{Str}}

\newcommand{\MAP}{\textup{MAP}}

\newcommand{\DEEMplus}{\ensuremath{\textup{DEEM}^{+}}\xspace}
\newcommand{\tDEEM}{{{\textup{\tiny DEEM}_N}}}

\newcommand{\eps}{\varepsilon}

\newcommand{\E}{\ensuremath{\mathbb{E}}}

\renewcommand{\epsilon}{\varepsilon}

\newcommand{\cJ}{\ensuremath{\mathcal{J}}}
\newcommand{\cI}{\ensuremath{\mathcal{I}}}
\newcommand{\cS}{\ensuremath{\mathcal{J}}}
\newcommand{\Sk}{\ensuremath{\mathcal{S}}}
\newcommand{\cC}{\ensuremath{\mathcal{I}}}
\newcommand{\cD}{\ensuremath{\mathcal{D}}}
\newcommand{\bV}{\ensuremath{\overline{V}}}
\newcommand{\ls}{\textup{limsup}}

\newcommand{\cE}{\ensuremath{\mathcal{E}}}

\newcommand{\ind}{{\mathbb{I}}}

\newcommand{\uT}{{\underline{T}}}
\newcommand{\Bern}{\textup{Bernoulli}}

\newcommand{\KL}{{\textup{KL}}}
\definecolor{darkjunglegreen}{rgb}{0.1, 0.14, 0.13}

\newcommand{\approxt}[1]{\ensuremath{\stackrel{#1}{\approx}}}
\newcommand{\piwho}{{\underline{\pi}}}
\newcommand{\unu}{{\underline{\nu}}}
\newcommand{\echide}[1]{}
\newcommand{\cJfull}{{\cJ^*_{\tiny {\rm full}}}}
\newcommand{\ty}{{\tilde{y}}}
\newcommand{\tx}{{\tilde{x}}}

\newcommand{\ub}{{\underline{b}}}
\newcommand{\ob}{{\overline{b}}}

%% file: abstract.tex
\ABSTRACT{
We consider the problem faced by a service platform that needs to match limited supply with demand but also to learn the attributes of new users in order to match them better in the future. We introduce a benchmark model with heterogeneous ``workers'' (demand)  and a limited supply of ``jobs'' that arrive over time. Job types are known to the platform, but worker types are unknown and must be learned by observing match outcomes. Workers depart after performing a certain number of jobs. The expected payoff from a match depends on the pair of types and the goal is to maximize the steady-state rate of accumulation of payoff.  Though we use terminology inspired by labor markets, our framework applies more broadly to platforms where a limited supply of heterogeneous products is matched to users
 over time.

Our main contribution is a complete characterization of the structure of the optimal policy in the limit that each worker performs many jobs.
The platform faces a trade-off for each worker between myopically maximizing payoffs (\emph{exploitation}) and learning the type of the worker (\emph{exploration}). This creates a multitude of multi-armed bandit problems, one  for each worker, coupled together by the constraint on availability of jobs of different types (\emph{capacity constraints}).  We find that the platform should estimate a shadow price for each job type, and use the payoffs adjusted by these prices, first, to determine its learning goals and then, for each worker, (i) to balance learning with payoffs during  the ``exploration phase,'' and  (ii) to myopically match after it has achieved its learning goals during the ``exploitation phase.''
}

\KEYWORDS{matching, learning, two-sided platform, multi-armed bandit, capacity constraints.}

%% file: intro.tex
\section{Introduction}
\label{sec:intro}
A wide range of online platforms serve as matchmakers between demand and supply; for example, online labor markets match workers to jobs (e.g., Upwork for remote work, Handy for housecleaning, Thumbtack and Taskrabbit for local tasks, etc.); e-commerce platforms match consumers to goods (e.g., eBay, Amazon); and online fashion retailers  match clients to clothing items (e.g., Rent The Runway, Stitch Fix).  These platforms are characterized by two salient features that motivate our work.  {\em First}, they have a {\em limited supply} available; for example, in online labor markets, the supply of jobs is limited; while in e-commerce and online fashion platforms, the supply of goods is limited.  {\em Second}, these platforms need to {\em learn} enough about their users (the demand side) to be able to match them to the right units of supply.  Our paper addresses this twin challenge of matching while learning.

The problem we address is a version of the {\em exploration-exploitation} trade-off: on the one hand, efficient operation involves making matches that generate the most value (``exploitation''); on the other hand, the platform must continuously learn about newly arriving participants, so that they can be efficiently matched (``exploration'').  
The task is complicated in our setting due to the fact that supply is limited: matching a unit of supply to one user renders it unavailable to other users, an externality that cannot be ignored, whether exploring or exploiting.
In this paper, we develop a structurally simple and nearly optimal approach to resolving the exploration-exploitation trade-off in settings with limited supply.

For convenience, the terminology in our model will be inspired by online labor markets: we call the demand side of the platform the {\em workers}, and the supply side of the platform the {\em jobs}.  Jobs are in limited supply in the platform.  Despite this specific terminology, our model should be viewed as a stylized abstraction of many platforms where supply is matched to users in the presence of limited inventory, e.g., via algorithmic recommendation or matching engines.  Examples include online commerce and fashion retail platforms mentioned above, and similar platforms in other industries.

In our model, workers and jobs arrive over discrete time.  Workers depart after $N$ periods, while jobs each take one period for a single worker to complete (hence each worker performs $N$ jobs over her lifetime).  The supply of jobs at each period is limited.  Each time a worker and job are matched, a (random) payoff is generated and observed by the platform, where the payoff distribution depends on the worker type and the job type. (We assume a Bernoulli distribution for the payoffs.)  To incorporate the limited supply of jobs in the simplest possible way, our model considers a continuum of workers and jobs. As a consequence, in our analysis, we find that for a suitable class of policies, there is stochasticity only at the level of individual workers and not at the level of the overall system.

As our emphasis is on the interaction between matching and learning, our model has several features that focus our analysis on that interaction. First, we assume that the platform centrally controls matching: at the beginning of each time period, the platform matches each worker in the system to an available job.  Second, strategic considerations are not modeled; this remains an interesting direction for future work.  Finally, we focus on the prototypical goal of maximizing the steady-state rate of payoff generation. (This is a reasonable proxy for the goal of a platform that takes a fraction of the total surplus generated through matches.)

We assume the platform has system-level knowledge of the arrival rates of workers and jobs, as well as the expected payoff generated when workers of a given type are matched to jobs of a given type.  However, while we assume job types are known to the platform, we assume the platform is initially {\em unaware} of any specific worker's type on arrival.  (This is consistent with the observation that in most platforms, more is known about one side than the other.)

The platform learns about workers' types through the payoffs obtained when they are matched to jobs.  This gives rise to the central learning challenge: because the supply of jobs is limited, using jobs to learn can reduce immediate payoffs, as well as deplete the supply of jobs available to the rest of the marketplace.  Thus the presence of capacity constraints forces us to carefully design both exploration and exploitation in the matching algorithm in order to optimize the rate of payoff generation.





Our main contribution in this paper is the development of a matching and learning policy that is nearly payoff optimal.  Our algorithm is divided into two phases in each worker's lifetime: {\em exploration} (identification of the worker type) and {\em exploitation} (optimal matching given the worker's identified type).  We refer to our policy as {\em DEEM: Decentralized Explore-then-Exploit for Matching}.

DEEM is an algorithm that assigns jobs to workers over time.  We begin by noting that DEEM has a natural decentralization property: it determines the choice of job type for a worker based only on that worker's history, and not based on any other workers' histories. (We note, however, that DEEM itself is designed with knowledge of the global system-level statistics described above.)
This decentralization is inspired by the fact that in large-scale online platforms, matching is typically carried out on an individual basis.  For example, if a worker searches for jobs on an online labor market platform, the platform will generally display available jobs in a personalized rank order based on metadata about that worker. (In practice, this decentralization arises in part due to the inherent asynchronous nature of these platforms: workers and jobs arrive continuously over time, and batched centralized matching may be infeasible as a product design.)

At a high level, DEEM operates as follows during the lifetime of a given worker.  First, DEEM {\em explores} to make a confident estimate of the type of this worker.  This exploration phase consists of two modes: a {\em guessing} mode, where DEEM initially samples job types uniformly at random to develop a reasonable maximum a posteriori (MAP) estimate of the worker's type; and a {\em confirmation} mode, when DEEM chooses jobs to confirm the MAP type as efficiently as possible.  The exploration phase is followed by an {\em exploitation} phase, during which jobs are assigned based on the worker type that was confirmed during exploration.  Each of these phases is carefully designed to optimize the rate of payoff generation while ensuring that capacity constraints are met.

To develop intuition for our solution, 
consider a simple example with two types of jobs (Easy and Hard) and two types of workers (Expert and Novice).  Experts can do both types of tasks well; but novices can only do easy tasks well.  Suppose that there is a limited supply of easy jobs: more than the mass of novices available, but less than the total mass of novices and experts.  In particular, to maximize payoff the platform must learn enough to match some experts to hard jobs.

DEEM has several key features, each of which can be understood in the context of this example.  {\em First}, because DEEM operates at the level of a given worker, we must ensure that the algorithm nevertheless does not violate capacity constraints.  In particular, it is essential for the algorithm to account for the externality to the rest of the market when a worker is matched to a given job.  For example, if easy jobs are relatively scarce, then matching a worker to such a job makes it unavailable to the rest of the market.  Our approach is to ``price'' this externality: we find {\em shadow prices} for the capacity constraints, and adjust all per-match payoffs downward using these prices.

{\em Second}, our algorithm design specifies {\em learning goals} that ensure an efficient balance between exploration and exploitation. 
In particular, in our example, we note that there are two kinds of errors possible while exploring: misclassifying a novice as an expert, and vice versa.  Occasionally mislabeling experts as novices 
is not catastrophic: some experts need to do easy jobs anyway, and so the algorithm can account for such errors in the exploitation phase.  Thus, relatively less effort can be invested in minimizing this error type.  However, mistakenly labeling novices as experts {\em can} be catastrophic: in this case, novices will be matched to hard jobs in the exploitation phase, causing substantial loss of payoff; thus the probability of such errors must be kept very small.  A major contribution of our work is to precisely identify the correct learning goals that determine progression of the algorithm from the exploration phase to the exploitation phase, and to then design DEEM to meet these learning goals while maximizing payoff generation.

{\em Third}, the exploitation phase in DEEM is carefully constructed to ensure that capacity constraints are met while maximizing payoffs. A naive approach during the exploitation phase would match a worker to any job type that yields the maximum externality-adjusted payoff corresponding to his type label. It turns out that such an approach leads to significant violations of capacity constraints, and hence poor performance. The reason is that in a generic capacitated problem instance, one or more worker types are indifferent between multiple job types, and appropriate allocation across multiple optimal job types is necessary to achieve good performance. In our theoretical development, we achieve this by modifying the solution to the static optimization problem with known worker types, whereas our practical implementation of DEEM achieves appropriate allocation via simple but dynamically updated shadow prices.

Our main result (Theorem \ref{thm:mainresult}) shows that DEEM achieves essentially optimal regret as the number of jobs $N$ performed by each worker during her lifetime grows, where regret is the loss in payoff accumulation rate relative to the maximum achievable with known worker types.  In our setting, a lower bound on the regret is $(C\log N/N)(1+ o(1))$ for some $C \in [0,\infty)$ that is a function of system parameters (we use the technical machinery developed in \cite{agrawal1989asymptotically} for a related problem to prove this bound).  DEEM achieves this level of regret to leading order when $C>0$, while it achieves a regret of $O(\log\log N/N)$ when $C=0$.

Situations where $C>0$ are those in which there is an inherent tension between the goals of learning and payoff maximization.  To develop intuition, consider an expanded version of the above example, where each worker can be either an expert or novice programmer, as well as an expert or novice graphic designer.  Suppose that the supply of jobs is such that if worker types were known, only expert graphic designers who are also novice programmers would be matched to graphic design jobs. (This would be the case, e.g., if there were an excess supply of programming jobs, whereas the supply of graphic design jobs were less than the volume of available workers who are both expert graphic designers and novice programmers.)
But if we are learning worker types, then expert graphic designers must be matched to approximately $\Omega(\log N)$ programming jobs to distinguish between novice and expert programmers, so that they can be matched to graphic design and programming jobs, respectively.  Thus $\Omega(\log N/N)$ average regret per period is incurred relative to the optimal solution with known types. DEEM precisely minimizes the regret incurred while these distinctions are made, thus achieving the lower bound on the regret.

Our theory is complemented by a practical heuristic that we call $\textrm{DEEM}^+$, which optimizes performance for small values of $N$, an implementation leveraging queue-length based shadow prices that demonstrates a natural way of translating our work into practice, and supporting simulations. In particular, our simulations reveal substantial benefit from jointly managing capacity constraints and learning, as we do in DEEM and $\textrm{DEEM}^+$.


The remainder of the paper is organized as follows.  After discussing related work in Section \ref{sec:relatedwork}, we present our model and outline the optimization problem of interest to the platform in Section \ref{sec:model}.
In Section~\ref{sec:DEEM}, we discuss the above three key ideas in the design of DEEM, and present its formal definition.   In Section \ref{sec:mainresult}, we present our main theorem, and discuss the optimal regret scaling.  In Section~\ref{sec:proofsketch} we present a sketch of the proof of the main result. 
In Section~\ref{sec:deemplus}, we discuss the practical implementation of DEEM and present the heuristic \DEEMplus. In Section~\ref{sec:simulations}, we use simulations to compare the performance of \DEEMplus with benchmark multi-armed bandit algorithms. We conclude in Section \ref{sec:conclusion}. All proofs are in the appendices. 

%% file: relatedwork.tex
\section{Related literature}
\label{sec:relatedwork}

Below we discuss the relationship between our work and several related threads in the literature on (1) general stochastic multi-armed bandits, (2) dynamic pricing and learning, (3) combinatorial bandits, including bandits with matching constraints, and (4) dynamic stochastic matching models.

Before surveying these threads of literature, we note here that in the period since the initial development of our results, our paper has inspired a subsequent paper, \cite{hsu2018integrating}, which studies a very similar matching while learning setting, and shows near optimality of a ``backpressure'' algorithm similar to the finite $N$ heuristic $\DEEMplus$ that we propose here (see Sections~\ref{sec:deemplus} and \ref{sec:simulations}).  {\em Backpressure} is a celebrated methodology that prescribes using current queue-lengths as shadow prices \citep{tassiulas1990stability}.  \cite{hsu2018integrating} goes beyond this paper by showing near optimality of backpressure in their setting, but at a cost: their bounds on the bandit (learning) problem are loose with a $1/\sqrt{N}$ upper (achievability) bound on the regret, which is much larger than their $\log N/N$ lower bound. By contrast, our theoretical analysis is focused on a tight characterization of regret; we obtain tight $\log N/N$ bounds on the regret, which match even in the constant factor.

{\bf Stochastic multi-armed bandits}.  A foundational model for investigating the exploration-exploitation tradeoff is the stochastic multi-armed bandit (MAB) problem  \citep{lai1985asymptotically,bubeck2012regret, gittins2011multi,icml_tutorial}.  The goal in this problem is to find an adaptive expected-regret-minimizing policy for choosing among arms with unknown payoff distributions, where regret is measured against the expected payoff of the best arm \citep{lai1985asymptotically,auer2002finite, agrawal2011analysis}.

The closest work in this literature to the MAB problem we tackle is by \citet{agrawal1989asymptotically}.  In their model, they assume that the joint vector of arm distributions can only take on one of finitely many values.  This introduces correlation across different arms.  Depending on certain identifiability conditions, the optimal regret is either $\Theta(1/N)$ or $\Theta(\log N/N)$.  In our model, the analog is that job types are arms, and for each worker, we solve a MAB problem to identify the true type of a worker from among a finite set of possible worker types.
In fact, the model of \cite{agrawal1989asymptotically} is a special case of our model with no capacity constraints on jobs. Like us, they study the limit $N \to \infty$ and find a policy that achieves regret that is optimal to leading order as $N \to \infty$. To the best of our knowledge, their result remains the state of the art in their setting. Capacity constraints are of course the innovation and focus of the present paper. Notably, our main result generalizing that of \cite{agrawal1989asymptotically} to allow capacity constraints is \emph{as sharp as the result they obtained in their much simpler setting} in the case where there is a tension between learning and exploitation (i.e., the case where regret is $\Theta(\log N/N)$).

As demonstrated in ~\cite{agrawal1989asymptotically}, the key to attaining the instance-dependent optimal leading-order regret in such multi-armed bandit problems is the following intuition. Given a potential true model, there is a regret-optimal policy that distinguishes this model from all competing models that entail different optimal decisions (defined in \eqref{def:ci} for our setting). Hence, to minimize regret, the challenge is to utilize this model-specific regret-optimal policy to learn the true model, without a priori knowing the true model. This is precisely the challenge we tackle using the guess-then-confirm approach in the exploration phase of DEEM.
Recently, \cite{modaresi2019learning} have addressed a similar challenge in a general combinatorial bandit setting.

On a related note, \citet{massoulie2018capacity} study a pure learning problem in a setting similar to ours with capacity constraints on each type of server/expert; while there are some similarities in the style of analysis, that paper focuses exclusively on learning the exact type, rather than balancing exploration and exploitation as we do in this paper.

{\bf Dynamic pricing and learning}.  Some of the techniques used in our work have parallels in works on dynamic pricing and learning with a finite inventory of products (for a recent comprehensive survey of dynamic pricing and learning, see \cite{den2015dynamic2}). These are essentially MAB problems where the decisions involve choosing product prices dynamically over a selling horizon, with a capacity constraint arising from the finite inventory. 
A typical approach in these settings is to consider a regime where both the inventory and the demand grow large (although there are exceptions, notably \cite{den2015dynamic}). This is similar to the regime we consider for our technical results, which is equivalent to having both the job arrival rates and the worker lifetimes simultaneously approach infinity.\footnote{\label{fn:scaling-regime}We in fact consider a regime where the job capacities are held constant and we reduce the worker arrival rates as their lifetime increases. But it is straightforward to see that we can equivalently keep the worker arrival rates fixed and increase the job capacities as the worker lifetimes increase, without impacting any of our results or insights.}  
Such a regime was first analyzed in the case of a single product in \cite{besbes2009dynamic}, which proposed algorithms with an explore-then-exploit structure for settings with both parametric and non-parametric uncertainty. A more sophisticated algorithm that mixes exploration and exploitation with an improved regret performance in both settings is presented in \cite{wang2014close}. \cite{besbes2012blind}  and, recently, \cite{ferreira2018online} extend the analysis to network revenue management settings involving multiple products using multiple resources with finite inventories. More generally, a recently proposed formulation for MAB problems with capacity constraints, broadly referred to as {\em bandits with knapsacks} \citep{badanidiyuru2013bandits} and its extensions \citep{badanidiyuru2014resourceful,agrawal2014bandits,agrawal2015contextual,agrawal2015linear}, subsume several problems in revenue management under demand uncertainty; see for instance \cite{saure2013optimal} and \cite{babaioff2015dynamic}, in addition to the models discussed above.

The algorithms designed in all of these works critically leverage the solution to the optimal pricing problem in the full information setting in a deterministic world where stochastic quantities are replaced by their means. Similar to these works, we also crucially utilize the full information optimal assignment problem (which is a linear program in our case), and in particular the optimal shadow prices for the jobs from the dual of this optimization problem, in determining the job assignments under DEEM. It is known 
that simply using the optimal price corresponding to the best model estimate from the obtained information at any step (also known as ``certainty equivalent'' control) can potentially lead to incomplete learning and hence linear regret; see for instance Proposition 1 in \cite{den2014simultaneously}. Thus judicious experimentation with prices is necessary.

In a similar fashion, naively using the optimal shadow prices from the full information optimization problem to greedily assign jobs based on current estimates of the worker type typically leads to linear regret in our setting (see Fact~\ref{fact:tiebreaking_necessary} in Section~\ref{subsec:example}). The problem is twofold in our case: the issue is not only that learning may  stop prematurely under such a policy, but also that {appropriate allocation} across {\it multiple} optimal job types is typically necessary in our setting to satisfy capacity constraints. Thus a good algorithm in our setting needs to achieve both goals, judicious experimentation and effective {allocation} across optimal assignments, to achieve low regret. In fact we go one step further, obtaining a policy that achieves not just sublinear but near-optimal regret.

Another key difference in our work compared to these models is that they consider a single MAB problem over a fixed time horizon. 
Our setting on the other hand can be seen as a system with an ongoing arriving {\em stream} of MAB problems, one per worker, that are coupled together by the capacity constraints on arriving jobs.

{\bf Bandits with matching constraints and combinatorial bandits.}  Several MAB problems with matching constraints can be seen as instances of a larger class of models typically referred to as combinatorial bandits \citep{gai2010learning,gai2012combinatorial, liu2012adaptive,chen2013combinatorial,saure2013optimal,kveton2015tight}. Considering the problem of matching all the workers that exist on a platform to the set of available jobs in a particular time period, one can think of the combinatorial set of all possible matchings as being the arms in a MAB setting (sometimes called ``superarms''); this formulation is the closest to the one in \cite{gai2010learning}. Several works have looked at exploiting the structure of such problems in various settings to yield efficient learning algorithms (e.g., \cite{saure2013optimal, gai2010learning, liu2012adaptive}).

In our case, there are two key aspects that make such a reduction to combinatorial bandits infeasible. First, the number of workers and jobs on real-world platforms is large, 
and hence the number of possible matchings is prohibitively large, even when one accounts for limited variety in job types (worker types are unknown and there is a vast heterogeneity in worker histories). Thus decentralization is critical to obtaining a practically feasible solution, which is a feature rarely seen in combinatorial bandit algorithms. 
Second, the fact that the workers are arriving and leaving asynchronously means that the set of possible matchings, and hence the set of combinatorial arms, is changing over time, which is another feature that is relatively uncommon in the extant literature. An example is \cite{chakrabarti2009mortal}, which considers this problem in a non-combinatorial setting.

{\bf Other dynamic stochastic matching models.}  We briefly discuss a few other directions that are related to this paper.  There are a number of recent studies that consider efficient matching in dynamic two-sided matching markets \citep{akbarpour2014dynamic, anderson2015dynamic, baccara2015optimal, ozkan2017dynamic, hu2015dynamic, kadam2015multi, damiano2005stability, kurino2020credibility, das2005two}.  A related class of dynamic resource allocation problems, online bipartite matching, is also well studied in the computer science community (see \cite{mehta2012online} for a survey). 
Similar to the present paper, \citet{fershtman2015dynamic} also study matching with learning, mediated by a central platform.  Relative to our model, their work does not have constraints on the number of matches per agent, while it does consider agent incentives.

%% file: model.tex
\section{The model and the optimization problem}
\label{sec:model}

In this section we first describe our model.  In particular, we describe the primitives of our platform (``workers'' and ``jobs''), and give a formal specification of the matching process we study.  We conclude by precisely defining the optimization problem addressed in this paper.

{A key aspect of our approach is that we consider a model with a {\em continuum} of workers in the system.  The policies we propose for matching workers to jobs will recommend a job type independently for each worker as a function of the ``history'' of that worker alone.  In our analysis, we leverage the general framework provided by \cite[][esp. Section 2.4]{sun2006exact}, which provides a formal mathematical basis for a continuum of independent stochastic processes, including the exact law of large numbers (ELLN) for cross-sectional averages 
(Theorem 2.16 of \cite{sun2006exact}).   Informally, applying this framework allows the interchange of worker-level probabilistic statements with population-level statements about the evolution of the cross-sectional worker measure over time, yielding the tractable (though challenging) optimization problem we study in this paper (see Section \ref{sec:optimization}).  We apply the ELLN throughout our development below to yield such interchanges, as appropriate.}

\subsection{Preliminaries: A continuum model}
\label{subsec:model-prelims}

In this section, we describe the basic model that we work with.

{\bf Time}.  We assume that time is discrete $t = 0, 1, 2, \ldots$.

{\bf Probability space}.  We fix a probability space $(\Omega, \cF, P)$. An element $\omega \in \Omega$ is a state of the world.  All randomness throughout our development below is resolved by the state of the world $\omega \in \Omega$.  An event is a measurable subset $B$ of $\Omega$ (that is, an element of $\cF$), whose probability is $P(B)$. Any statements of events occurring ``with probability 1'' refer to almost sure events w.r.t.~the measure $P$.

{\bf Workers and jobs}.  For convenience we adopt the terminology of {\em workers} and {\em jobs} to describe the two sides of the market.  Each job in the system is of one of a fixed finite set of {\em job types} $\cS$, and each worker in the system is one of a fixed finite set of {\em worker types} $\cC$. {We consider a continuum model with infinitesimal workers and jobs, and thus refer to \emph{masses} of workers and jobs.} Informally, this approach is intended to capture a large market, i.e., where many workers and jobs are present at each time step.

We assume a fixed unit mass of workers; we view the space of workers as a measure space, endowed with the Lebesgue measure on $[0,1]$ and the Borel $\sigma$-algebra.
Each element $g \in [0,1]$ represents a worker. \cite{sun2006exact} provides a {\em Fubini extension} of the product measure corresponding to the worker measure space and the probability space $(\Omega, \cF, P)$; this extension is, roughly, a rich enough probability measure on the product space such that the Fubini property holds.  We leverage this extension in our development. 

We wish to model a process by which workers arrive and depart from the system; however, for technical simplicity we also wish to consider a system where the mass of workers remains finite at all times.
To achieve both goals, we consider a system where each worker {\em regenerates} after every $N$ time periods; we refer to $N$ as the {\em lifetime} of a worker.\footnote{Our analysis and results generalize to random (exogenous) worker lifetimes that are i.i.d.~across workers of different types, with mean $N$ and any distribution such that the lifetime exceeds $N/\textup{polylog}(N)$ with high probability. In particular, the definition of our DEEM policy in Figures~\ref{fig:def-DEEM} and \ref{fig:def-alphai-ystar} remains unchanged except that the condition $k<N$ in the while commands in lines 9 and 21 of Figure~\ref{fig:def-DEEM} is replaced by the condition that the worker has not yet left the system.
Theorem~\ref{thm:mainresult} remains unchanged as well. Note that the platform only needs to know the mean lifetime $N$ beforehand to implement DEEM; it suffices for the platform to find out about the departure (as per the realized lifetime) of a worker only when it occurs.}  We assume the platform knows $N$.

Formally, fix a distribution $\rho$ over worker types, i.e., $\rho_i>0 \ \forall i \in \cI$ such that $\sum_{i \in \cC} \rho_i = 1$. 
We assume that the system initially starts empty prior to $t = 0$, and in each time period $t = 0, \ldots, N-1$, a mass $1/N$ of workers arrives to the system. (In what follows we ultimately consider a steady-state analysis of the dynamical system, and initial conditions will be irrelevant.) 
Each worker is of type $i$ with probability $\rho_i$; these realizations are independent across workers.\footnote{Here and throughout, ``independence'' of a continuum of random variables means that any finite subcollection is mutually independent.} 
No further arrivals take place after time period $N$.  Instead, each worker subsequently regenerates every $N$ periods after their arrival: at a regeneration time, the worker type is resampled from the distribution $\rho$; i.e., the new type is $i$ with probability $\rho_i$, and these regenerations are also independent across workers {and across time}.  
The ELLN (Theorem 2.16 of \cite{sun2006exact}) ensures that, at each time $t$ subsequent to time $N$, the mass of workers of type $i$ in the system is exactly $\rho_i$.  (In what follows we will consider the scaling regime where $\rho_i$ is held constant and $N \to \infty$.)
When the meaning is clear from the context, we sometimes refer to a worker type regeneration as an ``arrival.'' 
Correspondingly, we sometimes refer to $\rho_i$ as the arrival rate of workers of type $i$.

Each worker has the opportunity to do at most one job during each time period of their lifetime.  We assume that in each time period a mass $\mu_j> 0 $ of jobs of type $j$ arrive to be matched to workers; each job lives for only a single time period; {we call $\mu_j$ the \emph{capacity constraint} of job type $j$}.  The platform's {\em matching policy} determines how workers are matched to jobs; we elaborate further on matching policies below.

We assume that type uncertainty exists only for workers; i.e., the platform knows the types of arriving jobs exactly, but only knows that each newly arrived worker has type drawn i.i.d. from $\rho$ and needs to learn the types of workers.  We also assume that the arrival rates of jobs $(\mu_j)_{j \in \cJ}$ and the distribution of worker types $(\rho_i)_{i \in \cI}$ are known to the platform. 

{\bf Matching and the payoff matrix}.  If a worker of type $i \in \cI$ is matched to a job of type $j \in \cJ$, then the resulting match, independent of everything else, generates a Bernoulli reward with success probability $A(i,j) \in [0,1]$.  The matrix $A$ thus characterizes compatibility between workers and jobs.  We call the matrix $A$ the {\em payoff matrix}.  Throughout, we assume that no two rows of $A$ are identical. (This mild requirement simply ensures that it is possible, in principle, to distinguish between each pair of worker types.) As we will only be concerned with the long-run rate of payoff generation, we do not concern ourselves with the division of this payoff between workers and employers. We assume that realized payoffs are observed by the platform.  

For ease of exposition, we define an ``empty'' job type $\kappa$, such that all worker types matched to $\kappa$ generate zero reward, i.e., $A(i,\kappa) = 0$ for all $i$.  We view $\kappa$ as representing the possibility that a worker goes unmatched, and thus assume that an unbounded capacity of job type $\kappa$ is available, i.e., $\mu_\kappa = \infty$. We assume that $\kappa$ is included in $\cJ$.

A key assumption in our work is that the platform {\em knows} the matrix $A$.  In particular, we are considering a platform that has enough aggregate information to precisely decipher the compatibility between different worker and job types.

We note here that a platform can estimate $\mu$, $\rho$, and $A$ from data: the job arrival rates $\mu$ can be directly estimated empirically since job types are observed, while the worker arrival rates $\rho$ and payoff matrix $A$ can be indirectly estimated using the observed outcome data as described in Appendix~\ref{apx:estimation}. 


{\bf Generalized imbalance}.  Throughout our technical development, we make a mild structural assumption on the problem instance, defined by the tuple $(\rho, \mu, A)$.  This is captured by the following definition.
We say that arrival rates $\rho = (\rho_i)_{i \in \cC}$ and $\mu = (\mu_j)_{ j \in \cS}$ satisfy the \emph{generalized imbalance condition} if there is no pair of nonempty subsets of worker types and job types $(\cC', \cS')$, such that the total worker arrival rate of $\cC'$ exactly matches the total job capacity of $\cS'$. Formally,
\begin{align}
  \sum_{i\in\cC'}\rho_i \neq \sum_{j\in\cS'}\mu_j \quad \forall   \cC'\subseteq \cC,  \cS'\subseteq \cS,  \cC' \neq \phi \, \label{eqn:genimb}.
\end{align}
The generalized imbalance condition holds generically.\footnote{The set $(\rho, \mu)$ for which the condition holds is open and dense in $\textup{Relint}(\Delta_{|\cI|}) \times \mathbb{R}_{++}^{|\cJ|}$, where $\Delta_{|\cI|}$ is the probability simplex in $|\cI|$ dimensions, $\textup{Relint}(\cdot)$ denotes the relative interior, and $\mathbb{R}_{++}$ are the strictly positive real numbers. 
} Note that this condition does not depend on the matrix $A$. (The condition will ensure that the shadow prices corresponding to capacity constraints under full information are uniquely determined; see Proposition~\ref{prop:uniqueness_of_prices} in Section~\ref{sec:preliminaries}.)

\subsection{Matching policies and platform objective}
\label{subsec:matching-policies}


A {\em matching policy} is what the platform uses to match jobs to workers.  Informally, we model the following process. The operator knows, at any point in time, the history of each worker in the platform, and also knows the job arrival rates $\mu_j$ for $j \in \cS$.  The matching policy of the platform decides how to match workers and jobs; in particular, it decides which job type each worker is assigned to, while respecting the capacity constraints on job types.  

With this intuition in mind, we now formally define a matching policy, and then define the platform's goal: to choose a matching policy that maximizes the long-run average rate of payoff generation. 

{\bf Worker history}.  To define the state of the system and the resulting matching dynamics, we need the notion of a worker history; informally, this is the full history of a given worker since her last regeneration.  Formally, a {\em worker history of length $k$} is a tuple $H_k = ((j_1, r_1), \ldots, (j_k, r_k))$, where $j_{k'}$ is the job type this worker was matched to at her ${k'}$-th time step in the system since her last regeneration, for $1 \leq {k'} \leq k$; and $r_{k'} \in \{0, 1\}$ is the corresponding reward obtained.  Note that since workers persist for $N$ jobs between regenerations, the histories will have lengths $k = 0, \ldots, N-1$.  We use $H$ to denote a generic history. We let $\phi$ denote the empty history (for $k = 0$). We let $\cH = \cup_{k = 0}^{N-1}{(\cJ \times \{0,1\})^k}$ denote the set of possible histories.

{\bf Full system state.}  The {\em full system state} (also referred to as the {\em full state} or simply the {\em state}) at time $t$ is a mapping from workers to their histories and true types, $\xi_t: [0,1] \rightarrow \cH \times \cI$.

{\bf Observable system state.}  Note that the platform is not able to observe the true type of a worker; in particular, for any $g \in [0,1]$, the platform only observes the history of the worker $g$.  Define $\hat{\xi}_t : [0,1] \rightarrow \cH$ as the projection of the full state $\xi_t$ onto the set of histories $\cH$; this is the {\em observable state} at time $t$.  Any policy the platform implements must depend on only the observable state.

Recall that the system starts with no workers in the system before time $t=0$.  Our subsequent development will ensure that $\hat{\xi}_t$ is Borel measurable with probability 1 for all times $t = 0, 1, \dots$. 
%
%

{\bf Matching policy.}
The platform uses a {\em matching policy} to assign each worker to a job type in $\cJ$ (recall that we think of unmatched workers as being matched to the empty job type $\kappa$).  
As mentioned above, we assume that any mass of jobs left unmatched in a given period disappears at the end of that period, though our results do not depend on this assumption. Fix $N$, and recall that the platform is assumed to know $N$.

Formally, a matching policy is a mapping, for each $t$, from the observable states $\hat{\xi}_t$ to assignments of workers to job types. 
We restrict attention to matching policies such that for all $t = 0, 1, \dots$ and for any measurable $\hat{\xi}_t$, 
with probability 1, the set of workers with each history in $\cH$ assigned to each job type in $\cJ$ is Borel measurable; we refer to these as {\em measurable} matching policies.

Further, we restrict attention to matching policies that are \emph{capacity-feasible}; a policy is capacity-feasible if, for all $t \in \mathbb{N}$ and for any measurable $\hat{\xi}_t$, w.p.~1, the set of workers assigned to each job type $j$ has mass (i.e., Lebesgue measure) no more than the capacity $\mu_j$ for each $j \in \cJ$. 

The matching policy can choose a randomized assignment; in this case all relevant randomness used by the policy is encompassed by $\omega$, the state of the world.

Note that the definition of a matching policy and the definitions of measurability and capacity feasibility all appeal only to the notion of the observable state. We also note in passing that the platform can define a matching policy and check that it is measurable and capacity-feasible even without knowing $A$ and $\rho$.




{\bf System dynamics.}  {Next we will describe the system dynamics; we subsequently use these to specify the platform objective.} 

Fix a matching policy.  In each period $t$, for each worker $g$ (with history denoted by $H$), the matching policy determines the job type $j$ assigned to that worker.  If that worker is actually of type $i$, then the realized payoff is $r \sim \textup{Bernoulli}(A(i,j))$ and the new history of $g$ becomes $(H, (j,r))$ (if $g$ does not regenerate); otherwise, the payoff accrues but $g$ regenerates to an empty history with true type resampled (independently) from distribution $\rho$.

Note that by the same Fubini extension 
of \cite{sun2006exact}, for any measurable matching policy $\pi$, the set of workers of history $H$ with true type $i$ assigned to job type $j$ will be Borel measurable with probability 1 at all times $t$; for policy $\pi$, call this mass $m_{\pi,t}(H, i, j)$.  
Then it follows by the ELLN of \cite{sun2006exact} that the reward generated from these assignments at time $t$ is $m_{\pi,t}(H,i,j) A(i,j)$. {Note that for a general policy $\pi$, the mass $m_{\pi, t}(H,i,j)$ is a random variable. Also observe that for any candidate policy $\pi$, the platform can compute the distribution of $m_{\pi, t}(H,i,j)$ and hence the reward generated using its knowledge of $A$ and $\rho$. (The platform can perform this computation offline for any candidate policy, notwithstanding the fact that the true worker types are unobservable.)  For brevity we skip the details of the computation for general policies, 
but provide the full calculation for the ``sufficient'' subclass of policies that we identify in the next section.}

For later reference, we let $x_{\pi,t}(i,j)$ be the derived (random) quantity representing the {\em fraction} of workers of true type $i$ matched to jobs of type $j$ at time $t$ under policy $\pi$; we refer to $x_{\pi,t}$ as the {\em routing matrix at time $t$} of policy $\pi$.  This is a (row) stochastic matrix for each $t$; i.e., each row sums to $1$.  Note that for times $t \geq N-1$, the mass of workers of true type $i$ in the system is exactly equal to $\rho_i$.  Therefore, for $t \geq N-1$, it follows that $x_{\pi,t}(i,j) = \frac{1}{\rho_i}\sum_{H \in \cH} m_{\pi,t}(H,i,j)$.



{\bf Platform objective: Rate of payoff generation.} Recall that each worker generates a payoff of 1 or 0, in each period. The platform then aims to maximize the long-run average of the mass of workers who generate a payoff of 1 in each period. 
(This choice of objective is the analog of the ``total payoff per period'' objective in a setting with finitely many workers.) As a result of the ELLN of \cite{sun2006exact}, the long-run average rate of payoff generation is identical to the long-run average of $\sum_{i \in \cI} \sum_{j \in \cJ} A(i,j) \sum_H m_{\pi,t}(H,i,j)$.  

The long-run average may not exist for an arbitrary measurable policy, and so formally we define the objective as the limit inferior of the expectation of this quantity:
\begin{align}
\underline{V}(\pi) &= \liminf_{T\rightarrow \infty} \E [ V_T(\pi) ] \label{eq:objective}\\
\textup{where}\ V_T(\pi) &= \frac{1}{T} \sum_{t=1}^T \sum_{i\in\cI}\rho_i\sum_{j\in\cJ} x_{\pi,t}(i,j) A(i,j) \, .
\label{eq:V_T}
\end{align}
Note that in the definition of $V_T$, we make the substitution that $\rho_i x_{\pi,t}(i,j) = \sum_{H \in \cH} m_{\pi,t}(H,i,j)$, since the latter relation holds for all $t \geq N$.  The goal is to find policies that maximize this objective. As per our earlier remark, the platform is able to compute offline the objective value $\underline{V}(\pi)$ for 
any candidate policy $\pi$, even though the true worker types are unobservable.

\subsection{Worker-history-only policies}
\label{subsec:WHO}

Note that, in general, policies may be time-varying, and may have complex dependence on the observable state $\hat{\xi}_t$.  In this subsection, we introduce a much simpler class of policies that we call {\em worker-history-only (WHO) policies}. These are policies where, as a function of the history of each individual worker, a job type is drawn independently from a given distribution, which does not depend on time or on the identity of the worker or on the state of the rest of the system. 

Formally, a WHO policy is associated with a mapping $\pi: \cH \rightarrow \Delta_\cJ$, where $\Delta_\cJ$ denotes the probability simplex of distributions on $\cJ$. (Thus, for WHO policies, we have chosen to identify the notation $\pi$ with the mapping that defines the policy.) For each worker $g$ with current history $H$, the job type for $g$ is sampled from the distribution $\pi(H)$, independently of the other workers. We use $\pi(H,j)$ to denote the $j$-th coordinate of $\pi(H)$.  {Note that WHO policies are {\em anonymous}; i.e., they do not depend on the worker's index.}  We let $\Pi^N$ denote the class of WHO policies, for a given $N$. 

The platform operator may choose the mapping $\pi$ using information available in aggregate, such as the payoff matrix $A$, the worker type distribution ${\rho}$, and the arrival rates of jobs $\mu$. However, the only way that observable state information influences the online matching of a worker to a job in a WHO policy is through the history of the individual worker.  For example, suppose that the platform uses a multi-armed bandit algorithm at the level of an individual worker's history to determine the next job they are matched to; in our model this would be a WHO policy.  In this sense, WHO policies are {\em decentralized} in their assignment of intended job types.

In the remainder of the section, we specialize our model to WHO policies; as we show, this yields a substantially more tractable setting. Observe that, a priori, there is no guarantee that a WHO policy will respect the capacity constraints on jobs.\footnote{In other words, it is possible that a WHO policy may lead to a mass of workers matched to jobs of type $j$ in some period that exceeds $\mu_j$. Formally, to ensure that WHO policies satisfy the capacity-feasibility requirement defined in Section~\ref{subsec:matching-policies}, we define that if the implied assignment under a WHO policy violates a capacity constraint, the WHO policy does not assign any jobs to workers in that period or in any subsequent period. This definition is merely for concreteness; it does not affect our results
because we will ensure that capacity violations occur with probability 0 (Lemma~\ref{lem:capacities-and-steady-state} below).}
To handle this issue, we begin by ignoring capacity constraints; we define the state dynamics and steady-state of a WHO policy, and use this to identify the steady-state rate of payoff generation for such policies.  We then characterize the subclass of WHO policies such that capacity constraints are satisfied.  Finally, we make the important observation that we may restrict attention to WHO policies {\em essentially without loss of optimality} (see Proposition \ref{prop:who_sufficiency} below).  For this reason, in the sequel, we focus on finding approximately optimal WHO policies.


{\bf Steady state of a WHO policy $\pi$.}  Assume no capacity constraints, i.e., $\mu_j = \infty$ for all $j$.
Because WHO policies are anonymous, in analyzing WHO policies it is convenient to work instead with a {\em reduced} state $\nu_t$, called the {\em system profile}, which only measures the aggregate mass of workers with history $H$ and true type $i$ just prior to period $t$, for each pair $(H,i)$.  Formally, $\nu_t(H,i) \triangleq | \xi_t^{-1}(H,i) |$,
where $|\cdot|$ denotes the Lebesgue measure of the set.  As before, we emphasize that this system profile is not observable to the platform, as it does not know the true types of workers.
Note that, using the ELLN {of \cite{sun2006exact}}, w.p.~1, we have that {$m_{\pi, t}(H, i, j) = \nu_t(H, i) \pi(H, j)$} is the total mass of workers of true type $i$ with history $H$ who are assigned to jobs of type $j$ at time $t$.  

The system dynamics are as follows. Since the system starts empty before $t=0$, we have
\begin{align}
\nu_0(H,i) = 0 \quad \textup{for all non-empty histories $H \in \cH \backslash \{\phi\}$ and all $i$}.
\label{eq:dynamics-init}
\end{align}
Worker arrivals and type regenerations lead to
\begin{align}
  \nu_{t}(\phi, i)  = {\rho}_i/N \qquad \textup{for all } t\geq 0 \dots
 \, . \label{eq:dynamics1}
\end{align}
For all $i$, $j$, $t\geq 1$, and histories $H \in \cH$ of length $\leq N-2$, we have
\begin{align}
\nu_{t}( (H, (j,1)), i ) & = \nu_{t-1}(H, i)  \pi(H, j )  A(i,j)\, ; \label{eq:dynamics2}\\
\nu_{t}( (H, (j,0)), i ) & = \nu_{t-1}(H, i)  \pi(H, j )  (1 - A(i,j))\, . \label{eq:dynamics3}
\end{align}
Since $\pi$ and $\rho$ are time independent,
the dynamics \eqref{eq:dynamics-init}--\eqref{eq:dynamics3} yield a unique steady-state after $N-1$ time periods; i.e., $\nu_s = \nu_t$ for all $s, t \geq N-1$ w.p. 1.
Abusing notation, we use $\nu_\pi$ to denote the {\em steady-state} system profile induced by the WHO policy $\pi$.
The steady-state can be inductively computed over histories of increasing length: for the empty history $\phi$ of length zero, we have
\begin{equation}
\nu_\pi(\phi, i)  = \rho_i/N\, . \label{eq:recursion1}
\end{equation}
Then for any history $H$ of length $0, \ldots, N-2$, we have
\begin{align}
\nu_\pi( (H, (j,1)), i ) & = \nu_\pi(H, i)  \pi(H, j )  A(i,j)\, ;  \label{eq:recursion2}\\
\nu_\pi( (H, (j,0)), i ) & = \nu_\pi(H, i)  \pi(H, j )  (1 - A(i,j))\, .   \label{eq:recursion3}
\end{align}

{\bf Routing matrix of a WHO policy $\pi$.}   In steady-state, $\pi$ induces a time-independent fraction $x_\pi(i,j)$ of the mass of workers of true type $i$ that are assigned to type $j$ jobs in each time step. 
In particular,
\begin{align}
x_{\pi}(i,j) \triangleq \frac{\sum_{H\in \cH} \nu_\pi(H,i)\pi(H,j)}{\sum_{H\in \cH} \nu_\pi(H,i)} = \frac{\sum_{H\in \cH}\nu_\pi(H,i)\pi(H,j)}{\rho_i}.
\label{eq:x_pi}
\end{align}

Let
\begin{align}
\cX^N\triangleq \left\{\, x_{\pi}:\pi\in \Pi^N\,\right\} \subseteq [0,1]^{|\cC|\times|\cS|}
\label{eq:cX-N-WHO}
\end{align}
be the set of (steady-state) routing matrices  achievable (when each worker does $N$ jobs) by WHO policies, i.e., for $\pi \in \Pi^N$.  Again, we emphasize that capacity constraints are ignored in the definition of $\cX^N$.  In Appendix~\ref{apx:polytope}, we show the following.
\begin{proposition}\label{prop:learningset}
The set $\cX^N$ is a convex polytope.
\end{proposition}

{\bf Steady-state rate of payoff generation of a WHO policy $\pi$}. Recall the $T$-period average payoff generation rate defined in \eqref{eq:V_T}.  Since a WHO policy is in steady-state for all $t \geq N-1$, it follows that for such a policy the limit $\lim_{T \to \infty} V_T$ exists and is equal to the following {\em steady-state rate of payoff generation} $W^N(\pi)$:
\begin{equation}
\label{eq:payoffrate}
W^N(\pi) \triangleq \sum_{i \in \cC}  \rho_i \sum_{j\in \cS} x_\pi(i,j) A(i,j)\, .
\end{equation}
W.p.~1, this is the payoff generated per time step in steady-state by the policy $\pi$ across the entire population of jobs and workers, since $x_\pi(i,j)$ is the fraction of workers of true type $i$ matched to jobs of type $j$, and $A(i,j)$ is the fraction of these matches that generate a unit reward.  
In the sequel, our goal will be to maximize this rate of payoff generation.


{\bf Satisfying capacity constraints.}  We now return to enforcing the capacity constraints, i.e., $\mu_j < \infty$ for $j \neq \kappa$.  In our analysis, we restrict attention to {\em WHO policies that satisfy capacity constraints}.  Given the above definitions, this is straightforward: we restrict attention to WHO policies $\pi \in \Pi^N$ such that the steady-state routing matrix $x_\pi$ does not require any more than mass $\mu_j$ of jobs of type $j$:
\begin{equation}
\sum_{i\in \cC} \rho_ix_\pi(i,j) =\sum_{i\in \cC} \sum_{H\in \cH}\nu_\pi(H,i)\pi(H,j) \leq \mu_j\qquad \forall j \in \cS. \label{eq:capacity_constraints}
\end{equation}
Any WHO policy $\pi$ that satisfies this constraint will ensure that the capacity constraints are satisfied by the implied assignment in steady-state (i.e., for $t \geq N-1$) w.p.~1 by the ELLN of \cite{sun2006exact}.  In fact, because we assume the system starts empty, the following lemma establishes that for any such policy, w.p.~1, capacity constraints are \emph{never} violated. The lemma is proved in Appendix~\ref{app:transient-capacity}.

\begin{lemma}
Recall that the system starts empty, i.e., $\nu_0 (H,i) = 0$ for all $H \neq \phi, i \in \cI$. Suppose that the WHO policy $\pi$ satisfies \eqref{eq:capacity_constraints}.  Then at all times $t = 0, 1, \ldots,$ w.p. 1, the implied assignment satisfies the capacity constraint; i.e., at each time $t$ and for each job type $j$, the mass of workers matched to jobs of type $j$ does not exceed $\mu_j$:
\begin{align}
\sum_{i \in \cC} \sum_H \nu_t(H, i) \pi(H, j) \leq \mu_j\qquad \forall j \in \cS .
\label{eq:transient-capacity}
\end{align}
Furthermore, the system reaches steady-state at $t=N-1$ and remains in steady-state for all $t \geq N-1$.
\label{lem:capacities-and-steady-state}
\end{lemma}


{\bf Optimality of WHO policies.}  We now establish that the restriction to WHO policies 
is without loss of optimality.
Recall that $V_T(\pi)$ as defined in \eqref{eq:V_T} is the $T$-period average payoff achieved by the (arbitrary, possibly time-varying) measurable and capacity-feasible policy $\pi$.  Hence, the largest possible asymptotic rate of payoff accumulation under policy $\pi$ is $\bV(\pi) \triangleq \ls_{T \rightarrow \infty} \E[V_T(\pi)]$.  The next proposition establishes that a WHO policy exists that satisfies capacity constraints and yields a steady-state rate of payoff generation arbitrarily close to $\bV(\pi)$. The proof can be found in Appendix~\ref{app:who_sufficiency}.
\begin{proposition}
\label{prop:who_sufficiency}
Fix $A$, $\rho$, $\mu$, and $N$. Fix any feasible policy $\pi$ and any $\eps>0$. Then there is a worker-history-only (WHO) policy satisfying \eqref{eq:capacity_constraints} that achieves a steady-state rate of payoff accumulation exceeding $\bV(\pi) - \eps$.
\end{proposition}



\subsection{The optimization problem}
\label{sec:optimization}

We are now in position to state our optimization problem of interest.  We want to find a WHO policy $\pi$ that maximizes the steady-state rate of payoff generation $W^N(\pi)$, subject to the capacity constraints \eqref{eq:capacity_constraints}.  Formally, we have the following problem:
\begin{align}
\text{maximize}\ \ \  & W^N(\pi) \triangleq \sum_{i \in \cC}  \rho_i \sum_{j\in \cS} x_\pi(i,j) A(i,j) \label{prob:mixedbandit}\\
\text{subject to}\ \ \ & \sum_{i\in \cC} \rho_ix_\pi(i,j) \leq \mu_j\qquad \forall j \in \cS \, ; \label{eq:const}\\
& x_\pi \in \cX^N.\label{eq:feasibility}
\end{align}
Since $\cX^N$ (defined above in \eqref{eq:cX-N-WHO}) is a convex polytope, this is a linear program, albeit a complex one. The complexity of this problem is hidden in the complexity of the set $\cX^N$, which includes all possible routing matrices that can be obtained using WHO policies $\pi \in \Pi^N$.  The remainder of our paper is devoted to solving this problem and characterizing its value, by considering an asymptotic regime where $N \to \infty$.

\subsection{The benchmark: Full information setting}
\label{sec:preliminaries}

We evaluate our performance relative to a natural benchmark: the maximal rate of payoff generation possible if worker types are perfectly {\em known} upon arrival.  We will refer to this as the full information setting. In this case, {\em any} (row) stochastic matrix is feasible as a routing matrix.  Let $\cD$ denote the set of all row stochastic matrices:
\begin{equation}
\cD=\bigg\{x\in \mathbb{R}^{|\cC|\times |\cS|}: x(i,j)\geq 0;\,\,\sum_{j\in \cS}x(i,j)= 1\bigg\}\,.
\label{eq:cD}
\end{equation}
Note that any routing matrix in $\cD$ is implementable by a simple policy if worker types are perfectly known: given a desired routing matrix $x \in \cD$, 
at each time step $t$ we match a fraction $x(i,j)$ of workers of type $i$ to jobs of type $j$.

Thus, with known worker types, the maximal rate of payoff generation is given by the solution to the following optimization problem:
\begin{align}
\text{maximize}\ \ \  & \sum_{i \in \cC}  \rho_i \sum_{j\in \cS} x(i,j) A(i,j) \label{eq:opt1}\\
\text{subject to}\ \ \ & \sum_{i\in \cC} \rho_ix(i,j) \leq \mu_j\qquad \forall j \in \cS \, ; \label{eq:opt2}\\
& x \in \cD.\label{eq:opt3}
\end{align}
We let $V^*$ denote the maximal value of the preceding optimization problem, and let $x^*$ denote the solution (breaking ties arbitrarily). We further use $\cJfull$ to denote the set of fully utilized job types
\begin{align}
\cJfull \triangleq \{j \in \cJ: \sum_{i\in\cC}\rho_ix^*(i,j)=\mu_j\}\, .
\label{eq:Jfull}
\end{align}
This linear program is a special case of the ``static planning problem'' that arises frequently in the operations literature (see, e.g., \cite{ata2005heavy}).  The problem can also be viewed as a version of the assignment problem due to Shapley and Shubik \cite{shapley1971assignment}, in which the resources are divisible. We denote the 
shadow prices associated with the capacity constraints \eqref{eq:opt2} by\footnote{Formally, the shadow prices $p^*$ are the values of the corresponding dual variables at an optimum of dual linear program \eqref{eq:dual1}\textendash\eqref{eq:dual2} stated in the appendix.}  $p^* = (p^*_j)_{j\in\cJ}$.
We prove the following fact about these prices in Appendix~\ref{apx:uniqueprices}.
\begin{proposition}
\label{prop:uniqueness_of_prices}
Under the generalized imbalance condition \eqref{eqn:genimb}, the job shadow prices $p^*$ are uniquely determined.
\end{proposition}
As we shall see, these uniquely defined prices $p^*$ will be key to our solution to the problem.

\subsection{Regret}

We evaluate the performance of a given policy in terms of its {\em regret} relative to $V^*$.  In particular, given $N$ and a WHO policy $\pi$ satisfying \eqref{eq:capacity_constraints}, we define the regret of $\pi$ as $V^* - W^N(\pi)$.

We focus on the asymptotic regime where $N \to \infty$, and try to find policies that have ``small'' regret in this regime. 
This asymptotic regime provides tractability, allowing us to identify structural aspects of policies that perform well. In particular, we focus on developing policies that achieve a nearly optimal {\em rate} 
at which the regret $V^* - W^N(\pi_N)$ approaches zero. 


\subsection{Summary}

We summarize our model as follows.
\begin{itemize}
\setlength{\itemsep}{3pt}
  \item The platform chooses a matching policy. In particular, without loss of optimality, it chooses a WHO policy $\pi$.
\item The policy $\pi$ induces a steady-state system profile $\nu_\pi$ and associated steady-state routing matrix $x_\pi$; i.e., 
    at all times $t \geq N-1$, the system profile is $\nu_\pi$ and the mass of workers of type $i$ matched to jobs of type $j$ is $x_\pi(i,j)$.
\item The steady-state routing matrix $x_\pi$ induces a steady-state rate of payoff generation $W^N(\pi)$.
\item The regret of the policy $\pi$ is $V^* - W^N(\pi)$.  We focus on finding WHO policies that yield low regret.
\end{itemize}

%% file: DEEM.tex
\section{Decentralized Explore-then-Exploit for Matching (DEEM): A payoff-maximizing policy}
\label{sec:DEEM}
In this section we present our proposed policy
{\em DEEM: Decentralized Explore-then-Exploit for Matching}.  Our main result (Theorem \ref{thm:mainresult}) will quantify the regret performance of DEEM and characterize it as nearly optimal. 
DEEM is formally defined in Figure~\ref{fig:def-DEEM} with supporting definitions in Figure~\ref{fig:def-alphai-ystar}. To assist the reader, we provide an informal schematic of DEEM in Figure~\ref{fig:DEEM-cartoon}.

\begin{figure}
\fbox{{\footnotesize\begin{minipage}{\textwidth}
\begin{center}{\vspace{0.1in}\normalsize \bf DEEM: Decentralized Explore-then-Exploit for Matching}\vspace{0.1in} \end{center}
\textbf{Input parameters:} $\cI$, $\cJ$, $A$, $\rho$, $\mu$, $N$ {such that the generalized imbalance condition \eqref{eqn:genimb} holds}.\\
{\bf Pre-compute:}\begin{itemize}
\item The $\cJ$-vector $p^*$ of shadow prices for the capacity constraint \eqref{eq:opt2} in the problem with known types \eqref{eq:opt1}\textendash\eqref{eq:opt3}. (Recall from Proposition~\ref{prop:uniqueness_of_prices} that under the generalized imbalance condition \eqref{eqn:genimb} the prices $p^*$ are uniquely determined.)
\item For each $i \in \cI$, the set of worker types
\begin{align}
  \Str(i) \triangleq \{ i' : \cS(i) \setminus \cS(i') \neq \emptyset \} \quad \textup{where} \ \cS(i)\triangleq \arg\max_{j\in \cS} A(i,j)-p^*_j\, .
  \label{eq:Str-i}
\end{align}
\item The distribution $\alpha(i) = \alpha(i,\cI, \cJ, A, p^*, \Str(i) )$ over $\cJ$, for all $i\in \cI$. 
    Defined in Figure \ref{fig:def-alphai-ystar}.
\item $(|\cI| \times |\cJ|)$-right stochastic matrix $y^* = y^*(\cI, \cJ, A, \rho, \mu, N) \in \cD$. 
    Defined in Figure \ref{fig:def-alphai-ystar}.
\end{itemize}
\medskip
\hrule
\medskip
\begin{algorithmic}[1]
\LineComment{Main Routine}
\Procedure{DEEM}{} \Comment{Acts independently on each worker, over her lifetime, from arrival to departure}
\LineComment{Initialization:}
\State $\lambda(i) \gets \rho_i$ for all $i\in \cI$ \Comment{The un-normalized posterior probabilities; initialized to the prior}
\State $\MAP \gets \argmax_{i\in\cI} \lambda(i)$ \Comment{Initialization of the MAP estimate}
\State $Label \gets \phi$ \Comment{Worker label; initially unassigned, denoted by $\phi$}
\State $k \gets 0$ \Comment{Number of time steps the worker has been in the system = Length of the worker's history}
\medskip

\LineComment{Explore phase:}
\While{$Label = \phi$ and $k<N$}
	\State $k \gets k+1$ \Comment{At the next time step}
	\State Assign job type $j_k \sim$  \Call{Explore}{$N$, $\lambda$, \MAP, $\alpha(\MAP)$}
	\State Observe reward $r_k$
	\State $\lambda(i) \gets \lambda(i) \times (A(i,j_k)\mathbf{1}_{\{r_k=1\}}+(1-A(i,j_k))\mathbf{1}_{\{r_k=0\}})$, for all $i \in \cI$
	\State $\MAP \gets \argmax_{i\in\cI} \lambda(i)$
	 \If{$\min_{i\neq \MAP}\frac{\lambda(\MAP)}{\lambda(i)} \geq \log N$ and $\min_{i\in \Str(\MAP)}\frac{\lambda(\MAP)}{\lambda(i)} \geq N$} \label{pc:label-condition}\Comment{If Confirmation is complete}
		\State $Label \gets \MAP$  \Comment{Worker label assigned. Will end Explore phase.}
		\EndIf 
\EndWhile
\medskip

\LineComment{Exploit phase:}
\While{$k<N$}
 \State $k \gets k+1$	\Comment{At the next time step}
 	\State Assign job type $j_k \sim$ \Call{Exploit}{$Label$, $y^*$}
 \EndWhile
 \EndProcedure
 \medskip
 \hrule
 \medskip
 \LineComment{Functions}
\Function{Explore}{$N$, $\lambda$, \MAP, $\alpha_{\tiny \MAP}$}
\If{$\min_{i\neq \MAP}\frac{\lambda(\MAP)}{\lambda(i)} <\log N$} \Comment{If MAP estimate is noisy}
		\State $dist \gets \textup{Uniform}(\cJ)$ \Comment{{\bf Guessing}}
\Else \Comment{MAP estimate is somewhat confident}
		\State $dist \gets \alpha_{\tiny \MAP}$ \label{pc:confirmation} \Comment{{\bf Confirmation}}
\EndIf
\State {\bf return} $dist$
\EndFunction

\medskip

\Function{Exploit}{$Label$, $y^*$}
\State $dist \gets y^*(Label, \cdot)$ \Comment{Sample from distribution given by the $Label$-th row of $y^*$}\label{pc:exploit-sample}
	\State {\bf return} $dist$
\EndFunction
\end{algorithmic}
\end{minipage}}}
\linespread{1}
\caption{Definition of DEEM.}  
\label{fig:def-DEEM}
\end{figure}

\edef\myindent{\the\parindent}
\begin{figure}[htbp]
\fbox{{\footnotesize\begin{minipage}{\textwidth}\setlength{\parindent}{\myindent}
\begin{center} {\vspace{0.1in}\normalsize \bf Definitions \vspace{0.1in}} \end{center}

\noindent {\bf Definition of} $\alpha(i,\cI, \cJ, A, p^*, \Str(i))\in \Delta(\cJ).$

 Let $ U(i) \triangleq \max_{j \in \cS} A(i,j) - p^*_j $ be the maximal externality-adjusted payoff of worker type $i$. Define the set
\begin{align}\label{eqn:alpha}
\mathcal{A}(i) &\triangleq \argmin_{\alpha\in \Delta(\cS)}\frac{\sum_{j\in \cS} \alpha_j\big(U(i)-[A(i,j)-p^*_j]\big)}{\min_{i'\in \Str(i)}\sum_{j\in \cS}\alpha_j\KL(i,i'|j)},
\end{align}
where $\Delta(\cS)$ is the set of distributions over $\cS$ and $\KL(i,i'|j) \triangleq A(i,j) \log \frac{A(i,j)}{A(i',j)} + (1- A(i,j) ) \log \frac{1-A(i,j)}{1-A(i',j)}\,$ is the Kullback--Leibler divergence between $\Bern(A(i,j))$ and $\Bern(A(i',j))$.
Then choose $\alpha(i)$ as per
\begin{equation}
\alpha(i) \in \argmax_{\alpha'\in\mathcal{A}(i)}\min_{i'\in \Str(i)}\sum_{j\in \cS}\alpha'_j\KL(i,i'|j),
\label{def:alphai-fastest-learning}
\end{equation}
breaking ties arbitrarily.\\

\noindent{\bf Definition of} $y^*(\cI, \cJ, A, \rho, \mu, N) \in \cD$.

The Explore phase, in particular the function \textproc{Explore()} and the confirmation policy $\alpha(i)$ for all $i\in \cI$, defines the following masses, which are time invariant for all $t \geq N$:
\begin{itemize}
\item $\mexp(i,j)\triangleq$ The mass of type $i$ workers who are in the Explore phase and get assigned to type-$j$ jobs in a given time step.
\item $l(i,i')\triangleq$ The mass of type $i$ workers in the system who were labeled as being of type $i'$ at the end of the Explore phase, and are now in the Exploit phase.
\end{itemize}
Then, based on a solution $x^*$ to \eqref{eq:opt1}\textendash\eqref{eq:opt3} and $\cJfull$ defined in \eqref{eq:Jfull}, we choose a routing matrix $y^*$ in the Exploit phase that satisfies:
\begin{align}
& y(i,j) = 0 & \forall i\in \cI, j \in \cJ \textup{ s.t. } x^*(i,j)= 0 \, ;\label{eq:y-support}\\
& m(i,j) = \mexp(i,j) + \sum_{i' \in \cI}l(i,i') y(i',j) & \forall i \in \cI, j \in \cJ \, ;\label{eq:y-m}\\
& \sum_{i\in \cC} m(i,j) = \mu_j& \forall j \in \cJfull \, ; \label{eq:y-capfull}\\
& \sum_{i\in \cC} m(i,j) < \mu_j& \forall j \in \cJ\backslash \cJfull \, ; \label{eq:y-capslack}\\
& y \in \cD. &\label{eq:y-stochastic}
\end{align}
Since the generalized imbalance condition holds, using Proposition~\ref{prop:construct},  for any $N$ large enough, there exists a feasible routing matrix $y^*$ such that \eqref{eq:y-support}\textendash\eqref{eq:y-stochastic} hold.\footnote{We ensure that our practical policies derived from DEEM (DEEM-discrete discussed in Section~\ref{sec:deemdiscrete}, and \DEEMplus, specified in Figure \ref{fig:def-DEEMplus}) are well defined for any value of $N$.}

\end{minipage}}}
\linespread{1}
\caption[]{Definitions of $\alpha(i)$ and $y^*$. 
In Appendix~\ref{apx:alpha}, we show that \eqref{eqn:alpha} can be expressed as a small linear program (with $|\cI|$ constraints and $|\cJ|$ variables). The quantity $m(i,j)$ given by \eqref{eq:y-m} represents the mass of type $i$ workers that is matched to type $j$ jobs in steady state.}\label{fig:def-alphai-ystar}
\end{figure}

\begin{figure}[ht]
    \centering
    \includegraphics[width=0.9\textwidth]{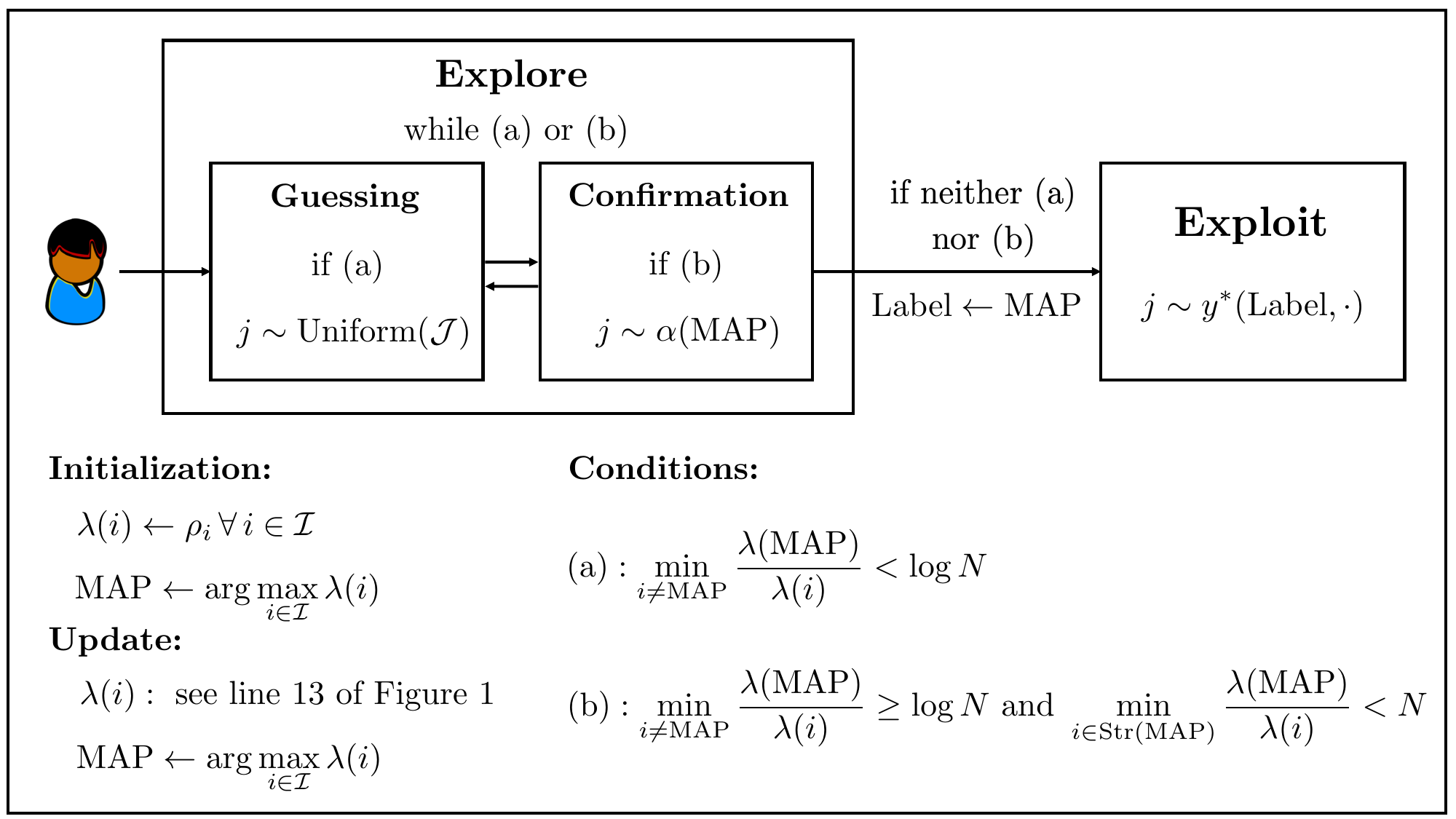}
    \caption{A schematic of DEEM at the individual worker level. $\lambda(i)$ for each $i\in\cI$ is the un-normalized posterior probability of the type being $i$ before each job assignment. It is initialized with the prior $\rho_i$, and is updated after each job assignment based on the observed reward (line 13 of Figure~\ref{fig:def-DEEM}). 
    $\alpha$ and $y^*$ are defined in Figure~\ref{fig:def-alphai-ystar}, while $\Str(i)$ is defined in \eqref{eq:Str-i}.} 
    \label{fig:DEEM-cartoon}
  \end{figure}

DEEM operates individually on every arriving worker; in fact DEEM is a WHO policy (WHO policies were defined in Section~\ref{subsec:WHO}). As shown in Figure \ref{fig:DEEM-cartoon}, DEEM is divided into two phases: Explore and Exploit. In the Explore phase the policy efficiently learns the type of the worker with appropriate confidence and generates a type label. In the Exploit phase, the algorithm focuses on payoff maximization for the given type label in a manner that accounts for system-level capacity constraints.

The Explore phase involves two possible modes of operation: Guessing and Confirmation, and starts in Guessing mode. The Guessing mode is in effect when there is not enough confidence in the maximum a posteriori (MAP) estimate of the worker type based on the observations so far (condition (a) in Figure~\ref{fig:DEEM-cartoon}). It assigns the worker to job types uniformly at random and aims to build confidence in the MAP estimate. The Confirmation mode is in effect when there is sufficient confidence in the MAP estimate to merit focusing on confirming that it is indeed the true type, but not enough to actually start exploiting (condition (b) in Figure~\ref{fig:DEEM-cartoon}). In this mode, DEEM boosts the confidence that the guessed type is correct, trying to rule out types in the carefully defined set $\Str(\textup{MAP})$ (defined in Figure \ref{fig:def-DEEM}, Eq.~\eqref{eq:Str-i}) that must be distinguished to facilitate exploitation. It achieves this goal while minimizing the loss in payoff by sampling job types from an appropriate distribution $\alpha(\textup{MAP})$, defined in Figure~\ref{fig:def-alphai-ystar}. Once an appropriate confidence level is reached in the MAP estimate (condition (c) in Figure~\ref{fig:DEEM-cartoon}), the worker is labeled according to this estimate and the algorithm permanently enters the Exploit phase, in which the label is treated as the true type of the worker.


DEEM uses externality adjustments on payoffs to capture the effect of system-wide aggregate capacity constraints. These adjustments are achieved by using {\em shadow prices} $p^*$ for the capacity constraint \eqref{eq:opt2} in the  problem with known worker types.
Under these adjusted payoffs, a particular worker type may have multiple optimal job types. Appropriate tie-breaking across these types in the Exploit phase is necessary to satisfy the aggregate capacity constraints. This is achieved by employing a specifically designed routing matrix $y^*$, 
which is a perturbed version of the solution $x^*$ to the problem with known worker types \eqref{eq:opt1}\textendash\eqref{eq:opt3}.
The matrix $y^*$ is defined in Figure~\ref{fig:def-alphai-ystar}, and the following proposition   shows the existence of $y^*$ satisfying the conditions specified in the figure.
\begin{proposition}\label{prop:construct}
Suppose that the generalized imbalance condition is satisfied. %
Then, for any $N$ large enough, there exists a feasible routing matrix $y^*$ such that \eqref{eq:y-support}\textendash\eqref{eq:y-stochastic} hold.
\end{proposition}
\noindent The proof is in Appendix~\ref{apx:construct} and shows, moreover, that as $N \to \infty$, a vanishing fraction of workers are in the Explore phase, i.e., $\mexp(i,j)= o (1) \ \forall i \in \cI, j \in \cJ$, that a vanishing fraction of workers are mislabeled at the end of the Explore phase, i.e., $l(i,i')= o (1) \ \forall i \neq i'  \in \cI$, and that there is a small perturbation of $x^*$ that is a feasible solution to \eqref{eq:y-support}\textendash\eqref{eq:y-stochastic}, i.e., $y^* = x^* +o(1)$.


In the next section, we use an example to illustrate the operation of DEEM.  Before we continue, we make three important remarks.

\begin{remark}[{\bf DEEM is a WHO policy}]
\label{rem:DEEM-is-WHO}
Observe that DEEM as defined in Figure \ref{fig:def-DEEM} is a WHO policy; in other words, DEEM defines a mapping from workers' histories to a distribution over job types. The input parameters are the model primitives, which are then used to pre-compute the derived quantities $p^*$, $(\Str(i))_{i \in \cI}$, $(\alpha(i))_{i \in \cI}$, and $y^*$; these are all functions of model primitives only, with no dependence on the current system state. For any individual worker, the control logic of \textproc{DEEM} relies only on the posterior $\lambda(\cdot)$, the current MAP estimate, and $Label$, all of which are functions of the history of the individual worker. Before $Label$ is set, job types are drawn using   \textproc{Explore()}, whereas after it is set, job types are drawn using \textproc{Exploit()}. These functions again construct a job type distribution based only on the history of the individual worker as captured by $\lambda(\cdot)$, MAP, and $Label$ (and model primitives and quantities derived from them).
\end{remark}

\begin{remark}[{\bf DEEM satisfies capacity constraints}]
Proposition \ref{prop:construct} shows that there is a feasible solution to the constraints \eqref{eq:y-support}\textendash\eqref{eq:y-stochastic}.  We use this solution $y^*$ as the routing matrix during exploitation. Since the capacity constraints are incorporated in \eqref{eq:y-capfull} and \eqref{eq:y-capslack}, it follows that DEEM does not run out of jobs of any type, in any time step.
\end{remark}

One possible concern could be that the definition of $y^*$ depends on DEEM and DEEM itself depends on $y^*$. In fact, there is no circularity in the definitions, as explained in the following remark. 
\begin{remark}
{Since DEEM is a WHO policy, and under DEEM each worker completes the Explore phase before entering the Exploit phase, the matrix $y^*$ satisfying \eqref{eq:y-support}\textendash\eqref{eq:y-stochastic} can be computed in an offline fashion using only aggregate statistics $\mexp(i,j)$ and $l(i,i')$ that depend on the Explore phase alone.  This matrix  $y^*$ is then only used to define the Exploit phase of DEEM, to ensure that capacity constraints are met.}
\end{remark}

%% file: results.tex
\subsection{Key features of DEEM via an example}
\label{subsec:example}

In this section, we discuss the structure and the main features of DEEM by way of an example. Consider a simple setting in the context of a labor platform like Upwork, in which worker skills differ along two dimensions: Programming and Design. Further, suppose the skill level is binary in each dimension: each worker either has that skill or doesn't. Suppose that the worker population is composed of three worker types (in order): Programmers (who know only Programming), Designers (who know only Design), and All-rounders (who know both Programming and Design). Finally, there are three job types (in order) with the corresponding subsets of relevant skills: Programming (which depend on only the Programming skill), Design (which depend on only the Design skill) and Mixed (which depend on both skills). Let the payoff matrix (consistent with the subsets of relevant skills) be
\begin{align}
A =
  \begin{blockarray}{*{3}{c}  l}
    \begin{block}{*{3}{>{$\footnotesize}c<{$}} l}
      Programming & Design & Mixed & \\
    \end{block}
    \begin{block}{[*{3}{c}]>{$\footnotesize}l<{$}}
      0.5 & 0.2 & 0.1   \bigstrut[t] \bigstrut[t]& Programmers \\
      0.3 & 0.8 & 0.2  & Designers \\
      0.5 & 0.8 & 0.6  & All-rounders\\
    \end{block}
  \end{blockarray}
\end{align}
Let the arrival rates be $\rho = [\begin{matrix}0.4/1.9 & 0.6/1.9 & 0.9/1.9\end{matrix}]^T$ and the job type capacities be $\mu = [\begin{matrix}1/1.9 & 1/1.9 & 1/1.9\end{matrix}]^T$. In this example, the optimal solution to the benchmark problem \eqref{eq:opt1}\textendash\eqref{eq:opt3} with {\em known} types results in the following allocation of the masses of workers to jobs:
\begin{align}
\rho^T x^* = [\rho_ix^*(i,j)]_{i\in \cI, j \in \cJ} = \left [
\begin{matrix}
  0.4/1.9 & 0 & 0\\
  0 & 0.6/1.9 & 0\\
  0 & 0.4/1.9 & 0.5/1.9
\end{matrix} \right ].\label{benchmark}
\end{align}
There are three key features of DEEM, which we now discuss.
\begin{enumerate}

\item {\em Shadow prices to account for capacity constraints.} 
The intuition behind using $p^*$ for externality adjustment of payoffs under DEEM is that with large $N$, learning will occur quickly relative to the worker lifetime, and $p^*$ will approximate well the shadow prices even with unknown worker types.
At the high level, DEEM is a near-optimal policy for the unconstrained externality-adjusted bandit problem:
\begin{equation}
\textup{maximize}_{x_\pi \in \cX^N} \sum_{i \in \cC}  \rho_i \sum_{j\in \cS} x_\pi(i,j) ( A(i,j) - p^*_j ). \label{eq:optpricedlearn2}
\end{equation}
In our example, the shadow prices corresponding to the capacity constraints in the benchmark linear program under full information \eqref{eq:opt1}\textendash\eqref{eq:opt3} are\footnote{These shadow prices are consistent with the fact that all Design jobs are assigned (but this is not the case for other job types), and marginal All-rounders could generate $0.8-0.6=0.2$ more in payoffs per unit if there were more Design jobs available.} $p^* = [\begin{matrix}0 & 0.2 & 0\end{matrix}]^T$. 
DEEM makes job-assignment decisions based on  the externality-adjusted payoff matrix
\begin{align}
[A(i,j)-p^*_j]_{i\in \cI, j \in \cJ} = \left [
\begin{matrix}
  {\bf 0.5} & 0.0 & 0.1\\
  0.3 & {\bf 0.6} & 0.2\\
  0.5 & {\bf 0.6} & {\bf 0.6}
\end{matrix} \right ]\label{example}
\end{align}
instead of the original payoff matrix $A$. The sets $\cJ(i)$ defined in Figure~\ref{fig:def-DEEM}, Eq.~\eqref{eq:Str-i} capture the job types that maximize the externality-adjusted payoff for worker type $i$; {these maximal externality-adjusted payoffs for each worker type are shown in bold face in the above matrix.} Hence we have, $\cJ(\textup{Programmer})=\{\textup{Programming}\}$, $\cJ(\textup{Designer}) = \{\textup{Design}\}$ and $\cJ(\textup{All-rounder})=\{\textup{Design}, \textup{Mixed}\}$. The routing matrix $y^*$ in the Exploit phase of DEEM assigns labeled workers exclusively to these jobs during exploitation, just as $x^*$ in the benchmark solution exclusively assigns workers to these jobs (see \eqref{benchmark}).

But DEEM also needs to satisfy capacity constraints to be feasible. As implied by the following fact, in general the $\cJ(i)$ may not be singleton sets (as is the case for All-rounders) and appropriate \emph{tie-breaking} between multiple optimal job types for one or more worker types is \emph{necessary during exploitation} to avoid capacity violations.

\begin{fact}
\label{fact:tiebreaking_necessary}
Under the generalized imbalance condition, as long as there is at least one capacity constraint that is binding in some optimal solution $x^*$ to the benchmark problem \eqref{eq:opt1}\textendash\eqref{eq:opt3} with known types, there is at least one worker $i$ such that $x^*(i, \cdot)$ is supported on multiple job types. This implies that $\cJ(i)$ has more than one element.
\end{fact}

\begin{proof}{Proof of Fact~\ref{fact:tiebreaking_necessary}.}
Fix $x^*$ such that at least one capacity constraint binds; i.e., $\cJfull$ defined in \eqref{eq:Jfull} is non-empty. Consider worker types $\cI_{\textup{full}}^* = \{i: \exists j \in \cJfull \textup{ s.t. } x^*(i,j)>0\}$. By the generalized imbalance condition and the fact that all worker types are fully matched (recall that the ``empty'' job type $\kappa$, which serves as a proxy for remaining unmatched, is included in $\cJ$), there must be some  $i \in \cI_{\textup{full}}^*$  and $j' \notin \cJfull$  such that $x^*(i,j')>0$. (If not, $\sum_{i \in \cI_{\textup{full}}^*} \rho_i = \sum_{j \in \cJfull} \mu_j$, which contradicts the generalized imbalance condition.)\hfill $\Box$
\end{proof}

In DEEM, the exploitation-phase routing matrix $y^*$ is carefully constructed in \eqref{eq:y-support}\textendash\eqref{eq:y-stochastic} to achieve the proper tie-breaking (the feasibility of this construction is shown in Proposition~\ref{prop:construct}). Because of this construction, DEEM simultaneously satisfies a) aggregate capacity constraints  and b) complementary slackness conditions with respect to the prices $p^*$. These properties are key to showing the near-optimality of DEEM for the original capacity-constrained optimization problem \eqref{prob:mixedbandit}.

\item {\em Appropriate learning goals for the Explore phase.} DEEM's tolerance for labeling errors in the Explore phase depends on their impact on payoffs during exploitation.
For instance, in our example, suppose that at the end of the Explore phase, the algorithm mislabels a Programmer as an All-rounder.  This has a dire impact on payoffs: the Programmer is then assigned to Design or Mixed jobs in the Exploit phase (since $\cJ(\textup{All-rounder})=\{\textup{Design}, \textup{Mixed}\}$), neither of which is optimal for Programmers. These errors lead to a constant regret relative to the optimal externality-adjusted payoffs per unit mass of workers per time step.

In fact, in our example, every other kind of mislabeling is also similarly problematic, with one exception: If an All-rounder is labeled as a Designer, this is acceptable, since the worker will then be assigned Design jobs during exploitation, but $\textup{Design} \in \cJ(\textup{All-rounder})$; i.e., $\cJ(\textup{Designer})\subset \cJ(\textup{All-rounder})$. To ensure that capacity constraints do not pose a problem, we nevertheless ensure that even such ``acceptable'' mislabeling occurs for only $\Theta(1/\log N)$ fraction of workers.

This motivates our definition of the sets $\Str(i)$ for each worker type $i$: 
this is the set of types that $i$ must be distinguished from with high confidence, or ``strongly'' distinguished. 
In particular, if $\cJ(i)\setminus \cJ(i')\neq \phi$ then $i'\in\Str(i)$. The target error probability for types in $\Str(i)$ is chosen to be $1/N$ (see the second condition in line \ref{pc:label-condition} of Figure \ref{fig:def-DEEM}): if we choose a much larger target, we will incur a relatively large expected regret during exploitation due to misclassification; if we choose a smaller target, the Explore phase will be unnecessarily long, and we will thus incur a relatively large regret in the Explore phase. In our example $\Str(\textup{Programmer}) = \{\textup{Designer}, \textup{All-rounder}\}$, $\Str(\textup{Designer})=\{\textup{Programmer}\}$ and $\Str(\textup{All-rounder})=\{\textup{Programmer}, \textup{Designer}\}$. For every other type $i'\notin\Str(i)$, we have $\cJ(i)\subseteq\cJ(i')$, and we ``weakly'' distinguish $i$ from such $i'$, with a target misclassification probability of $1/\log N$ (see the first condition in line 15 of Figure~\ref{fig:def-DEEM}).

 \item {\em Minimizing regret during Confirmation:} After quickly obtaining a fairly confident estimate $i$ for the worker type using the Guessing mode, DEEM attempts to distinguish $i$ from all types in $\Str(i)$ with high confidence using the Confirmation mode. Regret relative to the largest possible externality-adjusted payoff $U(i) = \max_j (A(i,j) - p^*_j)$ for worker type $i$ may be inevitable in this process. 
     A key feature underlying the near-optimality of DEEM is that it tries to minimize the regret incurred during Confirmation.

For instance, the only job type that is optimal for Programmers; i.e., Programming, does not allow the policy to distinguish between Programmers and All-rounders. Thus if the guessed worker type is Programmer, then since All-rounder $\in \Str(\textup{Programmer})$, during Confirmation DEEM must assign the worker either Design or Mixed jobs in order to make sure she is not an All-rounder. Thus confirming the guess necessitates regret in the event that the guess is correct.

To minimize this regret, DEEM samples job types during Confirmation of worker type $i$ from the carefully chosen distribution $\alpha(i)$ defined in Figure~\ref{fig:def-alphai-ystar}. For a job type distribution $\alpha\in\Delta(\cJ)$, for workers of true type $i$, the smallest value of the log posterior odds, $\min_{i'\in \Str(i)}\log\lambda(i)/\lambda(i')$, increases at an expected rate of $\min_{i'\in \Str(i)}\sum_{j\in \cS}\alpha_j\KL(i,i'|j)$.
In order to confirm $i$ against worker types in $\Str(i)$ with a probability of error of $1/N$, the smallest value of log posterior odds needs to cross the threshold of $\log N$ (see second condition in line 15 in Figure~\ref{fig:def-DEEM}). The expected number of jobs needed to cross this threshold is $\log N/(\min_{i'\in \Str(i)}\sum_{j\in \cS}\alpha_j\KL(i,i'|j))$. Hence, the expected externality-adjusted regret incurred during Confirmation is 
$$ \frac{\log N}{N} \frac{\sum_{j\in \cS} \alpha_j\big(U(i)-[A(i,j)-p^*_j]\big)}{\min_{i'\in \Str(i)}\sum_{j\in \cS}\alpha_j \KL(i,i'|j)},$$
{where the factor $N$ in the denominator arises from the worker's lifetime, since regret is defined per period.}
DEEM choses a policy $\alpha(i)$ that minimizes this quantity, which in effect minimizes the ratio of the rate of regret accumulation and the rate of learning, or, informally, the ``regret per unit of learning.'' This minimal regret of $C(i)\frac{\log N}{N}$ to leading order
is inevitable per unit mass of workers of type $i$ for any optimal policy that solves $\eqref{eq:optpricedlearn2}$ for a large $N$, where
\begin{align}
C(i)&\triangleq \min_{\alpha\in\Delta(\cJ)} \frac{\sum_{j\in \cS} \alpha_j\big(U(i)-[A(i,j)-p^*_j]\big)}{\min_{i'\in \Str(i)}\sum_{j\in \cS}\alpha_j \KL(i,i'|j)}. \label{def:ci}
\end{align}

In Appendix~\ref{apx:alpha}, we show that the above optimization problem reduces to a linear program. In general, this problem can have several optimal solutions, denoted by the set $\mathcal{A}(i)$. Of these, DEEM chooses the one that gives the highest learning rate.\footnote{In Appendix \ref{apx:alpha}, we show that this choice is well defined, despite the fact that $\mathcal{A}(i)$ could be an open set.}

In our example, $U(\textup{Programmer}) = 0.5$, and thus
\begin{align}
&[U(\textup{Programmer}) - (A(\textup{Programmer},j)-p^*_j)]_{j\in\cJ} = [\begin{matrix}0 &0.5 &0.4\end{matrix}]^T;\\
&[\KL(\textup{Programmer}, \textup{Designer}|j)]_{j\in\cJ} = [\begin{matrix}D_{\KL}(0.5\Vert 0.3) & D_{\KL}(0.2\Vert 0.8) & D_{\KL}(0.1\Vert 0.2)\end{matrix}]^T;\\
&[\KL(\textup{Programmer}, \textup{All-rounder}|j)]_{j\in\cJ}= [\begin{matrix}0 & D_{\KL}(0.2\Vert 0.8) & D_{\KL}(0.1\Vert 0.6)\end{matrix}]^T\, .
\end{align}
In this case, one can show that $\mathcal{A}(\textup{Programmer}) = \{(1-\epsilon, \epsilon, 0): \epsilon\in (0,1]\}$ and $C(\textup{Programmer}) = 0.5/D_{\KL}(0.2\Vert 0.8)\approx 0.6011$. Of these solutions, DEEM picks $\alpha (\textup{Programmer}) = (0,1,0)$ as per \eqref{def:alphai-fastest-learning}, since this distribution is the quickest in confirming a Programmer while possessing the optimal regret per unit of learning (this $\alpha$ achieves the log posterior targets for $i' = $ Designer and for $i' = $ All-rounder in the same expected time).

On the other hand, Mixed jobs are optimal for All-rounders and further allow All-rounders to be distinguished from both Designers and Programmers (both these types are in $\Str(\textup{All-rounder})$). 
Thus no regret needs to be incurred while confirming an All-rounder. Recall that Design jobs are also optimal for All-rounders. It is straightforward to verify that $\mathcal{A}(\textup{All-rounder}) =  \{(0,1-\epsilon, \epsilon): \epsilon\in (0,1]\}$, and $C(\textup{All-rounder}) = 0$. The choice of $\alpha(\textup{All-rounder}) = (0,0,1)$ results in the highest learning rate as per \eqref{def:alphai-fastest-learning}. (One can similarly verify that $C(\textup{Designer}) = 0$ and $\alpha(\textup{Designer}) = (0,1,0)$.)

The distinction in $C(i)$ for $i = \textup{Programmer}$ and $i' = \textup{All-rounder}$ is fundamental, and motivates the following definition.
\begin{definition}
Consider a worker type $i$. Suppose that there exists another type $i' \in \cI \backslash \{i\}$ such that $A(i,j) = A(i',j)$ for all $j \in \cS(i)$, and $i'\in\Str(i) \, \Leftrightarrow \, \cJ(i) \not \subseteq \cJ(i')$,
where $\cJ(i)$ and $\Str(i)$ are as defined in \eqref{eq:Str-i}.
Then we say that the ordered pair $(i,i')$ is a {\em difficult type pair}.\footnote{A similar definition also appears in \cite{agrawal1989asymptotically}; the modification here is that the sets $\cS(i)$ are defined with respect to externality-adjusted payoffs to account for capacity constraints.}
\label{def:diff_pair}
\end{definition}
Note that $C(i) >0$ if and only if there is some other $i'$ such that $(i,i')$ is a difficult type pair. In this case, there is a non-trivial (asymptotic) trade-off between myopic payoffs and learning, and a regret of $\Omega(\log N)$ is necessary per unit mass of workers. In the above example, (Programmer, All-rounder) is a difficult type pair.

\end{enumerate}

\section{Main result}
\label{sec:mainresult}

Our main result is the following theorem.  In particular, we prove a lower bound on the regret of any policy, and show that DEEM (essentially) achieves this lower bound. For the result and discussion below, we denote $\textup{DEEM}_N$ to be the instantiation of DEEM for a given $N$.

\begin{theorem}
\label{thm:mainresult}
Fix $(\rho, \mu, A )$ such that (a) no two rows of $A$ are identical and (b) the generalized imbalance condition holds. 
Let $C\triangleq \sum_{i\in\cC}\rho_iC(i)$ for $C(i)$ as defined in \eqref{def:ci}. If $C>0$, we have:
\begin{enumerate}
\item (Lower bound) For any sequence of WHO policies $(\pi_N)_{N\in \mathbb{N}}$, indexed by worker lifetime $N$, that are feasible for \eqref{prob:mixedbandit}\textendash\eqref{eq:feasibility}, we have
\begin{equation}
\label{eq:lower}
\liminf_{N\rightarrow \infty} \frac{N}{\log N}(V^*-W^N(\pi_N)) \geq C \, .
\end{equation}
\item (Upper bound) There exists $N_0< \infty$ such that for all $N \geq N_0$, $\textup{DEEM}_N$ is feasible for \eqref{prob:mixedbandit}\textendash\eqref{eq:feasibility} with
\begin{equation}
\label{eq:upper}
\limsup_{N\rightarrow \infty} \frac{N}{\log N}(V^*-W^N(\textproc{DEEM}_N)) \leq C \, .
\end{equation}
\end{enumerate}
If instead $C= 0$, then there exists $N_0< \infty$ such that for all $N \geq N_0$, $\textup{DEEM}_N$ is feasible for \eqref{prob:mixedbandit}\textendash\eqref{eq:feasibility} with
\begin{equation}
\label{eq:upper}
\limsup_{N\rightarrow \infty} \frac{N}{\log\log N}(V^*-W^N(\textproc{DEEM}_N)) \leq K,\,
\end{equation}
where $K=K(\rho, \mu, A)\in [0,\infty)$ is some constant.
\end{theorem}
Recall that $C(i)$ and hence $C$ depend on the primitives of the problem; i.e., $(\rho, \mu, A)$, and that $C(i) >0$ if and only if there exists $i' \neq i$ such that $(i,i')$ is a difficult type pair. We immediately deduce that $C>0$ if and only if there is a difficult pair of worker types $(i,i')$.
\begin{remark}
The constant $C$ in Theorem~\ref{thm:mainresult} is strictly positive if and only if there exists at least one difficult pair of worker types $(i,i')$, 
i.e., a pair of distinct worker types $(i,i')$ such that $A(i,j) = A(i',j) \; \forall j \in \cJ(i)$ and $\cJ(i) \not \subseteq \cJ(i')$, where
 \begin{align}
   \cS(\tilde{i})\triangleq \arg\max_{j\in \cS} A(\tilde{i},j)-p^*_j  \qquad \forall \tilde{i} \in \cI \,
 \end{align}
and $p^*$ are the shadow prices for the capacity constraints \eqref{eq:opt2} in the problem with known worker types \eqref{eq:opt1}\textendash\eqref{eq:opt3}.
{Theorem~\ref{thm:mainresult} implies that the smallest achievable regret is $\Theta(\frac{\log N}{N})$ (``large'') if there is a difficult type pair, while one can achieve a regret of $\textup{O}\big(\frac{\log \log N}{N}\big )$ (``small'') if there is no difficult type pair.}
\label{rem:main-result-cases-from-primitives}
\end{remark}
Below in Section \ref{subsec:difficulty-with-many-skills} we show that in a setting where workers are heterogeneous along more than one skill dimension, and some job type allows one to distinguish only a subset of skills, ``many'' problem instances contain such difficult type pairs; i.e., there is often a non-trivial (asymptotic) trade-off between myopic payoffs and learning.

%% file: proofsketch.tex
\subsection{Proof sketch}
\label{sec:proofsketch}
The proof of Theorem \ref{thm:mainresult} can be found in Appendix~\ref{apx:mainproof}. Here we present a sketch. The critical ingredient in the proof is the following relaxed optimization problem in which there are no capacity constraints, but capacity violations are charged non-negative prices $p^*$ from the optimization problem \eqref{eq:opt1}\textendash\eqref{eq:opt3} with known worker types.
\begin{eqnarray}
W^N_{p^*}&=&\max_{x \in \cX^N} \sum_{i\in\cC}\rho_i\sum_{j\in \cS} x(i,j) A(i,j)- \sum_{j\in \cS}p^*_j\Big [\sum_{i\in \cC}\rho_ix(i,j)-\mu_j\Big ].\label{approxpricedlearn}
\end{eqnarray}

\paragraph{Lower bound on regret.}
If $C>0$ (i.e., if there is at least one difficult pair of worker types; see~Section \ref{sec:mainresult}), there is a lower bound on the regret relative to $V^*$ under any policy in this problem. This result follows directly from Theorem 3.1 in \cite{agrawal1989asymptotically}:
$$\liminf_{N\rightarrow \infty} \frac{N}{\log N}(V^*-W^N_{p^*}) \geq C.$$
By a standard duality argument, we know that $W^N\leq W^N_{p^*}$, and hence this bound holds for $W^N$ as well, yielding the lower bound on regret in our original problem \eqref{prob:mixedbandit}. This is shown in Proposition~\ref{prop:lowerbound} in the Appendix~\ref{apx:mainproof}.

\paragraph{Upper bound on regret.}
There are two key steps in proving that $\textup{DEEM}_N$ is feasible for problem \eqref{prob:mixedbandit}--\eqref{eq:feasibility}, and that
$$\limsup_{N\rightarrow \infty} \frac{N}{\log N}(V^*-W^N(\textproc{DEEM}_N)) \leq C.$$
\begin{enumerate}
\item First, in Proposition~\ref{prop:upperbound}, we show that DEEM, with an arbitrary exploitation-phase routing matrix $y^*$ supported on $\cS(i)$ for each $i \in \cI$,  achieves near-optimal performance for the vanilla multi-armed bandit problem \eqref{approxpricedlearn}. Formally, if (with some abuse of notation) we let $W^N_{p^*}(\pi)$ denote the value attained by a policy $\pi$ in problem \eqref{approxpricedlearn}, i.e.,
$$W^N_{p^*}(\pi) = \sum_{i\in\cC}\rho_i\sum_{j\in \cS} x_{\pi}(i,j) A(i,j)- \sum_{j\in \cS}p^*_j\Big [\sum_{i\in \cC}\rho_ix_{\pi}(i,j)-\mu_j\Big ],$$
 then {we can show that}\footnote{\cite{agrawal1989asymptotically} proves a similar performance guarantee for a slightly different policy.}
$$\limsup_{N\rightarrow \infty} \frac{N}{\log N}(V^*-W^N_{p^*}(\textproc{DEEM}_N)) \leq C \qquad \textup{if }  C>0\, $$
and
$$\limsup_{N\rightarrow \infty} \frac{N}{\log \log N}(V^*-W^N_{p^*}( \textup{DEEM}_N)) \leq K \qquad \textup{if } C=0,$$
for some constant $K=K(\rho, \mu, A)\in [0,\infty)$. Thus, we have
$$\lim_{N\rightarrow \infty}\frac{N}{\log N}(W^N_{p^*}-W^N_{p^*}(\textup{DEEM}_N)) =0,$$
and hence, DEEM is near-optimal in problem \eqref{approxpricedlearn}. The proof of Proposition~\ref{prop:upperbound} utilizes two technical results presented as Lemma~\ref{lma:core} and Lemma~\ref{lma:driftreversal} in the appendix.

\item Finally, in Proposition~\ref{prop:mainupperbound}, we prove that if $\textproc{DEEM}_N$ uses an Exploit-phase routing matrix $y^*$ (that depends on $N$) that satisfies conditions \eqref{eq:y-support}\textendash\eqref{eq:y-stochastic} (recall that the existence of such a matrix for a large enough $N$ was shown in Proposition~\ref{prop:construct}), then $W^N(\textup{DEEM}_N) = W^N_{p^*}(\textup{DEEM}_N)$. In conjunction with Proposition~\ref{prop:upperbound}, this yields our upper bound on regret.  The proof of Proposition~\ref{prop:mainupperbound} crucially utilizes the fact that our specific choice of the exploitation-phase routing matrix $y^*$ ensures that the routing matrix $x_\tDEEM$, as defined in \eqref{eq:x_pi}, satisfies the following conditions:
\begin{enumerate}
\item (Complementary slackness) $\sum_{i\in \cC}\rho_ix_{\textup{\tiny DEEM}_N}(i,j)-\mu_j=0$ for all $j$ such that $p^*_j>0$, and
\item (Feasibility) $\sum_{i\in \cC}\rho_ix_{\tDEEM}(i,j)-\mu_j\leq 0$ for all other $j \in \cJ$.
\end{enumerate}
{This shows that $\textup{DEEM}_N$ with this choice of $y^*$ in the exploitation phase is feasible for problem \eqref{prob:mixedbandit}--\eqref{eq:feasibility} and the complementarity slackness property implies that $W^N(\textup{DEEM}_N) = W^N_{p^*}(\textup{DEEM}_N)$.

Note that this result crucially relies on Proposition~\ref{prop:construct}, which shows the existence of the routing matrix $y^*$ with the required properties. The basic idea behind this construction is as follows. At the end of the exploration phase of DEEM, the correct label of the worker is learned with a confidence of at least $(1-\mathrm{o}(1))$. This fact, coupled with the generalized imbalance condition (leading to flexibility in modifying $x^*$), is sufficient to ensure that an appropriate and feasible choice of $y^*=x^* + \mathrm{o}(1)$ will correct the deviations from $x^*$ in terms of the capacity utilizations of job types that were fully utilized under $x^*$, i.e., with $p^*_j> 0$ (these deviations arise because of the short exploration phase, and because of the infrequent cases in which exploitation is based on an incorrect worker label coming out of the exploration phase).}

\end{enumerate}

%% file: difficulty-with-many-skills.tex
\subsection{Difficult type pairs occur frequently with multiple skill dimensions}
\label{subsec:difficulty-with-many-skills}

In Section \ref{subsec:example} we saw a simple and natural example with two skill dimensions that included a difficult type pair.
In this subsection, we show that difficult type pairs occur frequently when there is more than one skill dimension. 

Again taking Upwork as an example, one of the categories on this platform is ``Web Development'' and there are a number of relevant ``Skills,'' e.g., HTML development, WordPress, Python, JavaScript, Payment Gateway Integration, and Web Design. A typical web developer has a subset of these skills.  For each of these skill {\em dimensions} there may be distinct levels, e.g., missing/inexpert/expert. The platform may have a prior but learns about a developer's skill along a certain dimension chiefly by observing outcomes. Moreover, job listings are heterogeneous in terms of the relevant skills. For example, a project requiring the skills Web Design and WordPress would reveal some information about these skill dimensions, but not whether the developer is an expert in HTML development. In this subsection, we formalize a special case of our model with multiple skill dimensions, and show that difficult type pairs occur frequently (i.e., in many instances) in such a setup.

Suppose that there is a set $\Sk$ of skill dimensions, and a worker type $i$ is a tuple $i = (i_1, i_2, ..., i_{|\Sk|})$, where each $i_s \in \cI_s$ captures the skill level of the worker in dimension $s$. Here $\cI_s$ is a finite set for each $s$, and the set of worker types $\cI \subseteq \cI_1 \times \cI_2 \times \dots \times \cI_{|\Sk|}$. For each job type $j\in \cJ$, there is a non-empty subset of relevant skill dimensions $\Sk_j \subseteq \Sk$.  We then require the expected payoff $A(i,j)$ to be 
\emph{consistent} with $\Sk_j$; i.e., there must be a function $f_j: \prod_{s \in \Sk_j}\cI_s\rightarrow [0,1]$ such that $A(i, j)= f((i_s)_{s \in \Sk_j})$.

Without loss of generality, assume that $\cup_{j\in \cJ} \Sk_j = \Sk$, since if this were not true, we could safely ignore the (irrelevant) skills in $\Sk \backslash \big( \cup_{j\in \cJ} \Sk_j \big )$, given that they do not matter for any job type (multiple worker types that differ only in irrelevant skills can be collapsed into a single worker type). Then, the following assumption on $(\cI, \cJ, \Sk, (\Sk_j)_{j\in \cJ})$ is very plausible if there are at least two skill dimensions $|\Sk| \geq 2$.
\begin{assumption}
\label{ass:non-triviality-of-learning}
There is a job type $j \in \cJ$ and a pair of distinct worker types $i \in \cI$ and $i' \in \cI \backslash \{i\}$ such that
 (i) a strict subset $\Sk_j \subset \Sk$ of skills is relevant to job type $j$ and
  (ii) worker types $i$ and $i'$ differ only in skill dimensions in $\Sk \backslash \Sk_j$.
\end{assumption}
\noindent The assumption implies that job type $j$ cannot distinguish between worker types $i$ and $i'$.

We establish the following result, showing that under Assumption \ref{ass:non-triviality-of-learning}, many instances have difficult type pairs. In particular, difficult type pairs \emph{do not} occur only on a knife edge.



\begin{proposition}
\label{prop:difficulty-frequent-with-multiple-skills}
Fix $\cI$, $\cJ$, $\Sk$, and $(\Sk_j)_{j \in \cJ}$ such that Assumption \ref{ass:non-triviality-of-learning} holds and $|\cI| \geq 3$. Also fix capacity constraints $\mu = (\mu_j)_{j \in \cJ}$ such that\footnote{This technical requirement says that the total arrival rate of jobs should not be exactly equal to the total mass of workers present. We also assume that $|\cI|  \geq 3$. These assumptions help us get unique shadow prices $p^*$ in our construction, simplifying the analysis, though we expect this proposition can be stated and proved without these assumptions.} $\sum_{j \in \cJ} \mu_j \neq 1$. Consider the set of possible instances
 \begin{align}
  \cP= \{ (\rho, A) \in \Delta_{|\cI|} \times [0,1]^{|\cI| |\cJ|} : \textup{$A(\cdot, j)$ is consistent with $\Sk_j$ for all $j \in \cJ$} \} \, .
 \end{align}
 Note that the distribution over worker types $\rho$ has $|\cI|-1$ degrees of freedom, and for each $j \in \cJ$ the $j$-th column of the payoff matrix $A(\cdot, j)$ has $|\cI_{\Sk_j}|$ degrees of freedom, where
\begin{align}
\cI_{\Sk_j}  = \{(i_s)_{s \in \Sk_j}: \exists i' \in \cI \textup{ s.t. } i'_s = i_s \forall s \in \Sk_j \} \,
\end{align}
 is the set of distinct worker type classes that arise when worker types are projected onto $\Sk_j$. Then the subset of instances 
 \begin{align}
 \cPdiff = \{(\rho, A) \in \cP: \textup{there is a difficult type pair}\}
 \end{align}
 has full dimension $|\cI|-1 + \sum_{j \in \cJ}|\cI_{\Sk_j}|$, equal to the dimension of $\cP$.
\end{proposition}

Proposition~\ref{prop:difficulty-frequent-with-multiple-skills} is proved in Appendix~\ref{apx:multiskills}. Starting with Assumption \ref{ass:non-triviality-of-learning} and the corresponding $i \in \cI$, $i' \in \cI$, and $j \in \cJ$, the main idea is to construct an instance such that in any solution in the full information setting, worker type $i$ will be matched exclusively to job type $j$, whereas type $i'$ will be matched exclusively to jobs types other than $j$. We then show that this remains true in a neighborhood of the given set of parameters.

%% file: simulations.tex
\section{Practical considerations and a heuristic}\label{sec:deemplus}

Although DEEM minimizes the leading-order term of regret as $N \to \infty$ in our continuum model, it is unclear how to obtain from it a policy that performs well in practice. The first concern is that it is defined in the context of the continuum model, whereas, in reality, we have finite arrivals of workers and jobs over time. The second concern is that, while DEEM is expected to perform well when $N$ is large, it may not perform well when $N$ is relatively small, which is the case in many practical settings. In Appendix~\ref{sec:deemdiscrete}, we present a fairly straightforward translation of DEEM to a discrete setting (with finite arrivals of workers and jobs at each time), which is a close analogue to our continuum setting.

Although this policy is simple and effectively addresses the first concern, it falls short of addressing the second concern: in fact we find that it performs well {\it only} when $N$ is large.
A main cause of this deficiency is that DEEM-discrete does not learn the right shadow prices for the constraints.
In fact, simulations tell us that for practical values of $N$, if we use the shadow prices $p^*$ for the capacity constraints from the static planning problem \eqref{eq:opt1}\textendash\eqref{eq:opt3} to define the Explore phase for each worker as in DEEM, then, in many instances, a subset of job types are fully (or substantially) utilized by workers in the Explore phase itself, and are thus unavailable for workers in the Exploit phase (see Remark~\ref{rem:deemdisc} in Appendix~\ref{sec:deemdiscrete}). In particular, {this suggests that} in many instances, $p^*$ constitutes a poor estimate of the shadow prices for finite $N$. Moreover, this is the case even for a large system with many workers and jobs.  Thus, an important practical issue is to learn the right shadow prices for the capacity constraints.

This is where we can leverage a key practical feature of most real platforms: jobs typically queue up rather than get assigned instantaneously.  In such systems, the ``inventory'' of jobs is dynamically evolving, and therefore the shadow prices should be responsive to the inventory of the system.  ``Backpressure'' \cite{tassiulas1990stability} is the  method {of choice} for learning shadow prices in a dynamic queueing context, by estimating shadow prices based on queue lengths. This prompts us to adopt it for the DEEM-inspired practical heuristic we present in this section. Besides learning the right shadow prices, backpressure has the added advantage of seamlessly handling stochasticity in the (finite) arrivals of workers and jobs.

In the remainder of this section, we present our heuristic derived from DEEM called \DEEMplus, which is (1) practically implementable in a market environment with discrete arrivals, and (2) is designed for good small $N$ performance. \DEEMplus is defined in Figures~\ref{fig:def-DEEMplus} and~\ref{fig:def-DEEMplus-cont}. Similar to DEEM, \DEEMplus is a decentralized algorithm that acts individually on each worker and hence is defined at the worker level. \DEEMplus
makes use of dynamic (queue-based) shadow prices $p^q$ as a key input at the worker level.
In our numerical simulations in Section~\ref{sec:simulations}, we will define an environment where pending jobs can queue up and specify a backpressure-like approach to computing queue-based prices. Notably, besides its use of queue-based shadow prices, \DEEMplus incorporates a few key modifications to the Explore and Exploit phases of DEEM that improve performance when $N$ is small. Below, we discuss the important features of \DEEMplus in detail.

\begin{figure}
\fbox{{\footnotesize\begin{minipage}{\textwidth}
\begin{center} {\vspace{0.1in}\normalsize \bf \DEEMplus: A practical WHO heuristic for finite $N$\vspace{0.1in}} \end{center}
\textbf{Input parameters:} $\cI$, $\cJ$, $A$, $\rho$, $N$, queue-based prices $p^q$.\\
{\bf Pre-compute:}\begin{itemize}
\item The sets $\cJ^q(i)$ defined in \eqref{eqn:opt-jobs-queue-prices} for each worker type $i\in\cI$.
\item The set of worker types $\Str^q(i)$ and $\Weak^q(i)$ defined in Definition~\ref{def-weak} (Section~\ref{subsec:finiteN}) for all $i \in \cI$.
\item The maximal externality-adjusted payoffs $U^q(i)$ for all $i\in\cI$ defined in \eqref{def-max-utility-queue-prices} and the maximal {per-job} mislabeling regrets $R(i,i')$ for all $i,\,i'\in \cI$ defined in \eqref{def-onestepreg-str}--\eqref{def-onestepreg-weak}.
\end{itemize}
\medskip
\hrule
\medskip
\begin{algorithmic}[1]
\LineComment{Main Routine}
\Procedure{\DEEMplus}{} \Comment{Acts independently on each worker, over her lifetime, from arrival to departure}

\medskip
\LineComment{Initialization:}
\State $\lambda(i) \gets \rho_i$ for all $i\in \cI$ \Comment{The un-normalized posterior probabilities; initialized to the prior}
\State $\MAP \gets \argmax_{i\in\cI} \lambda(i)$ \Comment{Initialization of the MAP estimate}
\State $Label \gets \emptyset$ \Comment{Worker label; initially unassigned, denoted by $\emptyset$}
\State $k \gets 0$ \Comment{{ Number of jobs} the worker has performed in the system}
\medskip

\LineComment{Explore phase:}
\While{$Label = \emptyset$ and $k<N$}
	\State Assign job type $j_k\sim$  \Call{Explore}{$N$, $\lambda$, \MAP} \Comment{{At the next job opportunity}}
	\State Observe payoff $r_k$
	\State $\lambda(i) \gets \lambda(i) \times (A(i,j_k)\mathbf{1}_{\{r_k=1\}}+(1-A(i,j_k))\mathbf{1}_{\{r_k=0\}})$, for all $i \in \cI$
	\State $\MAP \gets \argmax_{i\in\cI} \lambda(i)$
	\State \If{$\min_{i\in \Str^q(\MAP)}\frac{\lambda(\MAP)}{\lambda(i)R(\MAP,i)} \geq N$} \label{pc:label-condition}\Comment{if Confirmation is complete}
		\State $Label \gets \MAP$  \Comment{Worker label assigned. Will cause while loop to exit.}
		\EndIf 
	\State $k \gets k+1$
\EndWhile
\medskip

\LineComment{Exploit phase:}
\While{$k<N$}
 		\State Assign job type $j_k=$ \Call{Exploit}{$\lambda$} \Comment{{At the next job opportunity}}
 	\State $k \gets k+1$
 \EndWhile
 \EndProcedure
 \medskip
 \hrule
 \medskip
 \LineComment{Functions}
\Function{Explore}{$N$, $\lambda$, $\MAP$ }
\If{$\min_{i\in \Str^q(\MAP)\cup\Weak^q(\MAP)}\frac{\lambda(\MAP)}{\lambda(i)R(\MAP,i)} <\log N$} \Comment{if MAP estimate is noisy}
		\State $dist \gets\alpha_{\textup{guess}}(\lambda)$ \Comment{Guessing}
\Else \Comment{MAP estimate is somewhat confident}
		\State $dist \gets\alpha_{\textup{conf}}(\MAP,\,\lambda)$ \label{pc:confirmation} \Comment{Confirmation}
\EndIf
\State {\bf return} $dist$
\EndFunction

\medskip

\Function{Exploit}{$\lambda$}
\State $j^* = \argmax_{j\in \cJ}\sum_{i\in\cI} \lambda(i)\big[A(i,j)-p^q_j\big]$ \Comment{Greedy}\label{pc:exploit-sample}
	\State {\bf return} $j^*$
\EndFunction

\algstore{DEEMplus}
\end{algorithmic}
\end{minipage}}}
\linespread{1}
\caption{Definition of \DEEMplus. The prices $p^q$ depend on the queue lengths of the different job types, as discussed in Section~\ref{subsec:queueprices}, and formally defined in Appendix~\ref{sec:pd-control}.}
\label{fig:def-DEEMplus}
\end{figure}

\edef\myindent{\the\parindent}
\begin{figure}
\fbox{{\footnotesize\begin{minipage}{\textwidth}\setlength{\parindent}{\myindent}
\begin{algorithmic}[1]
\algrestore{DEEMplus}
\Function{$\alpha_{\textup{guess}}$}{$\lambda$}
\State  For each $j\in\cJ$, define \Comment{Thompson sampling}
\begin{align}
\alpha_j &= \bigg(\sum_{i\in\cI}\lambda(i)\frac{\ind_{j\in\cJ^q(i)}}{|\cJ^q(i)|}\bigg)\bigg/ \sum_{i\in\cI}\lambda(i) \nonumber\end{align}

\State {\bf return} $\alpha=(\alpha_j)_{j\in\cJ}$
\EndFunction
\medskip
\Function{$\alpha_{\textup{conf}}$}{\MAP, $\lambda$}
\State Define the set
\begin{align}
\mathcal{A} &= \argmin_{\alpha\in \Delta(\cS)}\bigg\{\bigg[\sum_{i\in\cI} \lambda(i)\sum_{j\in \cS} \alpha_j\big(U^q(i)-[A(i,j)-p^q_j]\big)\bigg]\bigg[\max_{i'\in\Str^q(\MAP)} \frac{\log N + \log R(\MAP,i')}{\sum_{j\in \cS}\alpha_j \KL(\MAP,i'|j)}\bigg]\bigg\}\nonumber
\end{align}
\State Choose any
\begin{align}
\alpha &\in \argmin_{\alpha'\in\mathcal{A}}\max_{i\in \Str^q(\MAP)}\frac{\log N + \log R(\MAP,i)}{\sum_{j\in \cS}\alpha'_j\KL(\MAP,i|j)}\nonumber
\end{align}
\State {\bf return} $\alpha$
\EndFunction
\end{algorithmic}
\end{minipage}}}
\linespread{1}
\caption{Definition of \DEEMplus (continued). }
\label{fig:def-DEEMplus-cont}
\end{figure}


\subsection{Dynamic queue-based shadow prices}\label{subsec:queueprices}

\DEEMplus uses dynamic shadow prices computed using job queue length information. 
Each new worker corresponds to a multi-armed bandit problem that gets ``adjusted'' under \DEEMplus by the instantaneous prices in the market when the worker arrives, to account for the externalities due to the presence of the capacity constraints. These prices remain fixed throughout the lifetime of the worker; this is analogous to having $p^*$ be the fixed prices under DEEM.
We refer to these instantaneous queue-based prices as $p^q$ in the definition of \DEEMplus in Figures~\ref{fig:def-DEEMplus} and~\ref{fig:def-DEEMplus-cont}.

Suitably defined prices as functions of job queue lengths can be used as a feedback control mechanism to stabilize the queue lengths, 
and hence stabilize the prices themselves. (Such ``backpressure'' pricing based on queue lengths is commonly utilized in designing distributed algorithms for resource allocation, e.g., for congestion control in communication networks \cite{shakkottai2008network,srikant2012mathematics}.) For our simulations in Section~\ref{sec:simulations}, we use a proportional-derivative (PD) control-based definition of these prices as an illustrative example (the technical details of the controller design are presented in Appendix~\ref{sec:pd-control}). 
The instantaneous shadow price for each job type is the sum of a ``proportional'' term and a ``derivative'' term.
The ``proportional'' term is an affine function of the queue length with a negative coefficient for the queue length. The idea is that if the queue length is small then this indicates that the job type is in high demand and hence the price for this type should be high. The derivative term is proportional to the negative of the recent rate of change of the queue length, to counteract rapid changes in queue length and prevent oscillatory behavior.  

Our definition of these prices does not depend on the vector of job arrival rates $\mu$. It also obviates the need to explicitly compute $y^*$ in the exploitation phase of DEEM; instead the Exploit phase can be implemented by allocating optimally for each worker given the queue-based prices that were supplied to the worker when the worker arrived. Natural (small) fluctuations in prices that arise in the process of stabilizing queue lengths result in appropriate tie-breaking in allocation (tie-breaking is typically needed; see~Fact \ref{fact:tiebreaking_necessary}), with randomization occurring \emph{across workers}. The resulting stability of the queue lengths implies that the capacity constraints are satisfied.

\subsection{Optimizing the Explore phase for finite $N$} \label{subsec:finiteN}

\DEEMplus shares the same explore-then-exploit structure as DEEM, but with a few changes to the exploration phase to optimize learning for finite $N$. These modifications lead to a significant improvement in performance. In Appendix~\ref{apx:dpconfirm}, we show that the regret estimate that the Explore phase of DEEM minimizes is a poor estimate of the true regret for small $N$. By contrast, the regret estimate that the Explore phase of \DEEMplus minimizes captures the true regret reasonably well even for small $N$, hence explaining the improvement in performance produced by these modifications.

The changes to the Explore phase in \DEEMplus are discussed in detail below.


\begin{enumerate}
\item  {\bf Refining learning goals.} Recall that in DEEM, we achieve a $1/\log N$ probability of error of misclassifying $i'$ as $i$ even if $\cJ(i) \subseteq \cJ(i')$. 
    In \DEEMplus, because we use queue-based prices, the counterparts of the sets $\cJ(i)$ are the sets
\begin{equation}
\cJ^q(i)\triangleq\arg\max_{j\in \cS} A(i,j)-p^q_j \label{eqn:opt-jobs-queue-prices}
\end{equation}
for each $i\in\cI$. These are sets of job types that are optimal for the different worker types under the queue-based prices $p^q$. Because of natural fluctuations in the prices across jobs, these sets are typically singletons and hence for any two types $i$ and $i'$, either $\cJ^q(i)\cap \cJ^q(i')= \emptyset$ or $\cJ^q(i)=\cJ^q(i')$. In the latter case, it would appear wasteful to achieve a $1/\log N$ probability of error of misclassifying $i'$ as $i$, except in the case
where the optimal confirmation policies for the two types are different, since confirming using a suboptimal job sampling policy leads to an increase in leading-order regret.\footnote{In DEEM, we made this distinction even when the confirmation policies for the two types are the same since leading-order regret is unaffected, and because it allowed us to leverage the solution to the problem with known worker types to obtain a policy with provable guarantees that satisfies capacity constraints 
(see the construction of $y^*$ in the proof of Proposition~\ref{prop:construct} in Appendix~\ref{apx:mainproof}). The resulting key property is that shadow prices remain close to those under known worker types, and, happily, we observe that this latter property holds under \DEEMplus in our numerical experiments (see Table~\ref{tbl:p} in Section~\ref{sec:simulations} below).}
Thus, in \DEEMplus, for each worker type $i$, we define the sets of types it needs to be distinguished from as follows (the definition does not assume that the sets $\cJ^q(i)$ are singletons).

\begin{definition} \label{def-weak}For each $i,\,i'\in\cI$,
\begin{enumerate}
\item $\Str^q(i) \triangleq \{ i': \cJ^q(i) \setminus \cJ^q(i') \neq \emptyset \}$.
\item $\Weak^q(i) \triangleq \{ i': \cJ^q(i) \subseteq \cJ^q(i') \textup{ and } \alpha(i,\cI, \cJ, A, p^q, \Str^q(i)) \neq \alpha(i',\cI, \cJ, A, p^q, \Str^q(i'))\}$, where the function $\alpha$ is defined as in Figure~\ref{fig:def-alphai-ystar}.
\end{enumerate}
\end{definition}
In the Guessing mode, \DEEMplus explicitly distinguishes $i$ only from types in $\Str^q(i)\cup\Weak^q(i)$ (see line 28 and the first condition in line 15 in Figure~\ref{fig:def-DEEMplus}). If, at some point in the Explore phase, type $i$ has been confirmed against all types in $\Str^q(i)$, then the algorithm can safely proceed to the Exploit phase even if $i$ has not been weakly distinguished from some $i'\in\Weak^q(i)$; the latter distinction is rendered unnecessary in this case.  This change accounts for the difference in the conditions for transitioning from Explore to Exploit in line 15 of DEEM versus line 15 of \DEEMplus.

Next, recall that DEEM tries to achieve a probability $1/N$ of misclassifying a worker of type $i'$ as type $i$ for any $i\in\Str(i)$.  For small $N$, however, we can do better: the desired probability of error should depend on the largest regret that type $i'$ could incur by performing a job that is optimal for $i$ but not for $i'$: if this regret is very small, then it isn't worth trying to make this distinction with high precision. For each $i\in\cI$, define the maximal externality-adjusted payoffs
\begin{align}
U^q(i)&\triangleq \max_{j \in \cJ}A(i,j)-p^q_j.\label{def-max-utility-queue-prices}
\end{align}
Then we define
\begin{align}
R(i,i') &\triangleq \max_{j \in \cJ^q(i),\,\, j \notin \cJ^q(i')}U^q(i') - [A(i',j)-p^q_j] \qquad \forall \, i'\in \Str^q(i)\, .\label{def-onestepreg-str}
\end{align}
We informally refer to $R(i,i')$ as the maximal per-step regret of mislabeling $i'$ as $i$.
It will further be convenient to (arbitrarily) define\footnote{The definition in \eqref{def-onestepreg-weak} plays a role only in line 28 of $\DEEMplus$ (Figure~\ref{fig:def-DEEMplus}), which contains the criterion for being in Guessing mode.}
\begin{align}
R(i,i') &\triangleq 1 \qquad  \forall i'\in\Weak^q(i) \, .\label{def-onestepreg-weak}
\end{align}
$\DEEMplus$ aims for a probability of mislabeling $i'$ as $i$ of $1/(R(i,i')\log N)$ instead of $1/\log N$ at the end of Guessing for all $i' \in \Str^q(i)\cup\Weak^q(i)$, and of $1/(N R(i,i'))$ instead of $1/N$ at the end of Explore for all $i'\in\Str^q(i)$.  These changes account for the fact that if $R(i,i')$ is small, then we can tolerate a higher probability of error (see line 15 in Figure~\ref{fig:def-DEEMplus}). This modification is especially important in preventing wasteful learning in cases where price fluctuations (necessary for tie-breaking; see Section~\ref{subsec:queueprices}) result in $\cJ^q(i) \subsetneq \cJ^q(i')$ even though $\cJ(i) \subseteq \cJ(i')$.

\item {\bf Incorporating the posterior during Explore.} The final change we propose is to explicitly incorporate the posterior into the Explore phase. First, recall that in DEEM, Guessing involves the naive approach of exploring uniformly at random.  When $N$ is finite, we can reduce regret by instead leveraging the posterior at each round to appropriately allocate learning effort across the different types until we form a good guess of the worker type. In \DEEMplus, we do so by using a price-adjusted version of Thompson sampling (TS), which is a popular Bayesian multi-armed bandit algorithm, in the Guessing mode. Under this algorithm, a worker type is sampled from the posterior distribution and the job type that maximizes the price-adjusted payoff for this sampled worker type is assigned to the worker; if there are multiple such types, one type is chosen uniformly at random.
This is defined in the function $\alpha_{\textup{guess}}$ in Figure~\ref{fig:def-DEEMplus-cont}.

Similarly, the posterior can also be incorporated in minimizing expected regret in the Confirmation mode.  This is captured by the function $\alpha_{\textup{conf}}$ defined in Figure~\ref{fig:def-DEEMplus-cont}. This function returns a sampling distribution that maximizes the product of the expected per-time-step regret relative to the maximal externality-adjusted payoffs (the term in the first parenthesis in the objective on line 44) and the estimated time until the confirmation of the MAP estimate assuming it is correct (the term in the second parenthesis in the objective in line 44). If there are multiple solutions, $\alpha_{\textup{conf}}$ is chosen to be the one with the smallest time until confirmation (line 45). Observe that as $\lambda(\MAP)/\big (\sum_{i\in\cI}\lambda(i)\big )\rightarrow 1$ for the $\MAP$ estimate, the objective function in line 44  in the definition of $\alpha_{\textup{conf}}$ converges to the objective function used while computing $\alpha(\MAP)$ in $\eqref{eqn:alpha}$, modulo the $(\log N+\log R(i,i'))$ factors that simply capture the fact that the learning goals have been adjusted for small $N$.

\end{enumerate}

\subsection{Optimizing the Exploit phase for finite $N$} \label{subsec:exploit-finiteN}

In {\DEEMplus}, we tweak the approach in the Exploit phase of DEEM to improve performance: instead of optimizing the price-adjusted payoff for the confirmed worker label, we can maximize the expected payoff with respect to the posterior distribution, thus accounting for the possibility that we may have confirmed incorrectly. This is reflected in the new definition of the \textproc{Exploit()} function in line 35. We find that this change significantly improves performance.

\section{Simulations}\label{sec:simulations}
In this section, we simulate \DEEMplus in a market environment with queue-based shadow prices and compare its performance against other policies. We consider $350$ instances with $|\cI| = 4 $ worker types and $|\cJ|=3$ job types. In each instance, the arrival rates of all the worker types are identical, i.e., $\rho_i = 0.25$ for all $i\in\cI$, while the arrival rates of the job types are randomly chosen. The choice of these instances is discussed in detail in Appendix~\ref{apx:instances}.
For our simulations, we consider $N\in \{10,20,30,40\}$. Given an instance $({\rho}, N ,\mu,A)$, our simulated marketplace is described as follows.


{\bf Arrival process.} Time is discrete, $t=1,2,\cdots,T$, where we assume that $T \in (200,400,600,800)$ for $N\in (10,20,30,40)$, respectively. At the beginning of each time period $t$, a fixed $M_N(i)$ number of workers of type $i$ arrive and they stay for $N$ periods. We choose $M_N(i) = 600/N$ for the different values of $N$ for each $i$, so that, irrespective of $N$, the total number of workers of any type $i$ present in the market at any time $t$ is $M_N(i)\times N = 600$ (making a total of $600\times |\cI| = 600 \times 4 = 2400$ workers present in the market at any time). Relating back to the continuum model, here we implicitly assume that a mass of 0.25 corresponds to 600 workers or jobs. Observe that the fact that $M_N(i)$ is the same for all $i\in\cI$ reflects the fact that $\rho_i=0.25$ for all $i$.

Job-matching decisions are sequentially made for each of the $2400$ workers in each time step in an arbitrary order. Between successive allocations, a job of type $j$ arrives with probability $\mu_j/\sum_{i\in\cI}\rho_i$ for each type $j$. Thus the relative proportions of the number of workers  of different types that are present in the market at any time, and the expected number of jobs that arrive in the market in any time step, are proportional to $\rho_i$ and $\mu_j$ for $i\in\cI$ and $j\in \cJ$, where a mass of 1 corresponds to 1200 workers or jobs. We assume that each job takes one period to complete.



{\bf Queues and queue-based prices.}   
As described in Section~\ref{subsec:queueprices}, we utilize a PD-control mechanism to set prices with the goal of stabilizing queue lengths; the details are presented in Appendix~\ref{sec:pd-control}. Queue lengths change at epochs where either a job arrives or it is assigned to a worker. The arriving jobs accumulate in queues for the different types, each with a finite buffer of capacity $B$, where we choose $B=50,000$. If the buffer capacity is exceeded for some job type then the remaining jobs are lost. 
(We use a large buffer in our simulations to keep the price fluctuations small, to ensure that differences in the observed performance of \DEEMplus and other benchmark policies are not caused by large price fluctuations.)


{\bf Matching process.}  In the period when a worker arrives, when a job-matching decision is made for the worker for the very first time, the instantaneous price at that time is assigned to the worker. This price is utilized by the policy in determining the job allocations for the worker throughout her lifetime.

At every job-matching opportunity, an assignment is generated by the chosen policy based on the assigned price and the history of assignments and outcomes for the worker. If a job of the required type is unavailable, then the worker remains unmatched. For each worker-job match, a random payoff is realized, drawn from the distribution specified by $A$, and the assignment-payoff tuple is added to the history of the worker.


{\bf Simulation output.} We look at two main outputs: the average performance and the prices assigned to the workers.
\begin{enumerate}
\item {\bf Performance ratio (PR):} For the performance measure, we compute the average reward per worker over the last $T/4$ periods of the simulation horizon (so that the mean queue lengths have had a sufficient time to stabilize) and divide it by the optimal per-worker reward in the full information setting (i.e., the optimal value of \eqref{eq:opt1} divided by $\sum_{i\in\cI}\rho_i$). We call this quantity the \emph{performance ratio} (PR) of the instance. The average PR over the 350 instances is a proxy for the performance of a fixed policy for a given $N$.

\item {\bf Average queue-based prices $\bar{p}^q_j$:} We also examine the prices for the different job types seen by the workers that arrived in the last $T/4$ periods. We look at the average price over these periods for each job type $j$, denoted by $\bar{p}^q_j$, and also the magnitude of typical fluctuations in prices around their average values. The main goal here is to investigate how these prices compare with $p^*$.
\end{enumerate}
 %


\subsection{Benchmark policies}
Along with \DEEMplus, we implement two other policies. These policies are formally defined in Appendix~\ref{apx:other}.
\begin{enumerate}
 \item  {\bf PA-TS: } We consider an extension of Thompson sampling (TS)~\cite{agrawal2011analysis} adapted to our setting, in which the payoffs are adjusted by queue-based prices to account for capacity constraints. We refer to this policy as PA-TS (for ``price-adjusted Thompson sampling''). PA-TS is formally defined in Figure~\ref{fig:def-PA-TS} in the appendix. TS is a popular Bayesian multi-armed bandit algorithm and is a natural benchmark for our setting (with the incorporation of queue-based prices, as we propose). TS explores due to sampling from the posterior distribution of the true (worker) type; the exploration is more aggressive in the early stages when the posterior is not sufficiently concentrated. Owing to sufficient exploration, it is known to attain the optimal leading-order of regret in several settings \cite{agrawal2012analysis,russo2016information}, but not necessarily the optimal constant factor for the leading term.
 \item {\bf TS-\DEEMplus: }We also consider a policy that is a hybrid of \DEEMplus and TS that also uses queue-based prices for payoff adjustment to account for capacity constraints. We will refer to this policy as TS-\DEEMplus. TS-\DEEMplus is formally defined in Figure~\ref{fig:def-DTS} in the appendix. This policy operates individually on each worker and has the same two-phase structure of \DEEMplus, consisting of an Explore phase and an Exploit phase. The Exploit phase is identical to that in \DEEMplus, in which the algorithm simply assigns a job that maximizes the expected price-adjusted payoff at each opportunity. The learning goals of the Explore phase are also identical to those of \DEEMplus: in order to label a worker as being of type $i$, the probability of misclassifying $i'$ as $i$ for any $i'\in\Str^q(i)$ must be less than $1/N$. The only difference is in the way the policy explores. TS-\DEEMplus uses Thompson sampling throughout the Explore phase until the learning goals are met, unlike \DEEMplus, which uses Thompson sampling only in the Guessing mode and a different sampling policy that optimizes the trade-off between learning and payoff maximization in the Confirmation mode.
\end{enumerate}
Comparing the performance of \DEEMplus with PA-TS and TS-\DEEMplus allows us to investigate the relative importance of the two main features of \DEEMplus that it inherits from DEEM (the other important feature being the use of shadow prices for externality adjustment): (1) setting appropriate learning goals and (2) achieving these goals while minimizing regret in the Confirmation mode. As we have seen earlier, both these features are critical to the leading-order regret optimality of DEEM.

We conjecture that both PA-TS and TS-\DEEMplus explore sufficiently and are asymptotically optimal; i.e., the average regret per unit mass of workers converges to $0$  as $N\rightarrow \infty$.  But unlike the Confirmation mode in DEEM or \DEEMplus, these algorithms don't optimize the exploration vs. exploitation trade-off in a way that is tailored to the instance. Hence, we conjecture that they do not achieve the optimal leading term of regret.

\subsection{Results}\label{sec:simresults}

\begin{figure}[ht]
  \begin{subfigure}[b]{0.5\textwidth}
    \centering
    \includegraphics[width=0.9\textwidth]{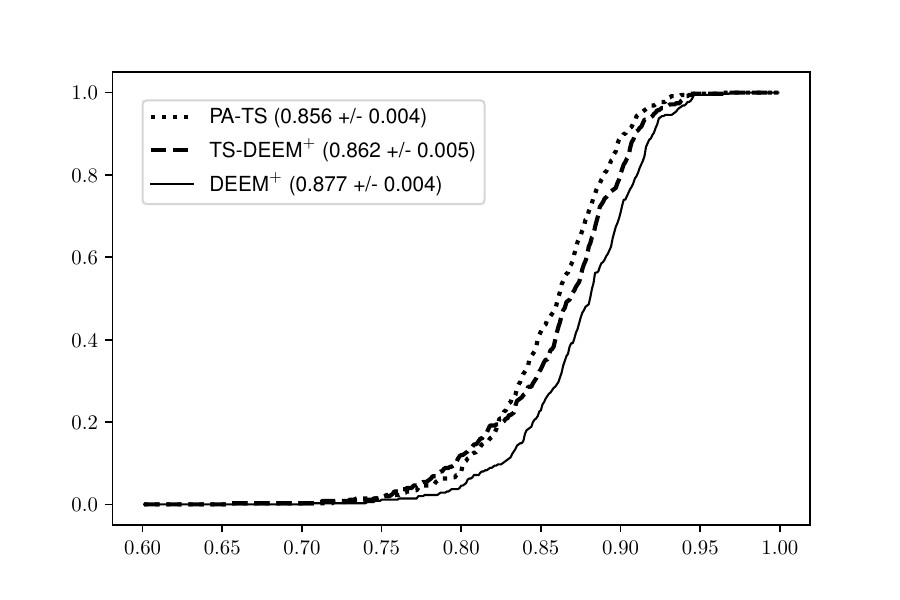}
    \caption{$N=10$}
    \label{perf:a}
    \vspace{4ex}
  \end{subfigure}
  \begin{subfigure}[b]{0.5\textwidth}
    \centering
    \includegraphics[width=0.9\textwidth]{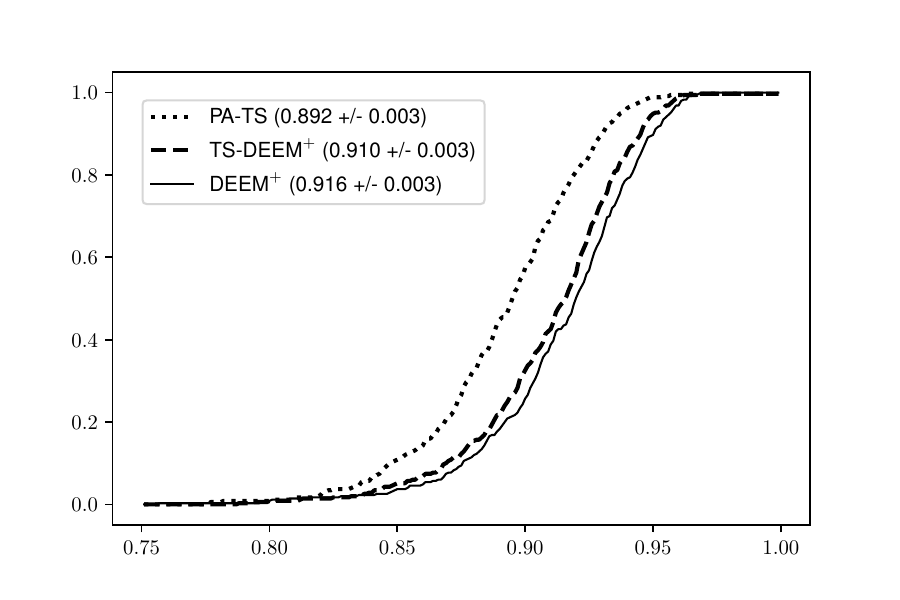}
    \caption{$N=20$}
    \label{perf:b}
    \vspace{4ex}
  \end{subfigure}
  \begin{subfigure}[b]{0.5\textwidth}
    \centering
    \includegraphics[width=0.9\textwidth]{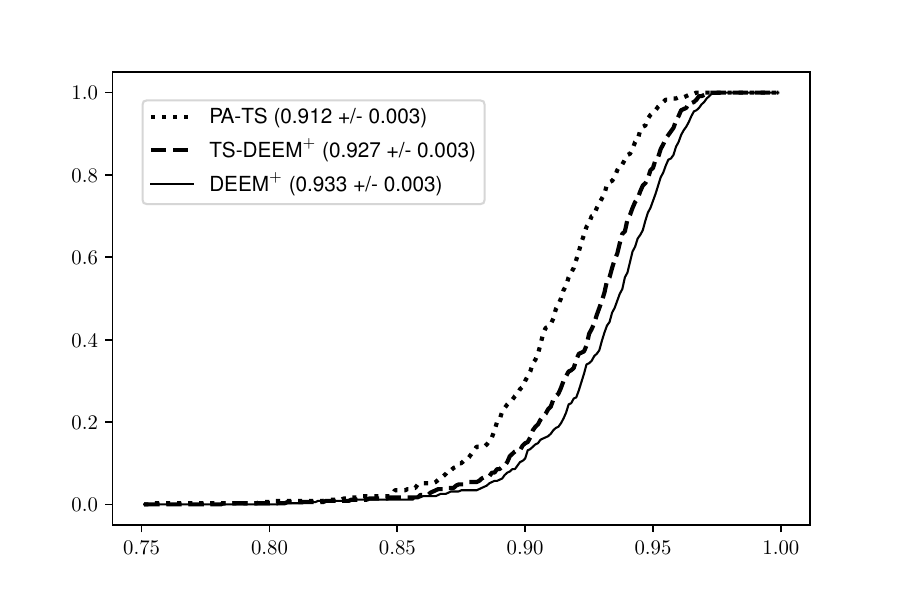}
    \caption{$N=30$}
    \label{perf:c}
  \end{subfigure}
  \begin{subfigure}[b]{0.5\textwidth}
    \centering
    \includegraphics[width=0.9\textwidth]{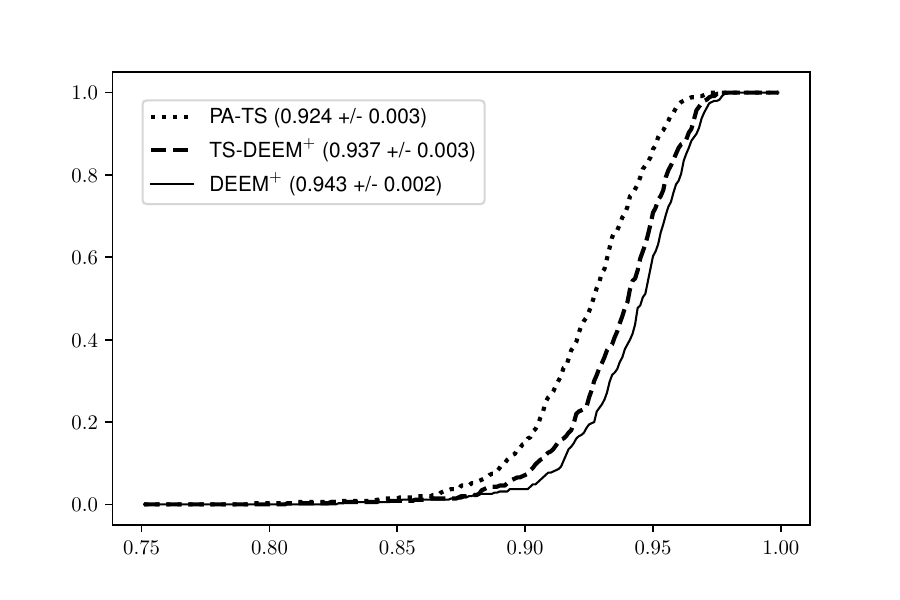}
    \caption{$N=40$}
    \label{perf:d}
  \end{subfigure}
  \caption{The empirical cumulative distributions of the PRs under the different policies over 350 instances for different values of $N$. The average PR along with the standard error under a policy is given in the corresponding parenthesis.}
  \label{fig:perf}
\end{figure}

{\bf Performance comparison}.  The four panels of Figure~\ref{fig:perf} show the empirical cumulative distribution function (CDF) over the $350$ instances of the per-instance PRs (defined in the discussion on ``simulation output'' above) for the three candidate policies, for different values of $N$. The average of these ratios over the set of instances for each policy, along with the standard error, is presented in the legends. As one can observe, both \DEEMplus and TS-\DEEMplus significantly outperform PA-TS by a wide margin in all four settings. This suggests the importance of setting appropriate learning goals and of focusing only on payoff maximization after the learning goals have been met. By contrast, PA-TS explores excessively: it keeps exploring beyond the point where such experimentation has any significant benefit for future payoffs. \DEEMplus outperforms TS-\DEEMplus in all four settings as well. This points to the added benefit from having an optimized Confirmation mode, which is central to the leading-order regret optimality of DEEM.

Figure~\ref{fig:cdf} shows the performance improvement of \DEEMplus with growing $N$.

{ {\bf Quality of our regret estimate.} In Appendix~\ref{apx:dpconfirm} we numerically investigate the quality of our finite $N$ regret estimate. Since \DEEMplus is a modified version of DEEM, this regret estimate is a corresponding modified version of the asymptotic regret estimate $C \log N/N$ in Theorem~\ref{thm:mainresult}. We find that our regret estimate captures the true regret reasonably well, which suggests that our specific design of the Confirmation mode under \DEEMplus is approximately minimizing regret, and explains why \DEEMplus outperforms the TS-\DEEMplus policy (see above).}

\begin{table*}[t]
\centering
    \begin{minipage}[c]{0.42\textwidth}
      \vspace{0pt}
      \includegraphics[width=\textwidth]{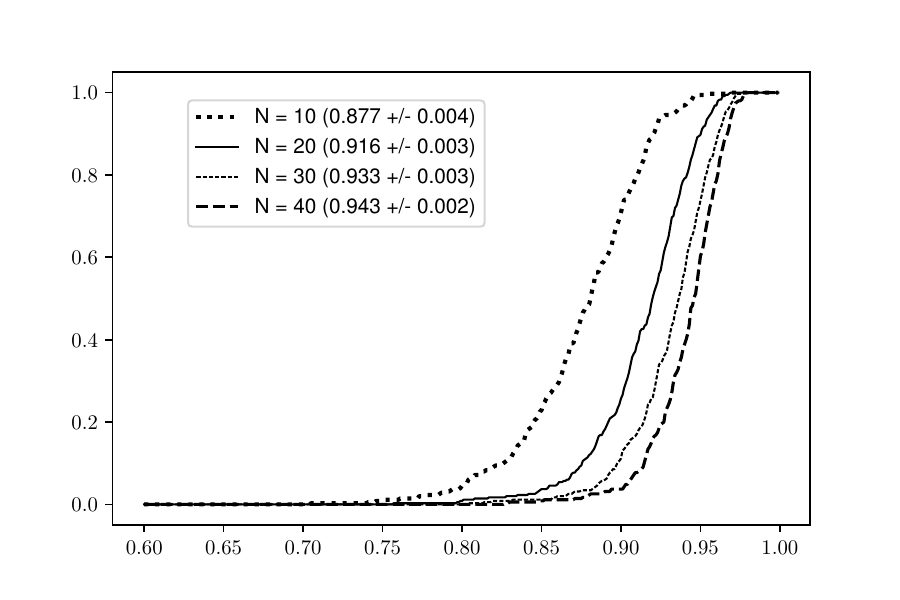}
      \captionof{figure}{The empirical cumulative distribution of the PRs under $\DEEMplus$ for different values of $N$. The average PRs along with standard errors are given in the parentheses.}
      \label{fig:cdf}
    \end{minipage}
    \hspace{0.05\textwidth}
    \begin{minipage}[c]{0.5\textwidth }
  \begin{tabular}{c c c}
        \toprule
         $N$ &  Med. $\|p^*-\bar{p}^q\|_\infty$ & (Med. of stdev($p_j^q$))$_{j\in\cJ}$\\ \midrule
          10           & 0.119  &   (0.011, 0.010, 0.011)  \\
          20     & 0.075  &  (0.006, 0.006, 0.006)  \\
         30            & 0.057  &  (0.006, 0.006, 0.006)  \\
        40 & 0.053 &  (0.007, 0.006, 0.007) \\ \bottomrule
       \end{tabular}
       \vspace{0.1in}
       \caption{First column: median over 350 instances of difference between the queue-based prices $\bar{p}^q$ under \DEEMplus and $p^*$ for different values of $N$ under the $\mathcal{L}^\infty$ norm. Second column: median over 350 instances of the standard deviation of the prices seen by workers in the last $T/4$ periods, under \DEEMplus, for the different values of $N$. All values are rounded to three decimal places.}
       \label{tbl:p}
       \end{minipage}%
\end{table*}

{\bf Prices.} Under generalized imbalance (recall condition~\eqref{eqn:genimb}), the potential capacity violations due to deviations from the optimal routing during the Explore phase can be corrected in the Exploit phase by designing a routing matrix $y^*$ that perturbs $x^*$ appropriately, but without changing the shadow prices for large enough $N$ (Proposition~\ref{prop:construct}).  For finite $N$, such a correction that leaves the shadow prices unaffected may
not be possible if GI is violated or near-violated. In these cases, even if $p^*$ is unique, the shadow prices under unknown worker types may a priori be quite different from $p^*$, and using $p^*$ for externality adjustment in these cases could prove to be detrimental to performance. On the other hand, our queue-length-based implementation organically discovers the appropriate shadow prices. It is nevertheless interesting to investigate how different these prices are from $p^*$.

The first column of Table~\ref{tbl:p} shows the average and median value of $\|p^*-\bar{p}^q\|_\infty$ under \DEEMplus for different values of $N$. As expected, the average queue-length-based prices $\bar{p}^q$ are close to $p^*$, and they get closer for larger $N$. The difference in these prices is small as compared to the typical range of variation of these prices: the standard deviation of $p^*_j$ for $j\in\cJ$ (rounded to three decimal places) is $(0.512, 0.530, 0.484)$. 
The closeness of $p^*$ and $\bar{p}^q$ supports our theoretical device of using $p^*$ to approximately capture the externalities due to capacity constraints in the large $N$ regime.
Moreover, we find that the price fluctuations around the mean values $\bar{p}^q$ are relatively small. This can be observed in the second column of Table~\ref{tbl:p}, which shows the median (rounded to three decimal places) across the 350 instances of the standard deviation of the prices seen by workers in the last $T/4$ periods, for different values of $N$.

Finally, since the fluctuations of $p^q$ around $\bar{p}^q$ are very small and $\bar{p}^q$ itself is typically close to $p^*$, we expect to frequently encounter difficult type pairs in the price-adjusted multi-armed bandit problems corresponding to the different workers (recall that all 350 instances have a difficult type pair assuming that the shadow prices are $p^*$). Indeed, we observe that under prices $\bar{p}^q$, each of the 350 instances possesses difficult type pairs. 

%% file: conclusion.tex
\section{Conclusion}
\label{sec:conclusion}


This work suggests a novel and practical algorithm for learning while matching, applicable across a range of online matching platforms.  Several directions of generalization remain open for future work.  First, while we consider a finite-type model, a richer model of types would admit a wider range of applications; e.g., workers and jobs may be characterized by features in a vector-valued space, with compatibility determined by the inner product between feature vectors.  Second, while our model includes only one-sided uncertainty, in general a market will include two-sided uncertainty (i.e., both supply and demand will exhibit type uncertainty). Finally, although our results extend immediately to settings where all jobs have a fixed duration larger than one time period, 
it would be interesting to consider settings where job durations are random and possibly depend on the job type and/or the worker type. 
We expect that a similar approach using shadow prices for capacity constraints, first to set learning objectives and then to achieve them while incurring minimum regret, should be applicable even in these more general settings. 


We conclude by noting that our model ignores strategic behavior by participants.  A simple extension might be to presume that workers are less likely to return after several bad experiences; this would dramatically alter the model, forcing the policy to become more conservative.  The modeling and analysis of these and other strategic behaviors remain important challenges.

%% file: appendix.tex
\begin{APPENDIX}{}
\section{Appendix to Section~\ref{sec:model}}\label{apx:model}
\subsection{Estimating system-level parameters}\label{apx:estimation}
As mentioned in Section~\ref{sec:model}, the job arrival rates $\mu$ can be directly estimated empirically since job types are observed, while the worker arrival rates $\rho$ and payoff matrix $A$ can be estimated using the observed outcome data. The latter can be done as follows.
For each job type $j$, consider a collection of workers $\mathcal{L}_j$ who have performed $N' \geq 2|\cI|-1$ jobs of type $j$. Let $U^l_j\in \{0,1,\ldots, N'\}$ denote the number of successes of worker $l\in \mathcal{L}_j$ on job type $j$. From our model, it is clear that $U^l_j$ is a mixture of Binomial random variables; in particular, $\textup{P}(U^l_j = u) = \sum_{i \in \cI}\rho_i {N' \choose u}A(i,j)^u(1-A(i,j))^{N'-u}$. It was shown in \cite{blischke1964estimating} that $N' \geq 2|\cI|-1$ is a necessary and sufficient condition that the model is identifiable, i.e., no two different sets of $(A(i,j))_{i\in\cI}$ and $\rho$ can lead to the same distribution on $U^l_j$. Under this condition, there are several methods for constructing efficient estimators for $(A(i,j))_{i\in\cI}$ and $\rho$ from the collection $(U^l_j)_{j\in \mathcal{L}_j}$, whose accuracy increases with $|\mathcal{L}_j|$; for instance \cite{blischke1964estimating} presents three different methods based on population moments. \cite{feldman2008learning} describes an estimator for mixtures of general discrete distributions with bounded support, which can also be utilized here. We can estimate $(A(i,j))_{i\in\cI}$ for each job type $j\in\cJ$ in this manner.

The aforementioned estimators are consistent, i.e., the estimator error vanishes asymptotically in the number of samples. Formally speaking, given that our model in Section~\ref{sec:model} assumes a continuum of workers, there are indeed an infinite number of workers available.  In particular,  in such a setting $A$ and $\rho$ could in principle be estimated with arbitrary accuracy; in this sense, our continuum model is consistent with the formal assumption in Section~\ref{sec:model} that $A$ and $\rho$ are exactly known.



\subsection{Proof of Lemma~\ref{lem:capacities-and-steady-state}: Assignments under WHO policies satisfying \eqref{eq:capacity_constraints} satisfy capacity constraints.}

In this section, we prove Lemma~\ref{lem:capacities-and-steady-state}, which shows that the implied assignments under WHO policies that satisfy condition \eqref{eq:capacity_constraints} satisfy capacity constraints at all times.
\label{app:transient-capacity}

\begin{proof}{Proof of Lemma~\ref{lem:capacities-and-steady-state}.}
We prove the result by establishing the following fact, using induction: At all times $t$, and for all $i \in \cI$,
\begin{align}
\nu_t(H, i) =
\left \{
\begin{array}{ll}
\nu_\pi(H, i) & \textup{if length}(H) \leq \min(t, N-1)\, , \\
0 & \textup{otherwise}\, .
\end{array} \right .
\label{eq:transient-induction}
\end{align}  
The claim clearly holds at $t = 0$. The new arrivals ensure $\nu_0 (\phi, i) = {\rho}_i = \nu_\pi(\phi,i)$, and since the system started empty, $\nu_0 (H,i) = 0$ for all $H \neq \phi, i \in \cI$.

Assume the inductive hypothesis at time $t$.  The inductive hypothesis, together with the definition of $x_\pi$ and the capacity constraint \eqref{eq:capacity_constraints} imply that
\[ \sum_{i \in \cC} \sum_H \nu_t(H, i) \pi(H, j) \leq \sum_{i \in \cC} \sum_H \nu_\pi(H, i) \pi(H,j) \leq \mu_j\, , \quad \forall j \in \cJ\, .\]
Not only does this establish \eqref{eq:transient-capacity} for time $t$, but it also ensures that sufficient capacity exists to ensure that all workers at time $t$ can be matched to their intended job type based on their respective histories, according to policy $\pi$.
As a result, we have prior to $t+1$ that for all $H$ such that length$(H) \leq \min(t, N-2)$, and all $j \in \cJ, i \in \cI$ that
\begin{align*}
 \nu_{t+1}((H, (j,1)), i) &= \nu_t(H, i) \pi(H, j) A(i,j)\\
& = \nu_\pi(H, i) \pi(H,j) A(i,j) \\
& = \nu_\pi((H, (j,1)), i).
\end{align*}
A similar argument applies to show that $\nu_{t+1}((H, (j,0)), i) = \nu_\pi((H,(j,0)),i)$. In other words, we have shown that $\nu_{t+1}(H', i) = \nu_\pi(H',i)$ for all $i \in \cI$ and $H'$ such that $1\leq \textup{length}(H') \leq \min(t+1, N-1)$.  Further, new arrivals (or worker type regenerations) at time $t + 1$ ensure that $\nu_{t+1}(\phi, i) = {\rho}_i = \nu_\pi(\phi, i)$. It remains to show that $\nu_{t+1}(H', i) = 0$ for all $H'$ such that length$(H')> \min(t+1, N-1)$. But this is immediate since, by our inductive hypothesis, $\nu_t(H, i) = 0$ for all $H$ such that length$(H)> t$, and workers leave after $N$ periods.
Induction completes the proof of \eqref{eq:transient-induction} for all $t \geq 0$.

Finally, note that the last sentence of the lemma saying that the system is in steady state for all $t \geq N-1$ follows immediately from \eqref{eq:transient-induction}.\hfill $\Box$
\end{proof}

\subsection{Proof of Proposition \ref{prop:who_sufficiency}: Sufficiency of worker-history-only policies }
\label{app:who_sufficiency}

In this section we prove Proposition \ref{prop:who_sufficiency}, showing that there is a worker-history-only (WHO) policy that achieves a rate of payoff accumulation that is arbitrarily close to the maximum possible. We will think of $N$ as being fixed throughout this section.
Throughout the section, we repeatedly rely on the exact law of large numbers (ELLN) (Theorem 2.16 in \cite{sun2006exact}.)

Fix a policy $\pi$, possibly randomized. We call a sample path (under $\pi$) \emph{typical} if it is such that at each time $t = 1, 2, \dots$, for every possible history $H \in \cH$ and worker type $i$ and job type $j$,  the total reward generated by the mass $m_{\pi, t}(H, i,j)$ of workers with history $H$ and true type $i$ assigned to job type $j$ is $m_{\pi, t}(H, i,j) A(i,j)$, i.e., consistent with the ELLN.  By the ELLN, we know that the realized sample path will be typical with probability 1. Note that the rate of payoff generation over $T$-periods $V_T(\pi)$ given by \eqref{eq:V_T} is a random variable (i.e., a function of the sample path), since the $\big (x_{\pi,t}(i,j)\big )$s are random variables as noted in Section~\ref{subsec:matching-policies}. Also recall that due to the ELLN, for any WHO policy $\piwho$, the steady state payoff $W^N(\piwho)$ is deterministic,  as given by \eqref{prob:mixedbandit}, together with \eqref{eq:recursion1}-\eqref{eq:x_pi}. The following lemma shows that for any policy $\pi$ and any typical sample path (evaluated over a long enough horizon), there is a WHO policy $\piwho$ which does almost as well (in steady state). The proposition will be easily proved using this lemma.

\begin{lemma}
Fix $\eps> 0$ and the worker lifetime $N$. There exists $\uT = \uT(N, |\cJ|, 1/\eps)< \infty$ such that the following holds. Consider any  policy $\pi$, any typical sample path under the policy, and any horizon $T \geq \uT$. Let $V_T(\pi)$ be the rate of payoff generation as per \eqref{eq:V_T} for the given sample path.  Then there is a WHO policy $\piwho$ such that
$$W^N(\piwho) > V_T(\pi)-\eps/2\, ,$$
where $W^N(\piwho)$ is the steady state rate of payoff generation under $\piwho$ as given by \eqref{prob:mixedbandit}.
\label{lma:WHO-can-match-any-sample-path}
\end{lemma}
\begin{proof}{Proof of Lemma~\ref{lma:WHO-can-match-any-sample-path}.}
Throughout this proof, we suppress dependence on $\pi$.  
{Consider any typical sample path.} Let $\nu_t(H_k)$ 
be the measure of workers in the system with history $H_k$ just before the start of time $t$, and abusing notation, let $\nu_t(H_k)_j$ be the measure of such workers who are assigned to job type $j$ at time $t$.
(Note that $\nu_t(H_k)_j = \sum_{i \in \cI} m_{\pi, t}(H_k, i,j)$.)
{Here $k \in \{1, 2, \dots, N\}$ is the length of the history $H_k$ plus 1.} 
Since the policy cannot assign more jobs than have arrived in any period, we have
\begin{align}
  \sum_{k=1}^{N} \sum_{H_k} \nu_t(H_k)_j \leq \mu_j  \qquad \textup{for all} \ t\geq 1 \, .
  \label{eq:det_cap_constraints}
\end{align}
Fix {$T \geq \uT$, subsequently we will choose $\uT$ sufficiently large}. The average measure of workers with history $H_k$ who are present is
\begin{align}
\bar{\nu}_T(H_k) = \frac{1}{T} \sum_{t=1}^T \nu_t(H_k) \qquad \textup{for all} \ H_k \ \textup{and} \  k = 1, 2, \ldots, N \, .
\end{align}
The average measure of such workers who are assigned job $j$ is similarly defined and (abusing notation) denoted by $\bar{\nu}_T(H_k)_j$.  We immediately have that
\begin{align}
  \sum_{k=1}^{N} \sum_{H_k} \bar{\nu}_T(H_k)_j \leq \mu_j \, ,
  \label{eq:arbpi_av_cap_constraints}
\end{align}
by averaging Eq.~\eqref{eq:det_cap_constraints} over times until $T$. Now, consider a worker with history $H_k$ assigned a job of type $j$. Using the known $A$ matrix and arrival rates $\rho$, we can infer the posterior distribution of the worker type based on $H_k$ {(the computed posterior will be correct since the sample path is typical)}, and hence, the likelihood of the job of type $j$ being successfully completed.
Let $p(H_k, j)$ denote this probability of success. Then the distribution of $H_{k+1}$ for the worker is given by
\begin{align*}
H_{k+1} = \big (\, H_k, \big( j, \textup{Bernoulli}(p(H_k, j))\big) \, \big )\, .
\end{align*}
{Since the sample path is typical, we thus have
\begin{align}
\nu_{t+1} \big (\, H_k, \big( j, 1)\big) &= \nu_t (H_k)_j p(H_k, j) \nonumber \\
\nu_{t+1} \big (\, H_k, \big( j, 0)\big) &= \nu_t (H_k)_j (1- p(H_k, j)) \, .
\end{align}
for all $H_k$ with $k< N$, all $j \in \cJ$ and all $t \in \mathbb{N}$.
}
Barring the edge effect at time $T$ caused by workers whose history was $H_k$ at time $T$, this identity allows us to relate $\bar{\nu}_T(H_{k+1})$ to $\bar{\nu}_T(H_k)_j$'s.
In particular, for any $\delta_1 > 0$, if $T \geq \max_{i \in \cI} \rho_i/(N \delta_1)$ we have that
\begin{align}
\bar{\nu}_T(H_{k}, (j,1) ) &\approxt{\delta_1}  \bar{\nu}_T(H_k)_j p(H_k, j) \nonumber \\
\bar{\nu}_T(H_{k}, (j,0) ) &\approxt{\delta_1}  \bar{\nu}_T(H_k)_j \, \big (1 - p(H_k, j) \big ) \, .
\label{eq:arbpi_measure}
\end{align}
Here, $a \approxt{\delta} b$ represents the bound $|a - b|\leq \delta$.
Note that we have
\begin{align}
  V_T =   \sum_{k=1}^{N} \sum_{H_k}  \bar{\nu}_T(H_{k})_j p(H_k,j) \, .
  \label{eq:arbpi_payoff}
\end{align}

We are now ready to define our WHO policy $\piwho$.  Fix $\delta_2 > 0$.  For every $H_k$ such that $\bar{\nu}_T(H_k) \geq \delta_2$, this policy will attempt to assign a worker of history $H_k$ to job type $j$ with probability $\bar{\nu}_T(H_k)_j/\bar{\nu}_T(H_k)$, independently of other workers.  For now we ignore capacity constraints, and return to ensure they are met below.  Note that the WHO policy is {\em stationary}, even though it depends on the time $T$ behavior of the original arbitrary policy $\pi$.

We define a history as {\em rare} if $\bar{\nu}_T(H_k) < \delta_2$; note that rarity is defined w.r.t.~frequency of occurrence under $\pi$.  Workers with rare histories will be assigned the empty job under $\piwho$.  This specification uniquely defines $\piwho$ as well as the steady state cross-sectional distribution of worker histories under $\piwho$.  In particular, again using the ELLN, the steady state mass $\unu (H_k)$ under $\piwho$ of workers with history $H_k$ that are not rare is bounded as:
\begin{align}
 (1- \delta_1/\delta_2)^{k-1} \bar{\nu}_T(H_k) \leq \unu (H_k) &\leq  (1+ \delta_1/\delta_2)^{k-1}\bar{\nu}_T(H_k)\nonumber
\end{align}
using Eq.~\eqref{eq:arbpi_measure}, and the fact that all subhistories of $H_k$ are also not rare, along with induction on $k$. It follows that
\begin{align}
 \unu (H_k) &\approxt{\delta_3} \bar{\nu}_T(H_k) \quad \textup{where} \ \delta_3 = \max(\exp(N\delta_1/\delta_2) - 1, \delta_2) \, ,
 \label{eq:piwho_close}
\end{align}
for all histories (including rare histories), using $k \leq N$ and $\bar{\nu}_T(H_k) \leq 1$ and assuming $\delta_2> \delta_1 \Rightarrow \exp(Nx) -1 \geq (1+x)^N -1 \geq 1-(1-x)^N$ for $x = \delta_1/\delta_2 \in (0,1)$.

Ignoring capacity violations, the steady state rate of accumulation of payoff under $\piwho$ is
\begin{align}
\sum_{k=1}^{N} \sum_{H_k}  \unu(H_{k})_j p(H_k,j) \,
\approxt{\delta_5} \sum_{k=1}^{N} \sum_{H_k}  \bar{\nu}(H_{k})_j p(H_k,j) = V_T(\pi)\nonumber \\
\textup{where}\  \delta_5 = 2^N |\cJ|^{N-1} \delta_3 \,
\end{align}
again using Eq.~\eqref{eq:piwho_close} and the fact that there are $\sum_{k \leq N} (2|\cJ|)^{k-1} \leq 2^N |\cJ|^{N-1} $ possible histories. 

We now return to consider capacity constraint violations by $\piwho$.  Violation of the $j$-capacity constraint under $\piwho$ can be bounded as follows:
\begin{align*}
  &\left( \sum_{k=1}^{N} \sum_{H_k} \unu(H_k)_j -  \mu_j \right)_+ \leq  \left(  \sum_{k=1}^{N} \sum_{H_k} \bar{\nu}(H_k)_j - \mu_j  \right)_+ + 2^N |\cJ|^{N-1} \delta_3 = 2^N |\cJ|^{N-1} \delta_3
\end{align*}
using Eq.~\eqref{eq:piwho_close} and Eq.~\eqref{eq:arbpi_av_cap_constraints}, and the fact that there are $\sum_{k \leq N} (2|\cJ|)^{k-1} \leq 2^N |\cJ|^{N-1} $ possible histories. It follows that the sum of capacity constraint violations across $j \in \cJ$ is bounded by $(2|\cJ|)^N \delta_3$.

{Now we eliminate capacity violations, if any, for each job type $j$ in turn by taking (in any sequence) the histories such that $\unu(H_k)_j > 0$ and reducing the likelihood under $\piwho$ that workers with history $H_k$ are assigned job type $j$ just enough that $\sum_{k=1}^{N} \sum_{H_k} \unu(H_k)_j =  \mu_j$. 
We compensate for the reduction in likelihood of being assigned job type $j$ by increasing by an equal amount the likelihood of assigning the history $H_k$ worker to the ``unassigned'' type $\kappa$ in that period. Note that this increases $\unu(H_k, (\kappa, 0))$, and so we make further adjustments to $\piwho$ as follows. For all $H_{k'}$ which is a continuation/superhistory of $(H_k, (\kappa, 0))$ (including $(H_k, (\kappa, 0))$ itself), we ensure that $(\unu(H_{k'})_{j'})_{j' \in \cJ \backslash \{\kappa\}}$ remain unchanged by increasing the relative likelihood of assigning a history $H_{k'}$ worker to $\kappa$ by the right amount. (This ensures that we are not increasing the utilization of any \emph{other} job type $j' \in \cJ \backslash \{\kappa \}$ in the process of satisfying the capacity constraint for $j$. Also, note that our adjustment has preserved the property that $\piwho$ is a WHO policy.) A useful (non-WHO) \emph{interpretation} of masses $\unu$ resulting from this adjustment is that a worker who is once denied job type $j$ due to an adjustment as above is never again assigned any job during their lifetime while other workers are unaffected. Each such worker loses payoff no more than $N$. Aggregating over all job types, the mass of such workers is no more than the sum of capacity constraint violations $(2|\cJ|)^N \delta_3$ and so the total resulting payoff loss is no more than $\delta_4= N(2|\cJ|)^N \delta_3$.} 

As per \eqref{prob:mixedbandit}, we use $W^N(\pi)$ to denote the true steady state rate of accumulation of payoff under $\piwho$ with this modification to ensure capacity constraints are met. Combining the above and noting $\delta_5 < \delta_4$, we deduce that $W^N(\piwho)\geq V_T(\pi)- 2\delta_4$.
Hence, it suffices to have $\delta_4 = \eps/4$, which can be achieved using $\delta_3 = \delta_2= \eps/(4N(2|\cJ|)^N)$ and $\delta_1 = \delta_3 \log (1+ \delta_3)/N$ and
$\uT = \max_{i \in \cI} \rho_i/(N \delta_1)$.
This yields the required bound of 
$W^N(\piwho)\geq V_T(\pi) - \eps/2$ for any $T \geq \uT$.\hfill $\Box$
\end{proof}

With Lemma~\ref{lma:WHO-can-match-any-sample-path} in place, the proof of Proposition~\ref{prop:who_sufficiency} is straightforward.
\begin{proof}{Proof of Proposition \ref{prop:who_sufficiency}.}
As in the proof of Lemma~\ref{lma:WHO-can-match-any-sample-path}, we suppress the dependence on $\pi$ in this proof.
By definition of $\bV$, we know that there exists an increasing sequence of times $T_1, T_2, \ldots$ such that $\E[V_{T_n}] > \bV -\eps/2$ for all $n = 1, 2, \ldots$. 
For the constant $\uT$ in the statement of Lemma~\ref{lma:WHO-can-match-any-sample-path}, let $n$ be such that $T_n \geq \uT$. Let $T \triangleq T_n$.
By definition of expectation and since a typical sample path arises with probability 1 (also, the payoff per period is bounded by $1$ for \emph{all} sample paths, so atypical sample paths have no contribution to $\E[V_{T}]$), there must be a typical sample path such that $V_T$ under the sample path is at least $\E[V_T]$. Furthermore, using Lemma~\ref{lma:WHO-can-match-any-sample-path}, we have that there is a WHO policy $\piwho$ such that $W^N(\piwho) + \eps/2$ exceeds $V_T$ under the sample path. It follows that $W^N(\piwho)> \E[V_T] -\eps/2 >  \bV -\eps$, where we used that $\E[V_T] > \bV -\eps/2$. This completes the proof of the proposition.\hfill $\Box$
\end{proof}
\subsection{Proof of Proposition~\ref{prop:learningset}: $\cX^N$ is a convex polytope }\label{apx:polytope}
\begin{proof}{Proof of Proposition~\ref{prop:learningset}}
For the purpose of this proof, let
$$\overline{\cX}^N=\{Nx:x\in \cX^N\}.$$
We will show that $\overline{\cX}^N$ is a polytope, from which the result will follow. We will prove this using an induction argument. We will represent each point in $\overline{\cX}^N$ as a $|\cC|\times|\cS|$ matrix $(x(i,j))_{|\cC|\times|\cS|}$. Let worker types in $\cC$ be labeled as $i_1,\dots,i_{|\cC|}$ and let job types in $\cS$ be labeled as $j_1,\dots,j_{|\cS|}$.

Now clearly, $\overline{\cX}^0=\{(0)_{|\cC|\times|\cS|}\}$ which is a convex polytope. We will show that if $\overline{\cX}^N$ is a convex polytope, then $\overline{\cX}^{(N+1)}$ is one as well, and hence the result will follow. To do so, we decompose the assignment problem with $(N+1)$ jobs, into the first job and the remaining $N$ jobs. \\

A policy in the $(N+1)$- jobs problem is a choice of a randomization over the jobs in $\cS$ for the first job, and depending on whether a reward was obtained or not with the chosen job, a choice of a point in $\overline{\cX}^N$ to be achieved for the remaining $N$ jobs. Each such policy gives a point in $\overline{\cX}^{(N+1)}$. Suppose that $\eta_1\in\Delta(\cS)$ is the randomization chosen for job $1$, and let $W(j,1)\in \overline{\cX}^N$ and $W(j,0)\in\overline{\cX}^N$ be the points chosen to be achieved from job $2$ onwards depending on the job $j$ that was chosen, and whether a reward was obtained or not, i.e.. $W(.,.)$ is a mapping from $\cS\times \{0,1\}$ to the set $\overline{\cX}^N$. Then this policy achieves the following point in the $(N+1)$- jobs problem:
\begin{eqnarray*}
\begin{bmatrix}
\eta_1(j_1) & \eta_1(j_2) & \dots & \eta_1(j_{|\cS|})\\
\vdots & \vdots & \ddots & \vdots\\
\eta_1(j_1) & \eta_1(j_2) & \dots & \eta_1(j_{|\cS|})
\end{bmatrix}+\sum_{j\in\cS}\eta_1(j)\bigg(\textup{Diag}[A(.,j)]W(j,1)+\textup{Diag}[\bar{A}(.,j)]W(j,0)\bigg),
\end{eqnarray*}
where
$$\textup{Diag}[A(.,j)]=
\begin{bmatrix}
A(i_1,j) & 0 & \dots & 0\\
0 & A(i_2,j)& \dots & 0\\
\vdots & \vdots & \ddots &0\\
0 & 0 &\dots & A(i_{|\cC|},j)
\end{bmatrix}$$
and
$$
\textup{Diag}[\bar{A}(.,j)]=\begin{bmatrix}
1-A(i_1,j) & 0 & \dots & 0\\
0 & 1-A(i_2,j)& \dots & 0\\
\vdots & \vdots & \ddots &0\\
0 & 0 &\dots & 1-A(i_{|\cC|},j)
\end{bmatrix}.
$$
And thus we have
\begin{eqnarray*}
&&\overline{\cX}^{(N+1)}=\\
&&~~\bigg\{\begin{bmatrix}
\eta_1(j_1) & \eta_1(j_2) & \dots & \eta_1(j_{|\cS|})\\
\vdots & \vdots & \ddots & \vdots\\
\eta_1(j_1) & \eta_1(j_2) & \dots & \eta_1(j_{|\cS|})
\end{bmatrix}+\sum_{j\in\cS}\eta_1(j)\bigg(\textup{Diag}[A(.,j)]W(j,1)+\textup{Diag}[\bar{A}(.,j)]W(j,0)\bigg)\\
&&~~~: \eta_1\in \Delta(\cS),\,\, W(.,.)\in \overline{\cX}^N\bigg\}.
\end{eqnarray*}

Let $\mathbf{1}_s$ be the $|\cC|\times |\cS|$ matrix with ones along column corresponding to job type $j$ and all other entries $0$. Then the set
$$\mathcal{J}(s)=\bigg\{\mathbf{1}_s+ \textup{Diag}[A(.,j)]W(j,1)+\textup{Diag}[\bar{A}(.,j)]W(j,0): W(s,.)\in \overline{\cX}^N\bigg\},$$
is a convex polytope, being a linear combination of two convex polytopes, followed by an affine shift. It is easy to see that $\overline{\cX}^{(N+1)}$ is just a convex combination of the polytopes $\mathcal{J}(s)$ for $j\in\cS$, and hence $\overline{\cX}^{(N+1)}$ is a convex polytope as well.\hfill $\Box$
\end{proof}

\subsection{Proof of Proposition~\ref{prop:uniqueness_of_prices}: Uniqueness of prices under generalized imbalance}\label{apx:uniqueprices}
We now show that the optimal prices for the job types are unique under generalized imbalance. To prove this, we will first prove the following key lemma, which is an important ingredient in the proofs of many other results in the paper.

\begin{lemma}\label{lma:imbalancepath}
Suppose that the generalized imbalance condition is satisfied. Consider any feasible routing matrix $[x(i,j)]_{\cC\times\cS}$. Consider any job $j$ such that $\sum_{i\in\cC}\rho_ix(i,j)=\mu_j$. Then there is a (finite) sequence of types $(j_0,i_1,j_1,i_2,j_2,\dots,i_\ell,j_\ell)$ starting from $j_0=j$ which alternates between job types in $\cJ$ and worker types in $\cI$ with the following properties:
\begin{itemize}
\item The sequence starts with job type $j_0 = j$.
\item The sequence ends with a job type $j_\ell$ whose capacity is under-utilized (this job type is permitted to be $\kappa$).
\item The sequence consists of distinct job types, i.e.,  $j_k \neq j_{k'}$ for $k\neq k'$. In particular, the sequence is finite.
\item Every other job type in the sequence is operating at capacity/all jobs are being served:
    \begin{align*}
      \sum_{i \in \cI} \rho_{i} x(i,j_k) = \mu_{j_k} \qquad \forall k = 1, 2, \dots, \ell-1 \, .
    \end{align*}
    (All worker types are fully utilized $\sum_{j \in \cJ}  x(i,j_k) = 1 \ \forall i \in \cI$ by definition, since we formally consider an unassigned worker as being assigned to job type $\kappa$.)
\item For every consecutive pair of worker type and job type, there is a positive rate of jobs routed on that edge in $x$:
    \begin{align*}
      x(j_{k-1}, i_k) > 0 \quad \textup{and} \quad x(i_k, j_k) > 0 \qquad \forall k = 1, 2, \dots, \ell-1 \, .
    \end{align*}
\end{itemize}
\end{lemma}

\begin{proof}{Proof.}
Consider the bipartite graph $\cG= (\cJ, \cI, \cE)$ between the job types $\cJ$ and the worker types $\cI$ such that there is an edge between a job type $j$ and a worker type $i$ if and only if $x(i,j)>0$:
$$\cE \triangleq \{(j, i) \in \cJ \times \cI: x(i,j)>0\}\, .$$
Consider the connected component $\mathcal{C}$ of job type $j$ in $\cG$, and let $\cI_{\mathcal{C}} \subseteq \cI$ be the worker types in $\cI$ and let $\cJ_{\mathcal{C}} \subseteq \cJ$ be the job types in $\cJ$. Note that by definition of $\cG$, all worker (resp., job) types in $\mathcal{C}$ are assigned only to job (resp., worker) types in $\mathcal{C}$, and in particular we have
\begin{align}
  \sum_{j' \in \cJ_{\mathcal{C}}}  x(i,j') = \sum_{j' \in \cJ} x(i,j')= 1 \qquad \forall i \in \cI_{\mathcal{C}} \, .
  \label{eq:worker-types-utilized}
\end{align}

As the main step in proving the lemma, we will now prove by contradiction that there exists a job type in $\cJ_{\mathcal{C}}$ that is underutilized: Suppose there is no job type in $\cJ_{\mathcal{C}}$ that is underutilized, i.e.,
\begin{align}
      \sum_{i \in \cI} \rho_{i} x(i,j') = \sum_{i \in \cI_{\mathcal{C}}} \rho_{i} x(i,j') = \mu_{j' } \qquad \forall j' \in \cJ_\mathcal{C}\, .
\label{eq:job-types-utilized}
\end{align}
Then, it follows from \eqref{eq:worker-types-utilized} and \eqref{eq:job-types-utilized} that the total mass of worker types in $\cC$ is exactly equal to total capacity of job types in $\cC$
\begin{align*}
  \sum_{i \in \cI_{\mathcal{C}} } \rho_i = \sum_{i \in \cI_{\mathcal{C}} }  \sum_{j' \in \cJ_{\mathcal{C}} }\rho_i x(i,j') =\sum_{j' \in \cJ_{\mathcal{C}} } \sum_{i \in \cI_{\mathcal{C}} } \rho_i x(i,j')
  = \sum_{j' \in \cJ_{\mathcal{C}} } \mu_{j'}
\end{align*}
 But this is a contradiction since generalized imbalance \eqref{eqn:genimb} holds. Hence there exists an underutilized job type $j'$ that can be reached from $j$. 
 
 Take any shortest path from $j$ to $j'$ in $\cG$. It is easy to see that no job type appears twice in the path (if a job type appears twice, eliminating the part of the path between consecutive occurrences of the same job type will produce a shorter path).  Now construct a sequence of types as follows. Traverse the path starting from $j$ and stop the first time you encounter an underutilized job type. The sequence of types satisfies all the properties in the lemma.\hfill $\Box$
\end{proof}

We are now ready to prove Proposition~\ref{prop:uniqueness_of_prices}.
\begin{proof}{Proof of Proposition~\ref{prop:uniqueness_of_prices}.}
Recall that the dual to problem \eqref{eq:opt1}--\eqref{eq:opt3} can be written as
\begin{align}
  \textup{minimize} \;&\sum_{j\in\cS} \mu_j p_j  + \sum_{i \in \cC} \rho_i v_i \label{eq:dual1} \\
  \textup{subject to} \nonumber \\
  p_j + v_i &\geq A(i,j) \quad \forall i \in \cC, j\in\cS \, ,\label{eq:dual2}\\
  p_j &\geq 0 \quad \forall j\in\cS\, ,\nonumber \\
  v_i &\geq 0 \quad \forall i \in \cC \, .\nonumber
\end{align}
The dual variables are $(P,V)$ where ``job prices'' $P=(p_j)_{j\in\cS}$ and ``worker values'' $V=(v_i)_{i \in \cC}$.
Now consider a job type $j \in \cJ$ that has at least two different optimal prices. Then one of the prices must be positive and hence, by complementary slackness, this job is fully utilized in any primal solution. Pick a primal solution $x^*$. By Lemma~\ref{lma:imbalancepath}, there exists some other job $j'$ that is underutilized under $x^*$ and there exists a path P between $j$ and $j'$ in the bipartite graph between the worker and the job types (where an edge $(i,j)$ exists if $x^*(i,j)>0$), such that every job type on this path is fully utilized. Now since $j'$ is underutilized, its optimal price $p^*_{j'} = 0$ is uniquely determined. Next, for the worker type $i'$ such that $(i',j')$ is on path P, since $x^*(i',j') >0$, by complementary slackness, we deduce that \eqref{eq:dual2} is tight, i.e., $p_{j'} + v_{i'} = A(i',j')$. As a result, $v^*_{i'} = A(i',j')$ is uniquely determined. Continuing this argument along the path towards $j$, using complementary slackness, we deduce all the optimal dual variables for the nodes on P, including $p^*_j$ are uniquely determined. This contradicts the presumption that job $j$ that has at least two different optimal prices.\hfill $\Box$
\end{proof}
\section{Appendix to Section~\ref{sec:DEEM}}
\subsection{Computation of the policy $\alpha$ in the confirmation subphase}\label{apx:alpha}
In this section, we show how to compute $\alpha(i)$ as defined in Figure \ref{fig:def-alphai-ystar}. Denoting
$$\min_{i'\in \Str(i)}\sum_{j\in \cS}\alpha_j\KL(i,i'|j) \triangleq h$$ (where $h$ is non-negative), the optimization problem in \eqref{eqn:alpha} is the same as:
\begin{align*}
\min &\sum_{j\in \cS} \frac{\alpha_j}{h}\big(U(i)-[A(i,j)-p^*_j]\big)\\
\textrm{ s.t. } &\sum_{j\in \cS}\frac{\alpha_j}{h}\KL(i,i'|j) \geq 1\textrm{ for all }i'\in \Str(i),\\
& \sum_{j\in \cS}\frac{\alpha_j}{h} = \frac{1}{h},\,\,\frac{1}{h}\geq 0 \textrm{ and }\frac{\alpha_j}{h}\geq 0 \textrm{ for all }j.
\end{align*}
Now redefine $\frac{\alpha_j}{h}\triangleq \bar{\alpha}_j$ and $\frac{1}{h}\triangleq \bar{h}$ to obtain the linear program (LP):
\begin{align*}
\min &\sum_{j\in \cS} \bar{\alpha}_j\big(U(i)-[A(i,j)-p^*_j]\big)\\
\textrm{ s.t. } &\sum_{j\in \cS}\bar{\alpha}_j\KL(i,i'|j) \geq 1\textrm{ for all }i'\in \Str(i),\\
&\bar{\alpha}_j\geq 0 \textrm{ for all }j\, .
\end{align*}
Here we eliminated the constraints $\bar{h} = \sum_{j\in \cS}\bar{\alpha}_j$ and $\bar{h} \geq 0$ and the variable $\bar{h}$ in stating the LP: The reasoning is that the requirement $\bar{h} \geq 0$ follows automatically from $\bar{h} =  \sum_{j\in \cS}\bar{\alpha}^*_j$ and $\bar{\alpha}_j \geq 0$, so we can eliminate it. And then $\bar{h}$ appears only once in the entire LP, in the constraint $\bar{h} = \sum_{j\in \cS}\bar{\alpha}_j$, so we can eliminate that constraint and $\bar{h}$.

Given any optimal solution $\bar{\alpha}^*$ to the LP, we can recover $\bar{h}^* = \sum_{j\in \cS}\bar{\alpha}^*_j$, and thus $\alpha(i)_j =\frac{\bar{\alpha}^*_j}{\sum_{\hat{j}\in \cS}\bar{\alpha}_{\hat{j}}}$. Note that a feasible solution exists to this linear program as long as $\KL(i,i'|j)>0$ for some $j$ for each $i'$. When there are multiple solutions,  we choose the solution with the largest learning rate; i.e., we choose a solution with the smallest $\bar{h}^*$, i.e., $\sum_{j\in \cJ}\bar{\alpha}^*_j$. One way to accomplish this is to modify the objective to minimize $\sum_{j\in \cS} \bar{\alpha}_j\big(U(i)-[A(i,j)-p^*_j] + \tau \sum_{j\in \cS} \bar{\alpha}_j$ for some small $\tau>0$.

For small problem instances, we can simply evaluate all the finite extreme points of the constrained set $\{\bar{\alpha}:\,\sum_{j\in \cS}\bar{\alpha}_j\KL(i,i'|j) \geq 1\textrm{ for all }i'\in \Str(i); \,\,\bar{\alpha}_j\geq 0 \textrm{ for all }j\}$; i.e., all the extreme points such that $\bar{\alpha}_j<\infty$ for all $j$. This is sufficient because we can show that there always exists a finite solution to the linear program.  To see this, note that
$$\bar{\alpha}_j=\bar{\alpha}^m_j \triangleq \sum_{i'\in\Str(i)}\frac{1}{\KL(i,i'|j)}\ind_{\{\KL(i,i'|j)>0\}}$$ is feasible and finite. Further, for any solution such that $\alpha^*_j>\bar{\alpha}^m_j$, $\alpha^*_j$ can be reduced to $\bar{\alpha}^m_j$ without loss in objective while maintaining feasibility.

For the practical heuristic \DEEMplus, $\alpha_{\textup{conf}}(\MAP,\lambda)$ can be similarly computed as a solution to a linear program as follows:
Defining
$$1/h\triangleq\max_{i'\in\Str(\MAP)} \frac{\log N + \log R(\MAP,i')}{\sum_{j\in \cS}\alpha_j \KL(\MAP,i'|j)},$$
$\alpha_{\textup{confirm}}(\lambda,\MAP)$ is a solution to the optimization problem
\begin{align*}
\min &\sum_i \lambda(i) \sum_{j\in \cS} \frac{\alpha_j}{h}\big(U(i)-[A(i,j)-p^*_j]\big)\\
\textrm{ s.t. } &\sum_{j\in \cS}\frac{\alpha_j}{h}\KL(\MAP,i'|j) \geq \log N+ \log R(\MAP,i')   \textrm{ for all }\, i'\in\Str(\MAP),\\
& \sum_{j\in \cS}\frac{\alpha_j}{h} = \frac{1}{h},\,\,\frac{1}{h}\geq 0 \textrm{ and }\frac{\alpha_j}{h}\geq 0 \textrm{ for all }j.
\end{align*}
We can again define $\bar{\alpha}_j \triangleq \alpha_j/h$ and $\bar{h}\triangleq 1/h$ to obtain a linear program. For any optimal solution $\bar{\alpha}^*$, return $\alpha =(\bar{\alpha}^*_j/ \sum_{j'\in \cS}\bar{\alpha}^*_j)_{j\in\cJ}$. If there are multiple solutions, then we can enumerate all the finite extreme points that are solutions and pick the solution with the smallest $\bar{h}^*$.

\subsection{Proof of Proposition~\ref{prop:construct}: Existence of $y^*$}\label{apx:construct}
We now prove Proposition~\ref{prop:construct}, which ensures that the matrix $y^*$ satisfying \eqref{eq:y-support} and \eqref{eq:y-capfull} exists (part of the definition of DEEM, see Figure~\ref{fig:def-alphai-ystar}) exists. The proof of this result will crucially use Lemma~\ref{lma:imbalancepath} from Appendix~\ref{apx:uniqueprices}. We remark that the $y^*$ we construct in the proof satisfies $\lVert y^* - x^* \rVert = o(1)$.

\begin{proof}{Proof of Proposition~\ref{prop:construct}.}
The defining condition \eqref{eq:y-m} for $m(i,j)$ can be written as
\begin{align}
m(i,j) = \rho_i y(i,j) + (l(i,i)-\rho_i)y(i,j) + \mexp(i,j) +  \sum_{i' \neq i}l(i,i') y(i',j) \quad \forall i \in \cI, j \in \cJ \, .
\label{eq:m-y-rewritten}
\end{align}
A key fact that we will use is that all terms except the first one are $o(1)$. (This holds because the Explore phase of DEEM lasts only $O(\log N)$ periods in expectation, and because a fraction $o(1)$ of workers receive the wrong type label at the end of exploration.) In particular, $m(i,j) = \rho_i y(i,j) + o(1)$.


We will find a $y$ such that \eqref{eq:y-support} and \eqref{eq:y-capfull} hold and $\lVert y - x^* \rVert = o(1)$. Note that this will imply \eqref{eq:y-capslack}, since
$$\sum_i m(i,j) = \sum_i \rho_i y(i,j) + o(1) = \sum_i \rho_i x^*(i,j) + o(1) < \mu_j \quad \forall j \in \cJ \backslash \cJfull \, ,$$
for $N$ large enough.


The requirement \eqref{eq:y-capfull} can be written as a set of $|\cJfull|$ linear equations using \eqref{eq:m-y-rewritten}. Instead of working with $y$, we work with a ``re-scaled'' version $\tilde{y} \triangleq \rho^T y = [\rho_i y(i,j)]_{i,j}$, where $\tilde{y}$ must have non-negative entries with the $i$-th row summing to $\rho_i$. Similarly, let $\tilde{x}^* \triangleq \rho^T x^*$. In the rest of this proof, we write $\ty$ (and later also $\tx^*$) as a column vector with $|\cC||\cS|$ elements:
\begin{align*}
B \ty + \hat{\eps} = (\mu_j)_{j\in\cJfull} \, .
\end{align*}
Here we have $\lVert \hat{\eps} \rVert = o(1)$ and matrix $B$ can be written as $B= B_0 + B_\eps$, where
 \begin{align*}
   B_0(j,(i,j')) = \left \{
   \begin{array}{ll}
     1 & \mbox{ if } j' = j\\
     0 & \mbox{ otherwise.}
   \end{array}\right .
 \end{align*}
 and $\lVert B_\eps \rVert= o(1)$. Expressing $\ty$ as $\ty = \tx^* + z$, we are left with the following requirement for $z$,
\begin{align}
B z = - (B_\eps \tx^* + \hat{\eps})
\end{align}
using the fact that $B_0 x^* = (\mu_j)_{j\in\cJfull}$ by definitions of $B_0$ and $\cJfull$. We will look for a solution to this underdetermined set of equations with a specific structure: we want $z$ to be a linear combination of flows along $|\cJfull|$ paths coming from Lemma \ref{lma:imbalancepath}, one path $\lambda_j$ for each $j\in\cJfull$. Each $\lambda_j$ can be written as a column vector with $+1$'s on the odd edges (including the edge incident on $j$) and $-1$'s on the even edges. Let $\Lambda = [\lambda_j]_{j\in\cJfull}$ be the path matrix. Then $z$ with the desired structure can be expressed as $\Lambda \eta$, where $\eta$ is the vector of flows along each of the paths. Now note that $Bz = (B_0 + B_\eps)\Lambda \eta = (I+B_\eps \Lambda) \eta$. Here we deduced $B_0 \Lambda = I$ from  the fact that $\lambda_j$ is a path which has $j$ as one end point, and a worker or else a job not in $\cJfull$ as the other end point. Our system of equations reduces to
\begin{align*}
(I + B_\eps \Lambda) \eta = - (B_\eps \tx^* + \hat{\eps})\, ,
\end{align*}
Since $\lVert B_\eps \rVert = o(1) $, the coefficient matrix is extremely well behaved, being only $o(1)$ different from the identity, and we deduce that this system of equations has a unique solution $\eta^*$ that satisfies $\lVert \eta^* \rVert = o(1)$. This gives us $z^* = \Lambda \eta^*$ that is also of size $o(1)$, and supported on permissible edges as per \eqref{eq:y-support} since each of the paths is supported on permissible edges (Lemma \ref{lma:imbalancepath}). Thus, we finally obtain $\ty^* = \tx^* + z^*$, and corresponding $y^*$, possessing all the desired properties. Notice that the (permissible) edges on which $y^*$ differs from $x^*$ had strictly positive values in $x^*$ by Lemma \ref{lma:imbalancepath}, and hence this is also the case in $y^*$ for large enough $N$. \hfill $\Box$
\end{proof}

\section{Appendix to Section~\ref{sec:mainresult}}
\subsection{Proof of Theorem~\ref{thm:mainresult}: The main result}\label{apx:mainproof}
We refer the reader to the roadmap laid out in the proof-sketch in Section~\ref{sec:proofsketch} for help navigating the following proof.
For the rest of this section, recall that:
\begin{align*}
C(i)&=\min_{\alpha\in \Delta(\cS)} \frac{\sum_{j\in \cS} \alpha_j\big(U(i)-[A(i,j)-p^*_j]\big)}{\min_{i'\in \Str(i)}\sum_{j\in \cS}\alpha_j\KL(i,i'|j)}\, ;\qquad
C=\sum_{i\in\cC}\rho_iC(i)\, .
\end{align*}
Recall optimization problem \eqref{prob:mixedbandit}--\eqref{eq:feasibility}. We will first show the following lower bound on the difference between $V^*$ and $W^N$, which follows directly from Theorem 3.1 in \citet{agrawal1989asymptotically}.
\begin{proposition}\label{prop:lowerbound}
$$\limsup_{N\rightarrow\infty}\frac{N}{\log{N}}\big(V^*-W^N\big) \geq C.$$
\end{proposition}
\begin{proof}{Proof.}
Consider the relaxed problem \eqref{approxpricedlearn} restated here:
\begin{align*}
W^N_{p^*}=\max_{x \in \cX^N} \sum_{i\in\cC}\rho_i\sum_{j\in\cS} x(i,j) A(i,j)- \sum_{j\in\cS}p^*_j\bigg[\sum_{i\in \cC}\rho_ix(i,j)-\mu_j\bigg ].
\end{align*}
By a standard duality argument, we know that $W^N_{p^*}\geq W^N$.
Then from Theorem 3.1 in \citet{agrawal1989asymptotically},
it follows that
$$\limsup_{N\rightarrow\infty}\frac{N}{\log{N}}\big(V^*-W^N_{p^*}\big)\geq C.$$
The result then follows from the fact that $W^N\leq W^N_{p^*}$.\hfill $\Box$
\end{proof}

Now, let $W^N_{p^*}(\textup{DEEM})$ be the value attained by DEEM in optimization problem \eqref{approxpricedlearn}. 
We will prove an upper bound on the difference between $V^*$ and $W^N_{p^*}(\textup{DEEM})$, and the only property of the exploitation phase routing matrix we will use here is that $y^*(i,\cdot)$ is supported on $\cS(i)$ for all $i$. Note that the difference between $W^N_{p^*}(\textup{DEEM})$ and $V^*$ is the same as the difference between $$\sum_{i \in \cC} \rho_i \sum_{j\in\cS}x_{\tDEEM}(i,j) (A(i,j)-p^*_j),$$ and $\sum_{i\in\cC}\rho_iU(i)$.
The following is the result.
\begin{proposition}\label{prop:upperbound} Consider {any sequence of WHO policies $(\pi^*_N)_{N\geq1}$ that has an Explore phase identical to that of DEEM and such that the routing matrix $y$ used in the exploitation phase satisfies $y(i,\cdot)\in\Delta(\cS(i))$ for all $i \in \cI$.} Then,
$$\limsup_{N\rightarrow\infty}\frac{N}{\log{N}}\big(V^*-W^N_{p^*}(\pi^*_N)\big) \leq C.$$
Further, if $C=0$, then,
$$\limsup_{N\rightarrow\infty}\frac{N}{\log\log{N}}\big(V^*-W^N_{p^*}(\pi^*_N)\big) \leq K$$
where $K=K(\rho, \mu, A)\in [0,\infty)$ is some constant.
\end{proposition}
\input{upper-bound-proof}



Finally, we show 
that DEEM
asymptotically achieves the required upper bound on regret.
\begin{proposition}\label{prop:mainupperbound}
Suppose that the generalized imbalance condition is satisfied. 
Let $W^N(\textup{DEEM})$ be the value attained by $\textup{DEEM}_N$ in optimization problem \eqref{prob:mixedbandit}-\eqref{eq:feasibility}.
Then
$$\limsup_{N\rightarrow\infty}\frac{N}{\log{N}}\big(V^*-W^N(\textup{DEEM})\big) \leq C.$$
Further, suppose that there are no difficult type pairs. Then,
$$\limsup_{N\rightarrow\infty}\frac{N}{\log\log{N}}\big(V^*-W^N(\textup{DEEM})\big) \leq K$$
where $K=K(\rho, \mu, A)\in (0,\infty)$ is some constant.
\end{proposition}
\begin{proof}{Proof.}
From Proposition~\ref{prop:construct} along with \eqref{eq:y-capfull} and \eqref{eq:y-capslack}, it follows that the policy $\textup{DEEM}_N$ is feasible in problem \eqref{prob:mixedbandit}-\eqref{eq:feasibility}, and further
\begin{eqnarray*}
W^N_{p^*}(\textup{DEEM})&=&\sum_{i\in\cC}\rho_i\sum_{j\in\cS} x_{\tDEEM}(i,j) A(i,j)- \sum_{j\in\cS}p^*_j\bigg [\sum_{i\in \cC}\rho_ix_{\tDEEM}(i,j)-\mu_j\bigg ]\\
&=& \sum_{i\in\cC}\rho_i\sum_{j\in\cS} x_{\tDEEM}(i,j) A(i,j),
\end{eqnarray*}
where the second equality follows from the fact that if $p^*_j>0$, then $\sum_{i\in \cC}\rho_ix^*(i,j)-\mu_j=0$ by complementary slackness, and hence from \eqref{eq:y-capfull} we obtain that $\sum_{i\in \cC}\rho_ix_{\tDEEM}(i,j)-\mu_j=0$ as well for these $j$.
Thus the rate of accumulation of payoff under DEEM in problem~\eqref{prob:mixedbandit}-\eqref{eq:feasibility} is exactly $W^N_{p^*}(\textup{DEEM})$. Moreover, $y^*(i, \cdot) \in\Delta(\cS(i))$ for all $i \in \cI$ by \eqref{eq:y-support} and the optimality of $x^*$ for the problem with known types. The result then follows from Proposition~\ref{prop:upperbound}.\hfill $\Box$
\end{proof}

\begin{proof}{Proof of Theorem~\ref{thm:mainresult}.}
The result follows from Propositions \ref{prop:lowerbound} and  \ref{prop:mainupperbound}.\hfill $\Box$
\end{proof}

\subsection{Proof of Proposition \ref{prop:difficulty-frequent-with-multiple-skills}: Difficult type pairs occur frequently with multiple skill dimensions}\label{apx:multiskills}
\begin{proof}{Proof of Proposition \ref{prop:difficulty-frequent-with-multiple-skills}.}
Starting with Assumption \ref{ass:non-triviality-of-learning} and the corresponding $i \in \cI$, $i' \in \cI$ and $j \in \cJ$, the main idea is to construct an instance such that in any solution under the full information setting, the worker type $i$ will be matched exclusively to job type $j$, whereas the type $i'$ will be matched exclusively to jobs types other than $j$. And that this remains true in a neighborhood of the given set of parameters.

Let $j$ be the job type and $i$ and $i'$ be the worker types corresponding to Assumption \ref{ass:non-triviality-of-learning}. Let $j'$ be some other job type such that $i'$ differs from $i$ on $\Sk_{j'}$. (Since $\Sk_j \subset \Sk$, but $\cup_{\hat{j} \in \cJ} \Sk_{\hat{j}} = \Sk$, it follows that there is such a job type $j'$.) 
Let $A(i,j)=1/2$. Let $\cI(i,\Sk_j)$ denote the set of worker types whose skill levels along dimensions in $\Sk_j$ are identical to $(i_s)_{s \in \Sk_j}$.  Then, by definition of $\Sk_j$, we have $A(\hat{i},j)=A(i,j)=1/2$ for all $\hat{i} \in \cI(i,\Sk_j)$. (In particular, $A(i',j) = 1/2$.)
Also let $A(i',j')=3/4$, and as above, it follows that $A(\hat{i}',j')=A(i',j')=3/4$ for all $\hat{i}' \in \cI(i', \Sk_{j'})$. Let all other terms in the payoff matrix $A$ be $1/4$.

\noindent{\bf Case $\sum_{\hat{j} \in \cJ} \mu_{\hat{j}} > 1$.} We start with the case $\sum_{\hat{j} \in \cJ} \mu_{\hat{j}} > 1$. Let
$$\rho_i = \frac{\min(\mu_j,1)}{3}\, ,\; \rho_{i'} = \frac{\min(\mu_{j'},1)}{3}\, ,\; \rho_{\hat{i}} = \frac{1- \rho_i - \rho_{i'}}{|\cI|-2} \ \textup{for all} \, \hat{i} \notin \{ i,i'\}\, .$$ 
 It follows that there is an optimal solution $x^*$ to the static planning problem \eqref{eq:opt1}-\eqref{eq:opt3} such that:
\begin{enumerate}[label=(\roman*)]
  \item All workers types are fully matched $\sum_{\hat{j}\in \cJ} x^*(\hat{i},\hat{j}) = 1$ for all $\hat{i} \in \cI$ and for all job types $\hat{j} \in \cJ$, the shadow price $p^*(\hat{j}) = 0$.
  \item Worker type $i$ has a unique optimal job type $\cJ(i) = \{j\}$ which it is fully matched to, i.e., we have $x^*(i,j) = 1$.
  \item Worker type $i'$ has a unique optimal job type $\cJ(i') = \{j'\}$ which it is fully matched to, i.e., we have $x^*(i',j') = 1$.
\end{enumerate}
In fact, it is easy to check that \emph{the above properties hold for any instance in a neighborhood of the instance above} (Here we flexibly choose the metric, for concreteness we use the $L^\infty$ distance between instances, for the purpose of defining the neighborhood). It follows that for any instance in a neighborhood of the above instance, the pair $(i,i')$ is a difficult type pair. It follows that $\cPdiff$ has full dimension.

\noindent{\bf Case $\sum_{\hat{j} \in \cJ} \mu_{\hat{j}} < 1$.} The case $\sum_{\hat{j} \in \cJ} \mu_{\hat{j}} < 1$ can be handled similarly, except that we need to construct $\rho$ in a way that generalized imbalance \eqref{eqn:genimb} holds in order to ensure unique $p^*$ using Proposition~\ref{prop:uniqueness_of_prices}, and then to obtain sufficient control on $p^*$. Begin with $\rho$ defined as above, which simplifies to $\rho_i = \frac{\mu_j}{3}$, $\rho_{i'} = \frac{\mu_{j'}}{3}$, and $\rho_{\hat{i}} = \frac{1- \rho_i - \rho_{i'}}{|\cI|-2}$ for all $\hat{i} \notin \{ i,i'\}$. If generalized imbalance holds then proceed with this $\rho$. If not, then add a small random perturbation to $\rho$ as follows, upon which generalized imbalance will hold with probability 1: For $\eps \triangleq \frac{\min(\mu_j, \mu_{j'})}{3|\cI|}$, add i.i.d. Uniform$[0, \epsilon]$ random variables to $\rho_{\hat{i}}$ for all $\hat{i} \notin \cI$. Note that  this will increase $\sum_{\hat{i} \in \cI \backslash \{i\}} \rho_{\hat{i}}$ by at most $(|\cI|-1) \eps < \frac{\mu_j}{3}$ and so we can safely define $\rho_i = 1- \sum_{\hat{i} \in \cI \backslash \{i\}} \rho_{\hat{i}}  \in (0, \frac{\mu_j}{3}]$. As a result of this perturbation, for each strict subset $\cI' \subset \cI$, the left-hand side $\sum_{\hat{i} \in \cI'} \rho_{\hat{i}}$ of \eqref{eqn:genimb} is perturbed by a random amount which has a non-atomic distribution, and hence for any $\cJ' \subseteq \cJ$ there is a probability $0$ that the perturbed value is exactly equal to the right-hand side  $\sum_{\hat{j} \in \cJ'} \mu_{\hat{j}}$ of \eqref{eqn:genimb}. Using a union bound over subset pairs $(\cI', \cJ')$ and noting that for $\cI' = \cI$ the condition \eqref{eqn:genimb} is satisfied automatically because $\sum_{\hat{i} \in \cI} \rho_{\hat{i}} = 1 > \sum_{\hat{j} \in \cJ} \mu_{\hat{j}} \geq \sum_{\hat{j} \in \cJ'} \mu_{\hat{j}} $ for all $\cJ' \subseteq \cJ$, we obtain that generalized imbalance \eqref{eqn:genimb} holds with probability 1.

It follows from generalized imbalance that $p^*$ is unique using Proposition~\ref{prop:uniqueness_of_prices}.
Note that we only need to show (ii) and (iii) above to deduce that $(i, i')$ is a difficult type pair, property (i) is unnecessary.
Define $A$ as above.
Then notice that 
since $\sum_{\hat{j} \in \cJ} \mu_{\hat{j}} < 1$, for the specified $A$ the unique shadow prices are simply $p^*(\hat{j}) =1/4$ for all $\hat{j} \in \cJ$, which implies (ii) and (iii).
In fact, for any $\delta > 0$, there exists small enough $\delta'>0$,  such that
for any instance in at $L^\infty$ distance at most $\delta'$ from the constructed instance, the following property holds:
\begin{enumerate}[label=(\roman*)]
  \item[(i')] For all job types $\hat{j} \in \cJ$, the shadow price $p^*(\hat{j})$ satisfies $|p^*(\hat{j})- \frac{1}{4}| \leq \delta$.
\end{enumerate}
We deduce property (i') from $\rho_{i} \in (0, \frac{\mu_j}{3} + \delta] \subseteq (0, \mu_j)$ and $\rho_{i'} \in [\frac{\mu_{j'}}{3}, \frac{\mu_{j'}}{3} + \eps + \delta] \subseteq [\frac{\mu_{j'}}{3}, \mu_{j'})$ for small enough $\delta$, and $|A(\hat{i}, \hat{j}) - 1/4| \leq \delta$ for $(\hat{i}, \hat{j}) \notin \{(i,j), (i', j') \}$ as follows: The proof of Proposition~\ref{prop:uniqueness_of_prices} tells us that both $p^*$ and worker type shadow prices $v^*$ are uniquely determined in the dual problem \eqref{eq:dual1} and \eqref{eq:dual2}. (Our argument will in fact independently establish this.) Our argument will control these shadow prices. Clearly all job types are fully utilized since $\sum_{\hat{j} \in \cJ} \mu_{\hat{j}} < 1$. 
Consider the path constructed in Lemma~\ref{lma:imbalancepath} from job type $\hat{j}$ on the complete bipartite graph between $\cI$ and $\cJ$ (to recap informally: the path has positive assignment probability along every edge, type $j$ is one endpoint, and every job type and worker type on the path is fully assigned except the job type which is the other end point). It must end in the unmatched job type $\kappa$ since all job types in $\cJ$ are fully utilized. Let the path be $\kappa\text{---}i^1\text{---}j^1\text{---}\ldots\text{---}i^\ell\text{---}j^\ell \!=\! \hat{j}$. Complementary slackness tells us that
\begin{align*}
  v_{i^1}^* &= 0 \\
  v_{i^l}^* + p^*_{j^l} &= A(i^l, j^l) \qquad \forall \, l = 1, 2, \dots, \ell\, , \\
v_{i^{l+1}}^* + p^*_{j^l} &= A(i^{l+1}, j^l) \qquad \forall \, l = 1, 2, \dots, \ell-1\, .
\end{align*}
Is it possible that $(i^1, j^1)$ is one of the ``heavy'' edges, $(i^1, j^1) \in \{(i,j), (i', j') \}$? If so, then $p^*_{j^1} = A(i^1, j^1) \geq \frac{1}{2}- \delta' > \frac{1}{4} + \delta' \geq A(i^2, j^1)$ for $\delta'$ small enough, i.e., $p^*_{j^1} > A(i^2, j^1)$ which contradicts positive assignment probability $x^*(i^2, j^1) > 0$. We deduce that $(i^1, j^1) \notin \{(i,j), (i', j') \}$ and hence that $|p^*_{j^1}-\frac{1}{4}| \leq \delta'$. This immediately implies (using the complementary slackness condition above) that $(i^2, j^1)$ is also not one of the ``heavy'' edges, $(i^2, j^1) \notin \{(i,j), (i', j') \}$, and we deduce from $|A(i^2,j^1)-\frac{1}{4}| \leq \delta'$ that $v^*_{i^2} \leq 2 \delta'$. Proceeding inductively along the path, we  infer that for $\delta'$ small enough, no edge along the path is in $\{(i,j), (i', j') \}$, and that
\begin{align}
  |p^*_{j^l} - \tfrac{1}{4}| &\leq (2l -1)\delta' \qquad \forall \, l = 1, 2, \dots, \ell\\
  v^*_{i^l} &\leq 2(l-1)\delta'\, \qquad \forall \, l = 1, 2, \dots, \ell.
\end{align}
In particular, $|p^*(\hat{j})- \frac{1}{4}| \leq 2|\cI| \delta' \leq \delta$. Repeating for all $\hat{j} \in \cJ$, we deduce property (i') for $\delta'$ small enough.
From (i') we immediately deduce that the heavy edges must be saturated in any optimal assignment, i.e., (ii) and (iii) hold, and it follows that $(i,i')$ is a difficult type pair as required.\hfill $\Box$
\end{proof}

\section{Appendix to Section~\ref{sec:deemplus}}
\subsection{DEEM-discrete: An extension of DEEM to finite settings}\label{sec:deemdiscrete}
In this section, we develop a discrete version of DEEM for a finite setting that is a close discrete analogue of our continuum model studied in the main text. We call this policy \emph{DEEM-discrete}.
\footnote{Since we strive to be close to the continuum setting here, the model will be slightly different than the environment in Section~\ref{sec:simulations} under which we do our simulation study of \DEEMplus: in Section~\ref{sec:simulations}, the jobs arrivals are ``spread out'' over the duration of each period and moreover jobs join a virtual queue upon arrival. Here, unlike in Section~\ref{sec:simulations}, we will assume that jobs arrive synchronously at each time and do not queue up.} We shall see via an example that DEEM-discrete works quite well for large $N$.
However, 
since DEEM was designed to minimize the leading order term of regret \emph{asymptotically} in $N$, DEEM-discrete may not always perform well for practical values of $N$. We discuss this in more detail in Remark~\ref{rem:deemdisc} later in this section. 

Consider the following finite analog of the continuum model of workers and jobs. Time is discrete as before. 
At each time $t= 0, 1, 2, \dots$, an integer $M$ workers arrive into the system. Each worker stays in the system for $N$ time periods. Hence, for each $t= N-1, N, N+1, N+2, \dots$, there are exactly $MN$ workers who are present in the system. The type of each arriving worker is independently sampled from the distribution $(\rho_i)_{i\in\cI}$. At each time $t$, $C_j=\lceil{MN\mu_j}\rceil$ jobs of type $j$ arrive for each\footnote{DEEM-discrete can be easily extended to cases where $M$ and $C_j$ are random quantities in a straightforward manner.} $j\in\cJ$. In each period, the assignment policy chooses how to match the available workers to available jobs. If a worker type $i$ is matched to a job type $j$, then the payoff is Bernoulli$(A(i,j))$. Jobs left unmatched leave the system at the end of each period.

We now develop {\em DEEM-discrete}, a translation of DEEM to the model above.  DEEM-discrete is also a WHO policy, i.e., it acts independently on each worker, based exclusively on that worker's history.  DEEM-discrete inherits the two-phase structure of DEEM: Explore, then Exploit.  The details of DEEM-discrete are as follows.
\begin{itemize}
\item {\bf Explore phase.} This phase is identical to that in DEEM.
\item {\bf Exploit phase.} This phase is identical to that in DEEM, except that a different randomized allocation policy $\yd$ is utilized in lieu of $y^*$, where $\yd$ depends on the parameter $M$ that controls the ``system size'', in addition to $N$ and the model preliminaries. The details of the computation of $\yd$ are presented below.
\item {\bf Job capacity constraints.}  If a particular job allocation is desired for a worker under the policy and if the corresponding job type is unavailable, then the policy doesn't assign any job to the worker in that time period and attempts to resume with the prescribed allocations in the next time period.\footnote{Unlike our continuum model, we are not assuming here that the set of job types $\cJ$ includes a type $\kappa$ with infinite capacity and $0$ payoff.}
\end{itemize}

The main difference between DEEM and DEEM-discrete at a worker level is the allocation policy during the Exploit phase ($\yd$ vs. $y^*$). The main idea behind the design of this policy is as follows.
Recall that DEEM is a WHO policy and hence, in principle, can be directly implemented independently for each worker in this discrete setting.  Recall that DEEM satisfies capacity constraints {\em exactly} in the continuum model, by design of $y^*$.  However, the analogous implementation in the finite model will only satisfy capacity constraints {\em on average}. In particular, the stochastic nature of job assignments resulting from the randomness in payoffs and the randomized choices of DEEM could result in capacity violations and hence job unavailability; {without the exact law of large numbers (ELLN) exploited in the continuum setting, such capacity violations have positive probability.}

{The possibility of capacity violations} is addressed by the new randomized allocation policy $\yd$ that is carefully designed for the Exploit phase of DEEM-discrete, which ensures that the capacity constraints are not violated with high probability. The key idea is to pretend that the job capacities are {\em slightly lower} than their true values, so that the slack absorbs any fluctuations in the demand resulting from the assignment policy. As the system size, i.e., $M$, becomes large, this slack tends to zero and we asymptotically achieve the optimal regret rate of the continuum model.


The randomized allocation policy $\yd$ to be used in the Exploit phase of DEEM-discrete is computed in the two steps described below.
\begin{enumerate}
\item{\bf Estimate resource consumption during Explore under DEEM-discrete via simulation.}
In this step, we estimate the following three quantities pertaining to the Explore phase of DEEM-discrete (this is the same as the Explore phase of DEEM, as defined in Figure~\ref{fig:def-DEEM}).
{\begin{align}
\tilde{\rho}_i & \triangleq \textup{probability that a worker gets labeled type $i$ at the end of Explore,}\label{eq:est1}\\
N^{\textup{xplr}}_i &\triangleq \textup{the mean duration of Explore for worker labeled type $i$,}\label{eq:est2}\\
c^{\textup{xplr}}_j &\triangleq \textup{the mean number of jobs of type $j$ assigned to a worker during Explore.}\label{eq:est3}
\end{align}}
These three quantities can easily be estimated to arbitrary precision through simulations. {Consequently, for simplicity, we will not make any notational distinction between these quantities (and other quantities to follow) and their estimates.}

To construct a sample Explore phase, 
 first we sample a worker type $i$ from the distribution $(\rho_i)_{i\in\cI}$. We then implement the Explore phase of DEEM for this worker. We obtain the label $i'$ at the end of Explore (potentially $i'\neq i$). We also obtain the time until Explore finishes, $n^{\textup{xplr}}_{i'}$. If the exit from Explore condition is not satisfied until time $N$ (line 15 of DEEM), then we define $n^{\textup{xplr}}_{i'}\triangleq N$ and we define the label $i'$ to be the MAP estimate of the type at $N$. Finally, we obtain the number of jobs allocated $c_j$ for each job type $j$ until the end of Explore. After generating $K$ 
such samples (we use a sample size of $K=1000000$ in our illustrative example to follow), we define $c^{\textup{xplr}}_j$ to be the average across the samples of $c_j$, i.e., the jobs consumed of type $j$.
We define $\tilde{\rho}_i$ be the fraction of workers labeled as type $i$ across the samples. Finally, we define $N^{\textup{xplr}}_i$ to be the average of $n^{\textup{xplr}}_{i}$ for each $i\in\cI$ across those samples for which the labeled type was $i$.


\item{\bf Adjust remainder capacity and compute $\yd$.}
{From our calculations of step $1$, we obtain estimates for the following system-level quantities at steady state.
 \begin{align}
M^{\textup{xplt}}_i &\triangleq M\tilde{\rho}_i(N- N^{\textup{xplr}}_{i}): \textup{mean number of workers in Exploit labeled type $i$,}\label{eq:est4}\\
C^{\textup{xplr}}_j&\triangleq Mc^{\textup{xplr}}_j: \textup{mean demand for type $j$ jobs from workers in Explore.}\label{eq:est5}
\end{align}
Here, $\tilde{\rho}_i$, $N^{\textup{xplr}}_{i}$, and $c^{\textup{xplr}}_j$ are defined in \eqref{eq:est1}, \eqref{eq:est2}, and \eqref{eq:est3} respectively. Note again that these estimates can be made arbitrarily accurate by increasing the sample size $K$.}

A naive approach to the Exploit phase would involve maximizing the expected payoff assuming the worker type labels are correct, under the constraint that the expected number of job allocations are at most the mean number of jobs available to allocate in the Exploit phase (this is a simple linear program). Note that the mean number of jobs of type $j$ available to allocate to workers in the Exploit phase is $C_j-C^{\textup{xplr}}_j$. However, as noted above, the randomness in the allocations (both during Explore and Exploit phases) could lead to violation of the capacity constraints. We address this concern by defining DEEM-discrete such that the expected number of jobs of type $j$ assigned to workers per period in steady state does not exceed a ``reduced'' capacity
slightly smaller than $C_j$. We define this reduced capacity to be $\frac{C_j}{1+\delta_j}$
for
\begin{align}
  \delta_j \triangleq 
  \sqrt{\frac{2\log (MN)}{MN\mu_j}} \, ,
  \label{eq:delta-j-def}
\end{align}
{The choice of $\delta_j$ results from a careful balance of the additional regret due to the reduced capacity and that due to capacity violations; we discuss this choice in detail in Section~\ref{sec:deem-disc-regret} below.} We then define the following adjusted leftover capacity of job $j$ that can be allocated to workers in the Exploit phase.\footnote{The quantity $C_j/(1+\delta_j) - C^{\textup{xplr}}_j$ for each $j$ is positive for a large enough $N$ and $M$, since at most $\textup{O}(\log N/N)$ fraction of type $j$ jobs are consumed in expectation during Explore, and a large $M$ and $N$ ensures that $\delta_j$ is small enough. See also Remark~\ref{rem:deemdisc} to follow.}
 \begin{align}
&C^{\textup{xplt}}_j \triangleq \max\big(\frac{C_j}{1+\delta_j} - C^{\textup{xplr}}_j,0\big): \textup{capacity of job type $j$ available for workers in Exploit.} \label{eq:est6}
\end{align}

We are now in a position to compute a randomized job allocation policy for the workers in the Exploit phase. This policy solves the following linear program.

\begin{align}
\text{maximize}\ \ \  & \sum_{i \in \cC}  M^{\textup{xplt}}_i \sum_{j\in \cS} y(i,j) A(i,j) \label{eq:optdisc1}\\
\text{subject to}\ \ \ & \sum_{i\in \cC} M^{\textup{xplt}}_iy(i,j) \leq C^{\textup{xplt}}_j\qquad \forall j \in \cS \, ; \label{eq:optdisc2}\\
& y \in \cD.\label{eq:optdisc3}
\end{align}
Here, $M^{\textup{xplt}}_i$ and $C^{\textup{xplt}}_j$ are defined in \eqref{eq:est4} and \eqref{eq:est6}, respectively. We denote the resulting optimal randomized job allocation policy by $y^D$, i.e., a worker labeled as type $i$ at the end of the Exploit phase is assigned a job of type drawn from the distribution $\yd(i,\cdot)$.

\begin{remark}\label{rem:deemdisc}
Asymptotically in  $N$,  a vanishing fraction of the jobs are utilized for assignments to workers who are currently in the Explore phase. However, for small values of $N$, this fraction could be significant. In fact, certain job types $j$ could be completely exhausted by workers who are in the Explore phase, and hence unavailable for workers in the Exploit phase; i.e., $C^{\textup{xplt}}_j = 0$.
In these scenarios, the performance of DEEM-discrete is expected to be poor. The main issue is that in this regime, the shadow prices $p^*$ from the full information linear program \eqref{eq:opt1}-\eqref{eq:opt3} do not capture the externalities imposed by the capacity constraints correctly, since the Explore phase accounts for a non-vanishing fraction of jobs. The more sophisticated algorithm \DEEMplus presented in Section~\ref{sec:deemplus} effectively addresses this issue by 
estimating the appropriate shadow prices from the queue-length information of the jobs, irrespective of the value of $N$.
\end{remark}
\end{enumerate}


\subsubsection{An example}\label{apx:exdiscrete}
We consider the example that we discussed in Section~\ref{subsec:example}.
Recall there were three worker types (in order): Programmers, Designers, and All-rounders; and three job types (in order): Programming, Design, and Mixed. The payoff matrix (consistent with the subsets of relevant skills) was:
\begin{align}
A = \left [
\begin{matrix}
  0.5 & 0.2 & 0.1\\
  0.3 & 0.8 & 0.2\\
  0.5 & 0.8 & 0.6
\end{matrix} \right ]\label{example}
\end{align}
The worker type distribution was $\rho = [\begin{matrix}0.4/1.9 & 0.6/1.9 & 0.9/1.9\end{matrix}]^T \approx  [\begin{matrix}0.2105 & 0.3157 & 0.4738\end{matrix}]^T$ and the job capacities were $\mu = [\begin{matrix}1/1.9 & 1/1.9 & 1/1.9\end{matrix}]^T$.

First, for $N \in\{ 50, 75, 125, 250\}$, we estimate $\tilde{\rho}$, $(N^{\textup{xplr}}_i)_{i\in\cI}$, and $(c^{\textup{xplr}}_j)_{j\in\cJ}$ under DEEM (we use a sample size of $K = 1000000$ worker lifetime trajectories while computing these estimates). These numbers are presented in Table~\ref{tbl:stepone} below.
\begin{table}[h]
\centering
\begin{tabular}{c c c c}
         \toprule
         $N$ &  $\tilde{\rho}$  & $(N^{\textup{xplr}}_i)_{i\in\cI}$  & $(c^{\textup{xplr}}_j)_{j\in\cJ}$ \\ \midrule
          50           & (0.212, 0.362, 0.426) &   (12.3, 17.1, 15.6) & (3.86, 5.54, 6.05)  \\
          75   & (0.211, 0.361, 0.428)  &  (13.3, 18.1, 17.3)  & (4.11, 5.90, 6.72)\\
         125           & (0.211, 0.359, 0.430) &   (14.3, 19.5, 19.4) & (4.48, 6.53, 7.39)  \\
          250  & (0.211, 0.353, 0.437)  &  (15.3, 21.4, 22.0)  & (4.83, 7.09, 8.44)\\ \bottomrule
  \end{tabular}
 \vspace{0.1in}
 \caption{Estimates of various mean quantities pertaining to the Explore phase (see  \eqref{eq:est1}, \eqref{eq:est2}, and \eqref{eq:est3}) for different values of $N$. The sample size is $1000000$. The largest standard errors in these quantities are 0.014 for the $N^{\textup{xplr}}_i$ and $c^{\textup{xplr}}_j$, and $0.0005$ for the $\tilde{\rho}_i$, across all $i\in\cI$, $j\in\cJ$ and $N$.  }
       \label{tbl:stepone}
\end{table}
There are a few illustrative points to note in this table, which we briefly discuss.

 First, the workers do not get perfectly labeled and so the distribution of worker labels is different from the distribution of true worker types. In particular, a comparison of the distribution of worker labels to the distribution of true worker types, which is $[\begin{matrix}0.2105 & 0.3157 & 0.4738\end{matrix}]^T$, suggests the possibility that a significant fraction of All-rounders are labeled as Designers, which we verify is indeed true based on the misclassification rates for the different pairs of types (we suppress the detailed numerics for brevity).
This is expected since, as we discussed in Section~\ref{subsec:example}, Designers are only weakly distinguished from All-rounders under DEEM. Note that the data suggests that this misclassification rate decreases as $N$ becomes large, which we again verify to be true.

Second, the fraction of time spent in the Explore phase for the different types decreases with $N$ (and the absolute length of the Explore phase increases very slowly with $N$). This is expected given that the expected length of the Explore phase is $\textup{O}(\log N/N)$.

Finally, note that even for the smallest value of $N=50$, the total number of jobs of type $j$ utilized in the Explore phase, $C^{\textup{xplr}}_j$ ($= Mc^{\textup{xplr}}_j$; see \eqref{eq:est5} and \eqref{eq:est3}), is a relatively small fraction of the total number of available  jobs, $C_j$ $(=MN\mu_j)$: the fraction $C^{\textup{xplr}}_j/C_j = c^{\textup{xplr}}_j/(N\mu_j) =   (1.9\times c^{\textup{xplr}}_j)/N$ is the largest for $j = 3$, in which case we have $C^{\textup{xplr}}_3/C_3 = (1.9\times 6.0501)/50 = 0.2299$. Hence, we can hope that the shadow prices $p^*$ from the full information linear program capture the externalities imposed by the capacity constraints reasonably well for the $(\rho, \mu, A)$ and $N$ values considered here.


Moving on to the next step, we consider $M\in \{50,100,200\}$, and for different values of $M$ and $N$, we calculate $\yd$ by solving \eqref{eq:optdisc1}-\eqref{eq:optdisc3}, for $(C^{\textup{xplt}}_j)_{j \in \cJ}$ and $(M^{\textup{xplt}}_i)_{i \in \cI}$ computed using the simulation-based estimates above (see \eqref{eq:est4}, \eqref{eq:est5}, and \eqref{eq:est6}).

{\bf Simulation results.} For this example, we implement DEEM-discrete in a simulated market setting (with dynamics as defined in the first paragraph of the current appendix section) using these computed job allocation policies $\yd$ for the different values of $N$ and $M$. Because of the adjustment to the capacities in the computation of $\yd$, the job capacity constraints were satisfied in all periods in all of the settings in our simulations. In Table~\ref{tbl:perfdiscrete}, we present the ratios of the average payoff obtained per worker per period by the platform at steady state and the optimal per worker per period payoff if the worker types were known a priori to the platform, for different values of $M$ and $N$. We refer to this ratio as the performance ratio (PR). The numbers show that the proposed translation of DEEM to a discrete market setting performs reasonably well, especially for a large $N$. Also, as expected, the PR increases with both $M$ and $N$.
\begin{table}[h]
\centering
\begin{tabular}{c c c c}
         \toprule
         $N$ & PR ($M=50$) & PR ($M=100$) & PR ($M=200$)\\ \midrule
          50           & 0.883 &   0.913  & 0.919 \\
          75   & 0.904  &  0.933 & 0.941\\
         125     & 0.921     & 0.951  &  0.959  \\
        250 & 0.938  &  0.966 &  0.975\\ \bottomrule
  \end{tabular}
  \vspace{0.1in}
  \caption{The performance ratios (PR) for different values of $M$ and $N$.}
       \label{tbl:perfdiscrete}
\end{table}

{To summarize, DEEM-discrete is a natural translation of DEEM to a discrete analogue of our continuum model. This translation works well in instances where a small fraction of the total supply of jobs is utilized for exploration (as is the case in the example above, and in general when $N$ is sufficiently large), but is not expected to work well in settings where a large fraction of the total supply of jobs is utilized for exploration, as discussed in Remark~\ref{rem:deemdisc} above. Instead, our practical recommendation is our heuristic \DEEMplus along with the queue-based computation of shadow prices, which works well for practical values of $N$.}

\subsubsection{The choice of $\delta_j$ and the regret of DEEM-discrete.}\label{sec:deem-disc-regret}  {The justification for the specific choice of $\delta_j$ (see \eqref{eq:delta-j-def}) in the definition of the adjusted capacity $C^{\textup{xplt}}_j$ in DEEM-discrete (see \eqref{eq:est6}) is as follows. }We know that there are $MN$ workers in the system available at any given time. {Since DEEM-discrete is a WHO policy, worker histories are independent of each other (assuming no capacity violations) and so the total number of requests of job type $j$ is a sum of $MN$ independent Bernoulli random variables.}
Suppose that we ensure that the mean number of requests of job type $j$ at any time is less than or equal to a reduced capacity of $\frac{C_j}{1+\delta_j}$. Then using the multiplicative Chernoff bound, 
the probability that there is a capacity violation of job type $j$ in any arbitrary period is at most $\exp(-\frac{\delta_j^2 C_j}{(1+\delta_j)(2+\delta_j)}) \leq \exp(-\frac{\delta_j^2C_j}{4})$ for any $\delta_j \leq 1/2$. 
Recall that if DEEM-discrete is unable to assign a job to a worker due to unavailability, then the policy doesn't assign any job to the worker in that time period and attempts to resume with the allocations in the next time period. Thus if a capacity violation occurs in a particular time period, then in the worst case, a worker will face an additional regret of $1$. Hence the expected additional regret per worker over her lifetime is at most $N\sum_{j\in\cJ}\exp(-\frac{\delta_j^2C_j}{\gamma})$. And hence, the per worker per period additional regret due to capacity violations is at most $\sum_{j\in\cJ}\exp(-\frac{\delta_j^2C_j}{\gamma})$.

Now we try to balance two costs: the additional regret due to the reduction of job capacities which is at most\footnote{Since the modified capacity for job type $j$ is $C_j/(1+\delta_j)\geq C_j(1-\delta_j)$, the capacity reduction is at most $C_j\delta_j\geq MN\mu_j\delta_j$. These ``lost'' jobs each reduce payoff by at most $1$, so the additional regret per worker per period is at most $MN\mu_j\delta_j/(MN) =\mu_j\delta_j$ for each job type $j \in \cJ$. 
} $\sum_{j\in\cJ}\mu_j \delta_j$, and the additional regret due to the capacity violations which is at most $\sum_{j\in\cJ}\exp(-\frac{\delta_j^2C_j}{4})$ as we argued above. It is easy to verify that the choice of\footnote{We write $f(t) = \tilde{\textup{O}}(g(t))$ if there exists a $k>0$ such that $f(t) = \textup{O}(g(t)(\log t)^k)$.} $\delta_j$ as in \eqref{eq:delta-j-def} balances these two costs for each $j \in \cJ$, resulting in a total additional regret of $\tilde{\textup{O}}(1/\sqrt{MN})$ for $(MN)\rightarrow\infty$ (for example, one can consider fixed $N$ and let the system size $M \to \infty$). 
We want $\delta_j \leq 1/2$ for all $j \in \cJ$ so that the Chernoff bound holds (see above), and we ensure this by requiring  $M$ and $N$ to satisfy
$$\frac{\log (MN)}{MN} \leq \tfrac{1}{8} \min_{j \in \cJ} \mu_j \, . $$
The bound on regret obtained from the above argument is captured in the following remark. 
\begin{remark}
\label{rem:DEEM-discrete-addnl-regret}
For any $M$ and $N$ that satisfy
$$\frac{\log (MN)}{MN} \leq \tfrac{1}{8} \min_{j \in \cJ} \mu_j \, ,$$
the additional regret of DEEM-discrete in the above finite setting relative to the regret under DEEM in the continuum setting of Section~\ref{sec:model} is bounded above as
$$\textup{Regret(DEEM-discrete)} - \textup{Regret(DEEM)}  \leq \frac{1}{\sqrt{MN}} \Big ( |\cJ| + \sqrt{2 \log (MN)} \sum_{j \in \cJ} \sqrt{\mu_j}\Big ) \, .$$
\end{remark}

\section{Appendix to Section~\ref{sec:simulations}}
\subsection{Description of instances used in our simulations} \label{apx:instances}


In this section, we describe the instances that we used in our simulations in Section~\ref{sec:simulations}.
Building upon our example in Section~\ref{subsec:example}, we consider instances with $4$ types of workers arising from a two dimensional skill set, namely Programming (P) and Design (D). Each type either has or doesn't have a skill. Let the type space be defined as $\cI=\{ 00,\, 01,\, 10,\,11\}$ where the first (resp., second) bit of each type signifies whether or not the worker has skill P (resp., D). There are $3$ types of jobs, denoted by the set $\cJ= \{\text{Programming},\, \text{Design},\, \text{Mixed}\}$. Programming jobs benefit from skill P, Design jobs benefit from skill D, and Mixed jobs require both skills P and D. Formally, Programming jobs have they same payoff for $00$ and $01$ worker types (and thus cannot distinguish them); and they also have the same payoff for $10$ and $11$ worker types (and thus cannot distinguish them).  Similarly, Design jobs have the same payoff for $00$ and $10$ worker types, and for $01$ and $11$ worker types. Mixed jobs can distinguish all worker types. Thus the $4\times 3$ $A$ matrix is fully specified by $8$ payoff entries: (1) Programming job payoffs for $00/01$ and $10/11$ worker types; (2) Design job payoffs for $00/10$ and $01/11$ worker types; and (3) Mixed job payoffs for all four types.

The interpretation of types based on skill levels implies a natural order on the payoffs: a Programming job's expected payoff for the $00/01$ types will be smaller than that for the $10/11$ types, and similarly a Design job's expected payoff for the $00/10$ types will be smaller than that for the $01/11$ types.  Finally, we arbitrarily assume that Programming skills are more important than Design skills for Mixed jobs, and hence the expected payoff for these jobs is increasing in the type sequence $00$, $01$, $10$, and $11$ (this assumption is without loss of generality since we can swap the roles of Programming and Design arbitrarily across instances). Notice that an $A$ matrix with this structure satisfies Assumption \ref{ass:non-triviality-of-learning}.

We assume that $\rho_i = 0.25$ for each worker type $i$.  We first generated $10,000$ instances, where for each instance: (1) $\mu_j$ is sampled independently across $j$ from a uniform distribution on $[1/6,1/2]$ (thus, the expected sum of job arrival rates is $1$, equal to the mass of workers in the system at any time); and (2) for each of the job types, the payoffs for different worker types are assigned to be the appropriate order statistics of independent uniformly generated random variables in $[0,1]$, so that the payoffs are monotonic in skills.
We found that out of the $10000$ instances, $8301$ instances, i.e., $\approx$ $83\%$, had a difficult type pair. The fact that a non-trivial fraction of instances have a difficult type pair is consistent with Proposition~\ref{prop:difficulty-frequent-with-multiple-skills}. We then randomly chose $350$ instances of those that had at least one difficult type pair and focused on these instances for our simulations.

\subsection{Using PD-control to stabilize queue-lengths under \DEEMplus}\label{sec:pd-control}

In the implementation of $\DEEMplus$ in our simulations in Section~\ref{sec:simulations}, we use PD-control on the prices with the goal of stabilizing queue-lengths. These prices consist of two terms: the proportional term, and the derivative term. The proportional term captures the idea that if the queue length is small then this signifies that the job type is in high demand and hence the price for this type should be high. Thus, if the queue length of job type $j$ at any epoch $l$ is $q_j(l)$, the proportional component of the price of $j$ at any time between that epoch and the next is set to be
\begin{equation}\label{eq:proportional-term}
p^{P}_j(l) = (B-q_j(l))/B
\end{equation} where $B$ is the buffer size of the queue. A purely proportional price control does not necessarily stabilize queue-lengths; it may lead to oscillations of queue-sizes and hence of the prices, which could be detrimental to performance.

In order to dampen possible oscillations we add a \emph{derivative} term. Let $q_j(l)$ be the queue-length of job type $j$ at job arrival/assignment epoch $l$. For a fixed window size $W$, we keep track of the moving-average queue-length $q_j^{\textup{avg, }W}(l)$, defined as
$$q^{\textup{avg, }W}_j(l)=(1-1/W) q^{\textup{avg, }W}_j(l-1)+1/Wq_j(l).$$
Informally, $(q_j(l)-q_j^{\textup{avg},W}(l))/W$ 
is an estimate of the derivative of the mean queue-length at epoch $l$ (the window size should be small enough, but not too small, so that stochastic fluctuations are averaged out). Then we define the derivative term of the prices to be
\begin{equation}\label{eq:derivative-term}
p^D_j(l) = -\frac{\zeta}{B}(q_j(l)-q_j^{\textup{avg, }W}(l)).
\end{equation}
Finally, the price for each job type $j$ between epochs $l$ and $l+1$ is given by
$$p^q_j(l) = p^P_j(l) + p^D_j(l).$$
It remains to decide the values of $W$ and $\zeta$.

In our simulations, we use two derivative terms as per \eqref{eq:derivative-term} corresponding to window sizes $W_1=2400$ and $W_2= 2400/1.8\approx 1333$. For both these terms, we choose a common value of $\zeta = 5$.  All these parameters were chosen through trial-and-error to achieve a reasonable degree of stability in the prices (c.f., Table~\ref{tbl:p}) and they were held constant throughout our simulations.

\subsection{Definitions of PA-TS and TS-\DEEMplus}\label{apx:other}
The two policies PA-TS and TS-\DEEMplus are defined in Figures~\ref{fig:def-PA-TS} and \ref{fig:def-DTS}.
\begin{figure}[h]
\fbox{{\footnotesize\begin{minipage}{\textwidth}
\begin{center} {\vspace{0.1in}\normalsize \bf PA-TS} \vspace{0.1in}\end{center}
\textbf{Input parameters:} $\cI$, $\cJ$, $A$, $\rho$, $N$, queue-based prices $p^q$.\\
{\bf Pre-compute:} The sets $\cJ^q(i)$ defined in \eqref{eqn:opt-jobs-queue-prices} for each worker type $i\in\cI$.
\medskip
\hrule
\medskip
\begin{algorithmic}[1]
\LineComment{Main Routine}
\Procedure{PA-TS}{} \Comment{Acts independently on each worker, over her lifetime, from arrival to departure}

\medskip
\State $\lambda(i) \gets \rho_i$ for all $i\in \cI$ \Comment{The un-normalized posterior probabilities; initialized to the prior}
\State $k \gets 0$ \Comment{Number of time steps the worker has been in the system}
\medskip


\While{$k<N$}
	\State Assign job type $j_k\sim$  \Call{TS}{$\lambda$} \Comment{At the next time step}
	\State Observe payoff $r_k$
	\State $\lambda(i) \gets \lambda(i) \times (A(i,j_k)\mathbf{1}_{\{r_k=1\}}+(1-A(i,j_k))\mathbf{1}_{\{r_k=0\}})$, for all $i \in \cI$
	\State $k \gets k+1$
\EndWhile
\medskip

 \EndProcedure
 \medskip
 \hrule
 \medskip
 \LineComment{Functions}
\Function{TS}{$\lambda$}
\State  For each $j\in\cJ$, define \Comment{Thompson sampling}
\begin{align}
\alpha_j &= \bigg(\sum_{i\in\cI}\lambda(i)\frac{\ind_{j\in\cJ^q(i)}}{|\cJ^q(i)|}\bigg)\bigg/ \sum_{i\in\cI}\lambda(i) \nonumber
\end{align}

\State {\bf return} $\alpha=(\alpha_j)_{j\in\cJ}$
\EndFunction

\end{algorithmic}
\end{minipage}}}
\linespread{1}
\caption{Definition of PA-TS. The prices $p^q$ depend on the queue-lengths of the different job types, as discussed in Section~\ref{subsec:queueprices}, and formally defined in Appendix~\ref{sec:pd-control}.}
\label{fig:def-PA-TS}
\end{figure}

\begin{figure}[htb]
\fbox{{\footnotesize\begin{minipage}{\textwidth}
\begin{center} {\vspace{0.1in}\normalsize \bf TS-\DEEMplus \vspace{0.1in}} \end{center}
\textbf{Input parameters:} $\cI$, $\cJ$, $A$, $\rho$, $N$, queue-based prices $p^q$.\\
{\bf Pre-compute:}\begin{itemize}
\item The sets $\cJ^q(i)$ defined in \eqref{eqn:opt-jobs-queue-prices} for each worker type $i\in\cI$.
\item The set of worker types $\Str^q(i)$ defined in Definition~\ref{def-weak} for all $i \in \cI$.
\item The maximal externality-adjusted payoffs $U^q(i)$ for all $i\in\cI$ defined in \eqref{def-max-utility-queue-prices} and the maximal per-step mislabeling regrets $R(i,i')$ for all $i,\,i'\in \cI$ defined in \eqref{def-onestepreg-str}--\eqref{def-onestepreg-weak}.
\end{itemize}
\medskip
\hrule
\medskip
\begin{algorithmic}[1]
\LineComment{Main Routine}
\Procedure{TS-\DEEMplus}{} \Comment{Acts independently on each worker, over her lifetime, from arrival to departure}

\LineComment{Initialization:}
\State $\lambda(i) \gets \rho_i$ for all $i\in \cI$ \Comment{The un-normalized posterior probabilities; initialized to the prior}
\State $\MAP \gets \argmax_{i\in\cI} \lambda(i)$ \Comment{Initialization of the MAP estimate}
\State $Label \gets \emptyset$ \Comment{Worker label; initially unassigned, denoted by $\emptyset$}
\State $k \gets 0$ \Comment{Number of time steps the worker has been in the system}
\LineComment{Explore phase:}
\While{$Label = \emptyset$ and $k<N$}
	\State Assign job type $j_k\sim$  \Call{Explore}{$\lambda$} \Comment{At the next time step}
	\State Observe payoff $r_k$
	\State $\lambda(i) \gets \lambda(i) \times (A(i,j_k)\mathbf{1}_{\{r_k=1\}}+(1-A(i,j_k))\mathbf{1}_{\{r_k=0\}})$, for all $i \in \cI$
	\State $\MAP \gets \argmax_{i\in\cI} \lambda(i)$
	\State \If{$\min_{i\in \Str^q(\MAP)}\frac{\lambda(\MAP)}{\lambda(i)R(\MAP,i)} \geq N$} \label{pc:label-condition}\Comment{if confirmation is complete}
		\State $Label \gets \MAP$  \Comment{Worker label assigned. Will cause while loop to exit.}
		\EndIf 
	\State $k \gets k+1$
\EndWhile
\LineComment{Exploit phase:}
\While{$k<N$}
 		\State Assign job type $j_k=$ \Call{Exploit}{$\lambda$} \Comment{At the next time step}
 	\State $k \gets k+1$
 \EndWhile
 \EndProcedure
 \medskip
 \hrule
 \medskip
 \LineComment{Functions}
\Function{Explore}{$\lambda$}
\State  For each $j\in\cJ$, define \Comment{Thompson sampling}
\begin{align}
\alpha_j &= \bigg(\sum_{i\in\cI}\lambda(i)\frac{\ind_{j\in\cJ^q(i)}}{|\cJ^q(i)|}\bigg)\bigg/ \sum_{i\in\cI}\lambda(i) \nonumber
\end{align}

\State {\bf return} $\alpha=(\alpha_j)_{j\in\cJ}$
\EndFunction
\Function{Exploit}{$\lambda$}
\State $j^* = \argmax_{j\in \cJ}\sum_{i\in\cI} \lambda(i)\big[A(i,j)-p^q_j\big]$ \Comment{Greedy}\label{pc:exploit-sample}
	\State {\bf return} $j^*$
\EndFunction

\end{algorithmic}
\end{minipage}}}
\linespread{1}
\caption{Definition of TS-\DEEMplus. The prices $p^q$ depend on the queue-lengths of the different job types, as discussed in Section~\ref{subsec:queueprices}, and formally defined in Appendix~\ref{sec:pd-control}.}
\label{fig:def-DTS}
\end{figure}

\subsection{Quality of our regret estimate for finite $N$}\label{apx:dpconfirm}

In this section we investigate the quality of our regret estimate for finite $N$. 

From the asymptotic analysis of DEEM, it is clear that optimizing the regret vs. learning tradeoff in the Confirmation mode is critical in achieving the optimal leading order term of regret. Our definition of the job sampling distribution $\alpha(i)$ in DEEM (see Figure~\ref{fig:def-alphai-ystar}) is specifically designed to achieve the smallest possible regret incurred in confirming workers of type $i \in \cI$ to leading order as $N \to \infty$ as discussed in Section~\ref{subsec:example} ``Minimizing regret during Confirmation''. This regret as a function of $\alpha(i)$ is $\frac{\log N}{N}\frac{\sum_{j\in \cS} \alpha_j(i)\big(U(i)-[A(i,j)-p^*_j]\big)}{\min_{i'\in \Str(i)}\sum_{j\in \cS}\alpha_j(i) \KL(i,i'|j)}$ and the
smallest regret achievable by optimizing over $\alpha(i) $ is hence $C(i)\frac{\log N}{N}$ for
\begin{align}
C(i)\triangleq \min_{\alpha\in\Delta(\cJ)} \frac{\sum_{j\in \cS} \alpha_j\big(U(i)-[A(i,j)-p^*_j]\big)}{\min_{i'\in \Str(i)}\sum_{j\in \cS}\alpha_j \KL(i,i'|j)} \, .
\label{eq:Ci-again}
\end{align}
where we repeated the definition \eqref{def:ci} of $C(i)$ for the convenience of the reader.
However, we do not expect $C(i)\frac{\log N}{N}$ to be a good estimate of the regret for small $N$ (we find empirically that indeed it is not).
This is the reason our heuristic \DEEMplus has a modified Confirmation mode of Explore (relative to that of DEEM) which accounts for small $N$. Corresponding to the modified Confirmation mode design of \DEEMplus, in this subsection, we define a modified regret estimate $C^+(i)$ for the regret incurred under \DEEMplus in confirming working type $i$ when $N$ is small. We find that the average regret estimate across true worker types $C^+ = \sum_{i \in \cI} \rho_i C^+(i)$ correlates strongly with the observed regret under \DEEMplus. This finding provides suggestive evidence that the proxy objective that \DEEMplus optimizes during Explore reasonably captures the true regret, and thus sheds light on why \DEEMplus does significantly better than TS-\DEEMplus; see Section~\ref{sec:simresults} (recall that the policy TS-\DEEMplus replaces the confirmation mode of \DEEMplus with externality-adjusted Thompson sampling).  



We now motivate and define the estimate $C^+(i)$ for the regret in confirming type $i$ under \DEEMplus. Given the queue-based shadow prices $(p^q_j)_{j\in\cJ}$, and given the posterior distribution over the worker type $\lambda=(\lambda(i))_{i\in\cI}$, the job type distribution $\alpha$ at any opportunity in the Confirmation mode of \DEEMplus solves the following optimization problem (see lines 43-46 in Figure~\ref{fig:def-DEEMplus-cont}).
\begin{align}
 \min_{\alpha\in \Delta(\cS)}\bigg\{\underbrace{\bigg[\sum_{i\in\cI} \lambda(i)\sum_{j\in \cS} \alpha_j\big(U^q(i)-[A(i,j)-p^q_j]\big)\bigg]}_{\text{(a)}}\underbrace{\bigg[\max_{i'\in\Str^q(\MAP)} \frac{\log N + \log R(\MAP,i')}{\sum_{j\in \cS}\alpha_j \KL(\MAP,i'|j)}\bigg]}_{\text{(b)}}\bigg\},
 \end{align}
where $\MAP \in \argmax_{i\in\cI} \lambda(i)$. This optimization problem is closely related to the optimization problem in the definition \eqref{eq:Ci-again} of $C(i)$. In particular, the second term (b) is (approximately) the expected number of opportunities it is projected to take to attain all learning goals (``confirm'') for the $\MAP$ estimate under the stationary allocation policy $\alpha$. The first term (a) approximates the expected regret per worker per job opportunity under the stationary allocation policy $\alpha$ given the posterior distribution $\lambda$.
This term (a) evolves with the posterior $\lambda$; however, on the event that the true worker type is $i$, at the end of the Guessing mode, with probability $1-\textup{O}(1/\log N)$, the posterior 
will place at most $\textup{O}(1/\log N)$ mass on the types that $i$ needs to be (either weakly or strongly) distinguished from. Thus a reasonable proxy
for this evolving optimization problem during Confirmation is the following time-independent optimization problem, for the different possibilities of $\MAP = i\in\cI$.
\begin{align}
 \min_{\alpha\in \Delta(\cS)}\bigg\{\bigg[\sum_{j\in \cS} \alpha_j\big(U^q(i)-[A(i,j)-p^q_j]\big)\bigg]\bigg[\max_{i'\in\Str^q(i)} \frac{\log N + \log R(i,i')}{\sum_{j\in \cS}\alpha_j \KL(i,i'|j)}\bigg]\bigg\}&\triangleq C^{+}(i).
 \end{align}
The objective in this problem is an approximation to the expected cumulative regret under \DEEMplus until confirmation of a type $i$ on the event that $i$ is indeed the true type (here we ignore the regret incurred in the Guessing mode).
Since the probability of this event is $\rho_i$, in summary, the design of the Confirmation mode of \DEEMplus corresponds roughly to an estimate of $C^{+} \triangleq \sum_{i\in\cI}\rho_i C^{+}(i)$ for the lifetime regret per worker that \DEEMplus is expected to incur (the resulting per-period regret estimate is $\frac{C^+}{N}$).
\begin{table}[h]
\centering
\begin{tabular}{c  c}
        \toprule
         $N$ & PCC \\ \midrule
         10           & 0.47 \\
          20   & 0.50 \\
         30     & 0.55 \\
        40 & 0.63 \\ \bottomrule
\end{tabular}
 \vspace{0.1in}
  \caption{The Pearson correlation coefficient (PCC) between $C^+$ and the actual regret incurred per worker on average under \DEEMplus (from simulation) for different values of $N$ across the $350$ instances.}
       \label{tbl:pearson}
\end{table}
To investigate if $C^+$ indeed does a good job of capturing the regret, for different values of $N\in\{10,20,30,40\}$, we compute the Pearson correlation coefficient between $C^{+}$ and the incurred regret across the $350$ test instances.\footnote{We utilize the estimated steady-state values of the queue-based prices $p^q$ from our simulations for each instance in our computations.}
The results, shown in Table~\ref{tbl:pearson}, show a significant correlation for all values of $N$ we consider, thus providing some justification for the design of our heuristic \DEEMplus.


\end{APPENDIX}


%% file: upper-bound-proof.tex
In order to prove this proposition, we will need two technical lemmas. The first lemma is the following.
\begin{lemma} \label{lma:core}
For a fixed worker, let $r_1,r_2,\cdots$ be i.i.d. random variables where $r_k$ is the outcome of choosing a job type $j_k\in\cS$ according to a distribution $\alpha\in \Delta(\cS)$, i.i.d. for each $k$. For any $\hat{i} \in \cI$, let $\mathrm{E}_{\hat{i}}[\cdot]$ denote expectations and $\mathrm{P}_{\hat{i}}(\cdot)$ denote probabilities of events when the true worker type is $\hat{i}$.
Suppose $i\in \cC$ and $\cB \subseteq \cC\setminus \{ i\}$ are such that
$$ \sum_{j\in\cS}\alpha_j \KL(i,i'|j)>0$$
for each $i'\in \cB$. Let
$$\Lambda^{\cB}_k(i)\triangleq \min_{i'\in \cB}\lambda_k(i)/\lambda_k(i')\, ,
\textup{ where } \lambda_k(\hat{i}) \triangleq \rho_{\hat{i}}\prod_{k'=1}^k \big[(A(\hat{i},j_{k'})\mathbf{1}_{\{r_{k'}=1\}}+(1-A(\hat{i},j_{k'}))\mathbf{1}_{\{r_{k'}=0\}})\big] \ \forall \hat{i} \in \cI \, .
$$ Then 
\begin{enumerate}
\item
$$ \limsup_{m \rightarrow\infty}\frac{\mathrm{E}_{i}\big [\inf\!\big \{k> 0: \Lambda^{\cB}_k(i)\geq m \big\}\big ]}{\log{m}}\leq  \frac{1}{\min_{i'\in \cB}\sum_{j\in\cS}\alpha_j \KL(i,i'|j)},$$
\item For any $a>0$,
$$\mathrm{P}_{i'}\big ( \max_{1 \leq k \leq N}\lambda_k(i)/\lambda_k(i') \, \geq \, a  \big )\leq \frac{\rho_i/\rho_{i'}}{a}\, .$$
\end{enumerate}
\end{lemma}
\begin{proof}{Proof.}
In order to prove the first statement, we need the following fact from \cite{agrawal1989asymptotically}.  Let $X_1,X_2,\ldots$ be i.i.d. r.v.'s on some finite state space $\mathcal{X}$, with marginals $p(x)$. Let $f^{(i)}:\mathcal{X}\rightarrow\mathbb{R}$ be such that $0<\mathrm{E}(f^{(i)}(X_1)<\infty$, $i\in I$, finite. Let $S^{(i)}_k = f^{(i)}(X_1) +f^{(i)}(X_2)+\cdots+f^{(i)}(X_k)$, $L^{(i)}_a = \sum_{k=1}^{\infty}\mathbf{1}_{\{\inf_{t\geq k}S^{(i)}_t\leq a\}}$ (this is the last time $k$ that $S^{(i)}_k$ takes value $\leq$ a) , and $L_a=\max_{i\in I}L^{(i)}_a$ (this is the time after which $S^{(i)}_k>a$ for all $i\in I$). Then it is shown in Lemma 4.3 in \cite{agrawal1989asymptotically} that
$$\limsup_{a\rightarrow\infty} \frac{\mathrm{E}(L_a)}{a}\leq \frac{1}{\min_{i\in I}  \mathrm{E}(f^{(i)}(X_1))}.$$
If we define $M_a= \inf\{k>0:\min_{i\in I}S^{(i)}_k > a\}$ (this is the time $k$ at which $S^{(i)}_k>a$ for all $i\in I$ for the first time), then clearly $M_a\leq L_a+1$, and thus
$$\limsup_{a\rightarrow\infty} \frac{\mathrm{E}(M_a)}{a}\leq \frac{1}{\min_{i\in I}  \mathrm{E}(f^{(i)}(X_1))}.$$
To show that the first statement follows from this, first fix $i\in \cI$ and then map
\begin{align*}
X_k &= (j_k,r_k),\\
I&=\mathcal{B},\\
f^{(i')}(X_k) &= \log\bigg(\frac{A(i,j_{k})\mathbf{1}_{\{r_{k}=1\}}+(1-A(i,j_{k}))\mathbf{1}_{\{r_{k}=0\}}}{A(i',j_{k})\mathbf{1}_{\{r_{k}=1\}}+(1-A(i',j_{k}))\mathbf{1}_{\{r_{k}=0\}}}\bigg),\\
\mathrm{E}(f^{(i')}(X_k)) &=  \sum_{j\in\cS}\alpha_j \KL(i,i'|j)>0,\\
S^{(i')}_k & = \log\bigg(\frac{\lambda_k(i)\rho_{i'}}{\lambda_k(i')\rho_i}\bigg).
\end{align*}
We can thus conclude that
$$ \limsup_{m \rightarrow\infty}\frac{\mathrm{E}_{i}\big [\inf\!\big \{k> 0: \lambda_k(i)/(\lambda_k(i')\geq m (\rho_i/\rho_{i'}) \textrm{ for all }i'\in\cB \big\}\big ]}{\log{m}}\leq  \frac{1}{\min_{i'\in \cB}\sum_{j\in\cS}\alpha_j \KL(i,i'|j)}.$$
Now defining $m'=m\max_{i'\in\cB}\rho_i/\rho_{i'}$, we have that
\begin{align*}
&\inf\!\big \{k> 0: \lambda_k(i)/(\lambda_k(i')\geq m (\rho_i/\rho_{i'}) \textrm{ for all }i'\in\cB \big\}\\
&~~\leq \inf\!\big \{k> 0: \lambda_k(i)/(\lambda_k(i')\geq m'\textrm{ for all }i'\in\cB \big\}\\
&~~=\inf\!\big \{k> 0: \Lambda^{\cB}_k(i)\geq m' \big\}.
\end{align*}
Thus we have,
\begin{align*}
\limsup_{m' \rightarrow\infty}\frac{\mathrm{E}_{i}\big [\inf\!\big \{k> 0: \Lambda^{\cB}_k(i)\geq m' \big\}\big ]}{\log{m'}}&\leq\limsup_{m' \rightarrow\infty}\frac{\mathrm{E}_{i}\big [\inf\!\big \{k> 0: \Lambda^{\cB}_k(i)\geq m' \big\}\big ]}{\log{m'}-\log(\max_{i'\in\cB}\rho_i/\rho_{i'})}\\
&\leq \frac{1}{\min_{i'\in \cB}\sum_{j\in\cS}\alpha_j \KL(i,i'|j)}.
\end{align*}
To show the second statement, observe that if the true worker type is $i$, then the sequence
$$\bigg(\frac{\lambda_k(i')\rho_i}{\lambda_k(i)\rho_{i'}}\bigg)_{k>0}$$
is a martingale. The statement then follows from Doob's martingale inequality (see \cite{ross2008stochastic}).\hfill $\Box$
\end{proof}
The next technical lemma considers $V$ random walks, each with positive drift. It considers the event $E$ that the running minimum of these random walks descends to some specific lower level $\underline{b}$, and shows that the expectation of the (time at which $E$ occurs)$\times$(the indicator of event $E$) is bounded uniformly  in $\ub$ and the starting values of each of the $V$ random walks. 

\begin{lemma}\label{lma:driftreversal}
Fix a positive integer $V$ and positive reals $M$, $\ub$ and $\ob$ such that $\ub < \ob$. For each $j=1,\cdots, V$, fix reals $k^{(j)} \in (\ub, \ob)$ and $m^{(j)}>0$, and  let the random walk $S^{(j)}_n$ have a deterministic starting value $k^{(j)}$ and i.i.d. steps $X^{(j)}_1, X^{(j)}_2, \ldots $ with mean $\mathrm{E}(X^{(j)}_i)= m^{(j)}$ and bounded as $|X^{(j)}_i|\leq M$,  
$$S^{(j)}_n\triangleq  k^{(j)} +\sum_{i=1}^n X^{(j)}_i \, .$$
Let $T\triangleq \inf\{n: \min_{1 \leq j \leq V} S^{(j)}_n < \ub \}$ and let $E\triangleq \{T<\infty\}$.
 Then $$\mathrm{E}[T\mathbf{1}_{E}]\leq Z$$
for some $Z = Z (V, M, (m^{(j)})_{j \leq V}) <\infty$ that does not depend on $\ub$, $\ob$, or $(k^{(j)})_{j \leq V}$.
\end{lemma}

\begin{proof}{Proof.}
 Define $k^{(j)}-\ub\triangleq z^{(j)}$.  If we define $T^{(j)}=\inf\{n: S^{(j)}_n < \ub \}$ and $E^{(j)} = \{T^{(j)}<\infty\}$ and, then we have $E\subseteq \cup_jE^{(j)}$, and thus we have
$\mathrm{E}[T\mathbf{1}_E]\leq \sum_{j=1}^V \mathrm{E}[T\mathbf{1}_{E^{(j)}}]\leq \sum_{j=1}^V \mathrm{E}[T^{(j)}\mathbf{1}_{E^{(j)}}]$. Now we have
\begin{align*}
\mathrm{E}[T^{(j)}\mathbf{1}_{E^{(j)}}] &= \sum_{n=1}^\infty  n \mathrm{P}(T^{(j)}=n)\\
&\leq \sum_{n=1}^\infty  n \mathrm{P}\Big(\sum_{i=1}^n X^j_i \leq -z^{(j)}\Big)\\
&\leq \sum_{n=1}^\infty  n \exp\Big(\tfrac{-(nm^{(j)}+z^{(j)})^2}{4nM^2}\Big)\\
&=\sum_{n=1}^\infty  n \exp\Big(-\tfrac{n(m^{(j)})^2}{4M^2} - \tfrac{m^{(j)}z^{(j)}}{2M^2} - \tfrac{(z^{(j)})^2}{4nM^2}\Big)\\
&\leq \sum_{n=1}^\infty  n \exp\left (-\tfrac{n(m^{(j)})^2}{4M^2}\right ) = Z(m^{(j)},M) < \infty,
\end{align*}
where the second inequality results from the Hoeffding bound. Taking $Z= \sum_{j=1}^VZ(m^{(j)},M)$ proves the result.\hfill $\Box$
\end{proof}
We can now proceed to prove Proposition \ref{prop:upperbound}.
\begin{proof}{Proof of Proposition \ref{prop:upperbound}.}
Let $X$ denote the type of the worker. Let $\textup{Reg}(i)$ denote the expected total regret over the lifetime of a worker on the event $\{X=i\}$, defined as
$$\textup{Reg}(i)=N\max_{j\in\cS}[A(i,j)- p^*_j]-N\sum_{j\in\cS}x_{\pi^*}(i,j)[A(i,j)- p^*_j)].$$
Here $Nx_{\pi^*}(i,j)$ is the expected total number of times a job of type $j$ is allotted to a worker of type $i$ under the policy $\pi^*(N)$. We will refer to the above quantity as just regret.  For the rest of the proof, all the expectations are on the event $\{X=i\}$. The proof will utilize the fact that the log of the ratio of the posteriors, $\log(\lambda_k(i)/\lambda_k(i'))$, for any $i$ and $i'$ is a random walk. That is, if $\alpha^{(k)}$ is the probability distribution over job types chosen at opportunity $k$, $j_k$ is the random job chosen, and $r_k$ is the random reward obtained, then $$\log(\frac{\lambda_{k+1}(i)}{\lambda_{k+1}(i')}) - \log(\frac{\lambda_k(i)}{\lambda_k(i')}) = \log\frac{A(i,j_k)\mathbf{1}_{\{r_k=1\}}+(1-A(i,j_k))\mathbf{1}_{\{r_k=0\}}}{A(i',j_k)\mathbf{1}_{\{r_k=1\}}+(1-A(i',j_k))\mathbf{1}_{\{r_k=0\}}}\triangleq \Delta_k,$$
where the random variables $\big \{\Delta_k\big \}_k$ are independent random variables with a finite support (since $r_k$ and $j_k$ take finite values), and with mean $\sum_j\alpha^{(k)}_j\KL(i,i'|j)$, which we will refer to as the drift of the random walk at opportunity $k$. Note here that if $\sum_j\alpha^{(k)}_j\KL(i,i'|j)=0$ then since $\KL(i,i'|j)\geq 0$, it must be that $\KL(i,i'|j) = 0$ for all $j$ such that $\alpha^k_j >0$, and in this case we must have $A(i,j)= A(i,j')$ for all such $j$. Thus $\Delta_k= 0$, i.e., the if the drift of the random walk is $0$ at some $k$ then the random walk has stopped. Additionally, recall that $\Delta_0 = \log(\rho_i/\rho_{i'})$.

Our goal is to compute an upper bound on $\textup{Reg}(i)$. To do so we first compute the expected regret incurred till the end of the exploration phase in our algorithm. Denote this by $\textup{Reg}_{\textup{xplr}}(i)$. Below, we will find an upper bound on this regret assuming that the worker performs an unbounded number of jobs. Clearly, the same bound holds on the expected regret until the end of exploration phase if the worker leaves after $N$ jobs.

Our strategy is as follows: we will decompose the regret till the end of exploration into the regret incurred till the first time one of the following two events occurs:
\begin{enumerate}
\item Event $\textup{Good}$:  $\min_{i'\neq i}\log(\lambda_k(i)/\lambda_k(i'))\geq \log\log N$ 
\item Event $\textup{Bad}$:   $\min_{i'\neq i}\log(\lambda_k(i)/\lambda_k(i'))\leq -\log\log N$ 
\end{enumerate}
followed by the residual regret, which will depend on which event occurred first (one of these two events will occur with probability $1$). The event $\textup{Good}$ occurs when the algorithm exits the Guessing mode by correctly identifying $i$ as a guess, and the event $\textup{Bad}$ occurs when for some $i'\neq i$, the log posterior odd $\log(\lambda_k(i)/\lambda_k(i'))$ crosses the lower threshold of $-\log\log N$.

We will compute two different upper bounds, depending on two different regimes of initial prior distributions of the different types. 
Define the following two sets:
\begin{align*}
\mathcal{L}_1\triangleq \bigg\{\rho'\in\Delta(\cI):& -\log\log N \leq \min_{i'\neq i}\log(\rho'_i/\rho'_{i'})\leq \log\log N \textrm{ and}  \\
&\min_{i'\neq i}\log(\rho'_i/\rho'_{i'})\geq \min_{i'\neq i}\log(\rho_i/\rho_{i'})\bigg\} \textrm{ and }\\
\mathcal{L}_2\triangleq \bigg\{\rho'\in\Delta(\cI):& -\log\log N \leq \min_{i'\neq i}\log(\rho'_i/\rho'_{i'})\leq \log\log N \textrm{ and}\\
& \min_{i'\neq i}\log(\rho'_i/\rho'_{i'})< \min_{i'\neq i}\log(\rho_i/\rho_{i'})\bigg\}.
\end{align*}
Let $\textup{Reg}^{(1)}(i)$ ($\textup{Reg}^{(2)}(i)$) be the highest expected regret incurred over all possible priors in $\mathcal{L}_1$ ($\mathcal{L}_2$). Then clearly, $\textup{Reg}_{\textup{xplr}}(i)\leq \textup{Reg}^{(1)}(i)$.


Let $G(i)$ denote the maximum expected regret incurred by the algorithm until one of $\textup{Good}$ or $\textup{Bad}$ occurs, where the maximum is taken over all possible starting priors in $\mathcal{L}_1\cup \mathcal{L}_2$. For convenience, we denote $\textup{Good}<\textup{Bad}$ as the event that $\textup{Good}$ occurs before $\textup{Bad}$ and vice versa (we use the same notation to signify the precedence order of any two events).
Thus we have for $s\in\{1,2\}$,
\begin{align*}
\textup{Reg}^{(s)}(i) &\leq G(i) + \sup_{\rho'\in \mathcal{L}_s}\mathrm{P}(\textup{Good}<\textup{Bad}\mid\rho')\mathrm{E}(\textrm{Residual regret}\mid \textup{Good},\rho') \\
&~~+ \sup_{\rho'\in \mathcal{L}_s}\mathrm{P}(\textup{Bad}<\textup{Good}\mid \rho')\mathrm{E}(\textrm{Residual regret}\mid \textup{Bad},\rho').
\end{align*}
$G(i)$ is upper-bounded as follows.
$$G(i)\leq \mathrm{E}(\inf\{k>0: \min_{i\neq i'}\log (\lambda_k(i)/\lambda_k(i'))\geq 2\log \log N\}\})=\mathrm{O}(\log\log N).$$
This inequality follows from Lemma~\ref{lma:core}, since if neither condition for event $\textup{Good}$, nor for event $\textup{Bad}$ is satisfied, then the policy in the Guessing mode, and thus all job types are utilized with positive probability. Hence the condition in the Lemma of the requirement of a positive learning rate for each distinction is satisfied.  Also, from the second statement in Lemma~\ref{lma:core}, since the posteriors in $\mathcal{L}_1$ are such that $\min_{i'\neq i}\log(\rho'_i/\rho'_{i'})\geq \min_{i'\neq i}\log(\rho_i/\rho_{i'})$, we have that $\mathrm{P}(\textup{Bad}<\textup{Good}\mid \rho')\leq \mathrm{P}(\textrm{Bad ever occurs})\leq \mathrm{O}(1)/\log N$. Finally we have $\sup_{\rho'\in \mathcal{L}_2}\mathrm{P}(\textup{Bad}<\textup{Good}\mid \rho')\leq w$ for some $w\in [0,1)$. This follows from the fact that the probability that a random walk with a positive drift ever returns to its starting point is strictly less than 1.
We thus have
\begin{align}
\textup{Reg}^{(1)}(i) &\leq \mathrm{O}(\log\log N)  + \sup_{\rho'\in \mathcal{L}_1}\mathrm{E}(\textrm{Residual regret}\mid \textup{Good},\rho')\nonumber \\
&~~+ \frac{\mathrm{O}(1)}{\log N}\sup_{\rho'\in \mathcal{L}_1}\mathrm{E}(\textrm{Residual regret}\mid \textup{Bad},\rho').\\
\textup{Reg}^{(2)}(i) &\leq \mathrm{O}(\log\log N)  + \sup_{\rho'\in \mathcal{L}_2}\mathrm{E}(\textrm{Residual regret}\mid \textup{Good},\rho')\nonumber \\
&~~+ w\sup_{\rho'\in \mathcal{L}_2}\mathrm{E}(\textrm{Residual regret}\mid \textup{Bad},\rho').
\end{align}

We next find upper bounds on $\sup_{\rho'\in \mathcal{L}_s}\mathrm{E}(\textrm{Residual regret}\mid \textup{Good},\rho')$ and $\sup_{\rho'\in \mathcal{L}_s}\mathrm{E}(\textrm{Residual regret}\mid \textup{Bad},\rho')$ for $s=1,2$. First, consider $\sup_{\rho'\in \mathcal{L}_s}\mathrm{E}(\textrm{Residual regret}\mid \textup{Good},\rho')$. Now the residual regret after event $\textup{Good}$ has occured depends on which of the following two events happens next:

\begin{enumerate}
\item Event $\textup{Revert}$:  $\min_{i'\neq i}\log(\lambda_k(i)/\lambda_k(i'))< \log\log N$ 
\item Event $\textup{Confirm}$:  $i$ gets confirmed, i.e.,  $\min_{i'\in \Str(i)}\log(\lambda_k(i)/\lambda_k(i'))>\log N$ 
\end{enumerate}
Again conditional on event $\textup{Good}$, one of the two events will occur with probability $1$. We have
\begin{align*}
&\sup_{\rho'\in \mathcal{L}_s}\mathrm{E}(\textrm{Residual regret}\mid \textup{Good},\rho') \\
&~~=\sup_{\rho'\in \mathcal{L}_s}\bigg[\mathrm{E}(\textrm{Residual regret}\mid \textup{Good},\textup{Revert}<\textup{Confirm},\rho') \mathrm{P}(\textup{Revert}<\textup{Confirm}\mid \textup{Good},\rho')\\
&~~~~+ \mathrm{E}(\textrm{Residual regret}\mid \textup{Good},\textup{Confirm}<\textup{Revert},\rho') \mathrm{P}(\textup{Confirm}<\textup{Revert}\mid \textup{Good},\rho')\bigg].
\end{align*}
Now from Lemma~\ref{lma:driftreversal} it follows that
\begin{align*}
&\mathrm{E}(\textrm{Residual regret}\mid \textup{Good},\textup{Revert}<\textup{Confirm},\rho') \mathrm{P}(\textup{Revert}<\textup{Confirm}\mid \textup{Good},\rho')\\\
&~~= \mathrm{E}(\textrm{Residual regret }\ind_{\{\textup{Revert}<\textup{Confirm}\}} \mid  \textup{Good},\rho')\\
&~~\leq M + \textup{Reg}^{(1)}(i) \mathrm{P}(\textup{Revert}<\textup{Confirm}\mid \textup{Good},\rho')
\end{align*}
for some constant $M>0$ that does not depend on $\rho'$ or $N$. To see this, note that $\textup{Revert}<\textup{Confirm}$ is the event that, starting from some values between $\log\log N$ and $\log N$, the random walk $\log(\lambda_k(i)/\lambda_k(i'))$ for some $i'\neq i$ crosses the lower threshold $\log\log N$ before the random walks $\log(\lambda_k(i)/\lambda_k(i''))$ for all $i''\in \Str(i)$ cross the upper threshold $\log N$. Now when all the random walks (corresponding to all $i''\in\cI$) are between these two thresholds, the job distribution $\alpha_k$ equals $\alpha(i)$ for all $k$. In particular, the drift for the random walks corresponding to $i''\in\Str(i)$ is strictly positive. Further, as we argued earlier, if the drift for any of the other random walks is $0$, then that random walk has stopped, and such random walks can be ignored. Thus the conditions of Lemma~\ref{lma:driftreversal} are satisfied, and hence $\mathrm{E}(\textrm{(Time till $\textup{Revert}$) }\ind_{\{\textup{Revert}<\textup{Confirm}\}} |\textup{Good},\,\rho') \leq\mathrm{E}(\textrm{(Time till $\textup{Revert}$) }\ind_{\{\textup{Revert} \textrm{ happens}\}} |\textup{Good},\,\rho')\leq  G <\infty$. Since the regret per unit of time is bounded, the deduction follows.

Next, we have
\begin{align*}
& \mathrm{E}(\textrm{Residual regret}\mid \textup{Good},\textup{Confirm}<\textup{Revert},\rho') \\
&~~\leq \; \mathrm{E}(\inf\{k>0: \min_{i' \in \Str(i)}\lambda_k(i)/\lambda_k(i')\geq  N\})\sum_{j\in \cS} \alpha_j\big(U(i)-[A(i,j)-p^*_j]\big).
\end{align*}
Thus we have
\begin{align*}
&\sup_{\rho'\in \mathcal{L}_s}\mathrm{E}(\textrm{Residual regret}\mid \textup{Good},\rho')\\
&\leq \; \mathrm{O}(1)+ \sup_{\rho'\in \mathcal{L}_s}\bigg[\mathrm{P}(\textup{Revert}<\textup{Confirm}\mid \textup{Good},\rho')\textup{Reg}^{(1)}(i)\\
&~~+(1-\mathrm{P}(\textup{Revert}<\textup{Confirm}\mid \textup{Good},\rho'))\times\\
&~~\mathrm{E}(\inf\{k>0: \min_{i' \in \Str(i)}\lambda_k(i)/\lambda_k(i')\geq  N\})\big[\sum_{j\in \cS} \alpha_j\big(U(i)-[A(i,j)-p^*_j]\big)\big]\bigg].
\end{align*}
Hence, finally we have
\begin{align*}
&\sup_{\rho'\in \mathcal{L}_s}\mathrm{E}(\textrm{Residual regret}\mid \textup{Good},\rho')\leq \mathrm{O}(1)+w_s\textup{Reg}^{(1)}(i)\\
&~~+(1-w_s)\mathrm{E}(\inf\{k>0: \min_{i' \in \Str(i)}\lambda_k(i)/\lambda_k(i')\geq  N\})\big[\sum_{j\in \cS} \alpha_j\big(U(i)-[A(i,j)-p^*_j]\big)\big],
\end{align*}
for some $w_s\in (0,1)$ for $s=1,2$,  since $\sup_{\rho'\in \mathcal{L}_s}\mathrm{P}(\textup{Revert}<\textup{Confirm}\mid \textup{Good},\rho')<1$ and $\inf_{\rho'\in \mathcal{L}_s}\mathrm{P}(\textup{Revert}<\textup{Confirm}\mid \textup{Good},\rho')>0$.

Next, consider $\sup_{\rho'\in \mathcal{L}_s}\mathrm{E}(\textrm{Residual regret}\mid \textup{Bad},\rho')$. Now the residual regret after event $\textup{Bad}$ has occured depends on which of the following two events happens next:

\begin{enumerate}
\item Event $\textup{Revert}$:  $\min_{i'\neq i}\log(\lambda_k(i)/\lambda_k(i'))< \log\log N$  or 
\item Event $\textup{Bad-confirm}$:  Some $i'\neq i$ gets confirmed, i.e.,  $\min_{i''\in \Str(i')}\log(\lambda_k(i')/\lambda_k(i''))>\log N$. 
\end{enumerate}
Again, conditional on $\textup{Bad}$, one of the two events will occur with probability $1$. Let $K(i)$ be the maximum expected regret incurred till either $\textup{Revert}$ or $\textup{Bad-confirm}$ occurs given that $\textup{Bad}$ has occurred and the starting likelihoods were in $\mathcal{L}_s$. Note that if $\textup{Bad-confirm}<\textup{Revert}$ then the exploration phase ends and hence there is no residual regret (Although, note that if $i'$ is such that $i\in \Str(i')$, then $\mathrm{P}(\textup{Bad-confirm}<\textup{Revert}\mid\textup{Bad},\, \rho')\leq \mathrm{O}(1)/N$ from the second statement in Lemma~\ref{lma:core}.). Then we have
$$\sup_{\rho'\in \mathcal{L}_s}\mathrm{E}(\textrm{Residual regret}\mid \textup{Bad},\rho')  \leq K(i) + \sup_{\rho'\in \mathcal{L}_s}\mathrm{P}(\textup{Revert}<\textup{Bad-confirm}\mid\textup{Bad},\, \rho') \textup{Reg}^{(2)}(i).$$
We first show that if there is a type $i'$ such that $i \in \cC \setminus \Str(i')$, then $K(i)\leq \mathrm{O}(\log N)$, where as if there is no such type, then $K(i)=\mathrm{O}(1)$. Let $T(i)$ be the maximum expected time until either $\textup{Revert}$ or $\textup{Bad-confirm}$ occurs given that $\textup{Bad}$ has occurred and the starting likelihoods were in $\mathcal{L}_s$. Clearly $K(i)\leq T(i)$ since the price adjusted payoffs lie in $[0,1]$. Now, let $T_1$ be the time spent after $\textup{Bad}$ has occurred, before $\textup{Revert}$ or $\textup{Bad-confirm}$ occurs, while either a) algorithm is in the Guessing mode or b) the algorithm is in the Confirmation mode for some guessed type $i'$ such that $\alpha(i')_j>0$ for some $j$ such that $\KL(i,i'|j)>0$. Under this case, we will say that the algorithm is in state 1, and let $\mathbf{1}_k$ be the event that the algorithm is in state 1 at time $k$. Next let $T_2$ be the time spent after $\textup{Bad}$ has occurred, before $\textup{Revert}$ or $\textup{Bad-confirm}$ occurs, while the algorithm is in the Confirmation mode for some guessed type $i'$ such that $\alpha(i')_j=0$ for all $j$ such that $\KL(i,i'|j)>0$ (clearly this can happen only for $i'$ such that $i\in \cC \setminus \Str(i')$; thus if such an $i'$ doesn't exist, then $T_2=0$). Under this case, we will say that the algorithm is in state 2, and let $\mathbf{2}_k$ be the event that the algorithm is in state 2 at time $k$. Now we clearly have $T(i) \leq \sup_{\rho'\in \mathcal{L}_s} \textup{E}(T_1\mid \textup{Bad},\,\rho')+\sup_{\rho'\in \mathcal{L}_s} \textup{E}(T_2\mid \textup{Bad},\,\rho')$.

Let $\Gamma_k(i) \triangleq \min_{i'\neq i}\log(\lambda_k(i)/\lambda_k(i'))$. Then observe that $\textup{E}[(\Gamma_{k+1}(i)-\Gamma_{k}(i))\ind_{\mathbf{1}_k}\mid \textup{Bad},\,\rho']>\psi$ for some $\psi>0$ that depends only on the primitives of the problem and $\textup{E}[(\Gamma_{k+1}(i)-\Gamma_{k}(i))\ind_{\mathbf{2}_k}\mid \textup{Bad},\,\rho']=0$; i.e, when the algorithm is in state 1, the drift of $\Gamma_k(i)$ is strictly positive where as when the algorithm is in state 2, then $\Gamma_k(i)$ does not change.

Now consider $\textup{E}(T_1\mid \textup{Bad},\,\rho')$. Let $k^*$ be the opportunity that $\textup{Bad}$ occurred for the first time. Then clearly $\Gamma_{k^*}(i) = -\log\log N -\epsilon$, where $\epsilon\geq 0$ is such that $\epsilon <M'$, where $M'$ is a constant depending only on the problem instance. Thus $\textup{P}(T_1>t\mid \textup{Bad},\,\rho') \leq \textup{P}(\Gamma_{k'}(i)- \Gamma_{k^*}(i) \leq \epsilon\mid \textup{Bad},\,\rho')$ for some $k'>k^*$ such that the algorithm has been in state 1 exactly $t$ times at opportunity $k'$. Thus our observation above, in particular that $\psi>0$, implies by a standard concentration bound that $\textup{P}(T_1>t\mid \textup{Bad},\,\rho') \leq \exp(-ct)$ for some $c>0$. Thus $\textup{E}(T_1\mid \textup{Bad},\,\rho')= \mathrm{O}(1)$.

Next consider $\textup{E}(T_2\mid \textup{Bad},\,\rho')$.  Consider the successive returns of the algorithm to state 2. Conditional on the algorithm having entered state $2$, the expected time spent in that state is bounded by the expected time till the guessed type $i'$ is confirmed, which is $\mathrm{O}(\log N)$ from Lemma~\ref{lma:core}, and the conditional probability that $i'$ gets confirmed is some $q>0$. Thus the total expected number of returns to state $2$ is bounded by $1/q$. Thus $\textup{E}(T_2\mid \textup{Bad},\,\rho')= \mathrm{O}(\log N)$ as well. Thus $K(i)\leq \textup{O}(\log N)$ and we have
$$\sup_{\rho'\in \mathcal{L}_s}\mathrm{E}(\textrm{Residual regret}\mid \textup{Bad},\rho')  \leq \textup{O}(\log N) + \textup{Reg}^{(2)}(i).$$
And thus we finally have
\begin{align}
\textup{Reg}^{(1)}(i) &\leq \mathrm{O}(\log\log N)  + w_1\textup{Reg}^{(1)}(i) + \frac{1}{\log N}\big(\mathrm{O}(\log N+ \textup{Reg}^{(2)}(i)\big)\nonumber\\
&+(1-w_1)\mathrm{E}(\inf\{k>0: \min_{i' \in \Str(i)}\lambda_k(i)/\lambda_k(i')\geq  N\})\big[\sum_{j\in \cS} \alpha_j\big(U(i)-[A(i,j)-p^*_j]\big)\big]; \\
\textup{Reg}^{(2)}(i) &\leq \mathrm{O}(\log\log N) + w_2\textup{Reg}^{(1)}(i)+ w\mathrm{O}(\log N)+ w\textup{Reg}^{(2)}(i)\nonumber\\
&  +(1-w_2)\mathrm{E}(\inf\{k>0: \min_{i' \in \Str(i)}\lambda_k(i)/\lambda_k(i')\geq  N\})\big[\sum_{j\in \cS} \alpha_j\big(U(i)-[A(i,j)-p^*_j]\big)\big].
\end{align}

Combining the above two equations, we deduce that
\begin{align}
\textup{Reg}_{\textup{xplr}}(i) &\leq \textup{Reg}^{(1)}(i) \leq\frac{1-q_1}{1-q_1-q_2/\log N}\big( \mathrm{O}(\log\log N) \nonumber\\
&~~+ \mathrm{E}[\inf\{k>0: \min_{i' \in \Str(i)}\lambda_k(i)/\lambda_k(i')\geq  N\}]\sum_{j\in \cS}\alpha_j\big(U(i)-[A(i,j)-p^*_j]\big)\big)\nonumber\\
&= \mathrm{O}(\log\log N) \nonumber \\
&~~+ (1+\textup{O}(\frac{1}{\log N}))\mathrm{E}[\inf\{k>0: \min_{i' \in \Str(i)}\lambda_k(i)/\lambda_k(i')\geq  N\}]\sum_{j\in \cS}\alpha_j\big(U(i)-[A(i,j)-p^*_j]\big).
\end{align}
Now, we observed earlier that $\mathrm{P}(\textrm{$i'$ gets confirmed }\mid X=i)\leq \textup{O}(1)/N$ if $i' \in \Str(i)$. Thus the regret in the exploitation phase is in the worst case of order $\mathrm{O}(N)$ with probability $1/N$ and $0$ otherwise. Thus the total expected regret in the exploitation phase is $\mathrm{O}(1)$. Thus
\begin{align*}
\textup{Reg}(i)\leq \; &\mathrm{O}(\log\log N) \; + \\
&(1+\mathrm{O}(\frac{1}{\log N}))\mathrm{E}[\inf\{k>0: \min_{i' \in \Str(i)}\lambda_k(i)/\lambda_k(i')\geq  N\}]\sum_{j\in \cS}\alpha_j\big(U(i)-[A(i,j)-p^*_j]\big).  
\end{align*} 
Thus Lemma~\ref{lma:core} implies the result. (Note that if there are no difficult type pairs, then $\sum_{j\in \cS}\alpha_j\big(U(i)-[A(i,j)-p^*_j]\big)=0$.)




\end{proof}

%% file: Paper_Final.bbl
\begin{thebibliography}{52}
\providecommand{\natexlab}[1]{#1}
\providecommand{\url}[1]{\texttt{#1}}
\providecommand{\urlprefix}{URL }

\bibitem[{Agrawal et~al.(1989)Agrawal, Teneketzis, \protect\BIBand{}
  Anantharam}]{agrawal1989asymptotically}
Agrawal R, Teneketzis D, Anantharam V (1989) Asymptotically efficient adaptive
  allocation schemes for controlled iid processes: finite parameter space.
  \emph{Automatic Control, IEEE Transactions on} 34(3):258--267.

\bibitem[{Agrawal \protect\BIBand{} Devanur(2014)}]{agrawal2014bandits}
Agrawal S, Devanur NR (2014) Bandits with concave rewards and convex knapsacks.
  \emph{Proceedings of the fifteenth ACM conference on Economics and
  computation}, 989--1006 (ACM).

\bibitem[{Agrawal \protect\BIBand{} Devanur(2015)}]{agrawal2015linear}
Agrawal S, Devanur NR (2015) Linear contextual bandits with global constraints
  and objective. \emph{arXiv preprint arXiv:1507.06738} .

\bibitem[{Agrawal et~al.(2015)Agrawal, Devanur, \protect\BIBand{}
  Li}]{agrawal2015contextual}
Agrawal S, Devanur NR, Li L (2015) Contextual bandits with global constraints
  and objective. \emph{arXiv preprint arXiv:1506.03374} .

\bibitem[{Agrawal \protect\BIBand{} Goyal(2011)}]{agrawal2011analysis}
Agrawal S, Goyal N (2011) Analysis of thompson sampling for the multi-armed
  bandit problem. \emph{arXiv preprint arXiv:1111.1797} .

\bibitem[{Agrawal \protect\BIBand{} Goyal(2012)}]{agrawal2012analysis}
Agrawal S, Goyal N (2012) Analysis of thompson sampling for the multi-armed
  bandit problem. \emph{Conference on Learning Theory}, 39--1.

\bibitem[{Akbarpour et~al.(2014)Akbarpour, Li, \protect\BIBand{}
  Oveis~Gharan}]{akbarpour2014dynamic}
Akbarpour M, Li S, Oveis~Gharan S (2014) Dynamic matching market design.
  \emph{Available at SSRN 2394319} .

\bibitem[{Anderson et~al.(2015)Anderson, Ashlagi, Gamarnik, \protect\BIBand{}
  Kanoria}]{anderson2015dynamic}
Anderson R, Ashlagi I, Gamarnik D, Kanoria Y (2015) A dynamic model of barter
  exchange. \emph{Proceedings of the Twenty-Sixth Annual ACM-SIAM Symposium on
  Discrete Algorithms}, 1925--1933 (SIAM).

\bibitem[{Ata \protect\BIBand{} Kumar(2005)}]{ata2005heavy}
Ata B, Kumar S (2005) Heavy traffic analysis of open processing networks with
  complete resource pooling: asymptotic optimality of discrete review policies.
  \emph{The Annals of Applied Probability} 15(1A):331--391.

\bibitem[{Audibert \protect\BIBand{} Munos(2011)}]{icml_tutorial}
Audibert JY, Munos R (2011) Introduction to bandits: Algorithms and theory.
  \emph{{ICML}}.

\bibitem[{Auer et~al.(2002)Auer, Cesa-Bianchi, \protect\BIBand{}
  Fischer}]{auer2002finite}
Auer P, Cesa-Bianchi N, Fischer P (2002) Finite-time analysis of the multiarmed
  bandit problem. \emph{Machine learning} 47(2-3):235--256.

\bibitem[{Babaioff et~al.(2015)Babaioff, Dughmi, Kleinberg, \protect\BIBand{}
  Slivkins}]{babaioff2015dynamic}
Babaioff M, Dughmi S, Kleinberg R, Slivkins A (2015) Dynamic pricing with
  limited supply. \emph{ACM Transactions on Economics and Computation} 3(1):4.

\bibitem[{Baccara et~al.(2015)Baccara, Lee, \protect\BIBand{}
  Yariv}]{baccara2015optimal}
Baccara M, Lee S, Yariv L (2015) Optimal dynamic matching. \emph{Available at
  SSRN 2641670} .

\bibitem[{Badanidiyuru et~al.(2013)Badanidiyuru, Kleinberg, \protect\BIBand{}
  Slivkins}]{badanidiyuru2013bandits}
Badanidiyuru A, Kleinberg R, Slivkins A (2013) Bandits with knapsacks.
  \emph{Foundations of Computer Science (FOCS), 2013 IEEE 54th Annual Symposium
  on}, 207--216 (IEEE).

\bibitem[{Badanidiyuru et~al.(2014)Badanidiyuru, Langford, \protect\BIBand{}
  Slivkins}]{badanidiyuru2014resourceful}
Badanidiyuru A, Langford J, Slivkins A (2014) Resourceful contextual bandits.
  \emph{Proceedings of The 27th Conference on Learning Theory}, 1109--1134.

\bibitem[{Besbes \protect\BIBand{} Zeevi(2009)}]{besbes2009dynamic}
Besbes O, Zeevi A (2009) Dynamic pricing without knowing the demand function:
  Risk bounds and near-optimal algorithms. \emph{Operations Research}
  57(6):1407--1420.

\bibitem[{Besbes \protect\BIBand{} Zeevi(2012)}]{besbes2012blind}
Besbes O, Zeevi A (2012) Blind network revenue management. \emph{Operations
  research} 60(6):1537--1550.

\bibitem[{Blischke(1964)}]{blischke1964estimating}
Blischke W (1964) Estimating the parameters of mixtures of binomial
  distributions. \emph{Journal of the American Statistical Association}
  59(306):510--528.

\bibitem[{Bubeck \protect\BIBand{} Cesa-Bianchi(2012)}]{bubeck2012regret}
Bubeck S, Cesa-Bianchi N (2012) Regret analysis of stochastic and nonstochastic
  multi-armed bandit problems. \emph{Machine Learning} 5(1):1--122.

\bibitem[{Chakrabarti et~al.(2009)Chakrabarti, Kumar, Radlinski,
  \protect\BIBand{} Upfal}]{chakrabarti2009mortal}
Chakrabarti D, Kumar R, Radlinski F, Upfal E (2009) Mortal multi-armed bandits.
  \emph{Advances in neural information processing systems}, 273--280.

\bibitem[{Chen et~al.(2013)Chen, Wang, \protect\BIBand{}
  Yuan}]{chen2013combinatorial}
Chen W, Wang Y, Yuan Y (2013) Combinatorial multi-armed bandit: General
  framework and applications. \emph{International Conference on Machine
  Learning}, 151--159.

\bibitem[{Damiano \protect\BIBand{} Lam(2005)}]{damiano2005stability}
Damiano E, Lam R (2005) Stability in dynamic matching markets. \emph{Games and
  Economic Behavior} 52(1):34--53.

\bibitem[{Das \protect\BIBand{} Kamenica(2005)}]{das2005two}
Das S, Kamenica E (2005) Two-sided bandits and the dating market.
  \emph{Proceedings of the 19th international joint conference on Artificial
  intelligence}, 947--952 (Morgan Kaufmann Publishers Inc.).

\bibitem[{den Boer(2015)}]{den2015dynamic2}
den Boer AV (2015) Dynamic pricing and learning: historical origins, current
  research, and new directions. \emph{Surveys in operations research and
  management science} 20(1):1--18.

\bibitem[{den Boer \protect\BIBand{} Zwart(2014)}]{den2014simultaneously}
den Boer AV, Zwart B (2014) Simultaneously learning and optimizing using
  controlled variance pricing. \emph{Management science} 60(3):770--783.

\bibitem[{den Boer \protect\BIBand{} Zwart(2015)}]{den2015dynamic}
den Boer AV, Zwart B (2015) Dynamic pricing and learning with finite
  inventories. \emph{Operations research} 63(4):965--978.

\bibitem[{Feldman et~al.(2008)Feldman, O'Donnell, \protect\BIBand{}
  Servedio}]{feldman2008learning}
Feldman J, O'Donnell R, Servedio RA (2008) Learning mixtures of product
  distributions over discrete domains. \emph{SIAM Journal on Computing}
  37(5):1536--1564.

\bibitem[{Ferreira et~al.(2018)Ferreira, Simchi-Levi, \protect\BIBand{}
  Wang}]{ferreira2018online}
Ferreira KJ, Simchi-Levi D, Wang H (2018) Online network revenue management
  using thompson sampling. \emph{Operations research} 66(6):1586--1602.

\bibitem[{Fershtman \protect\BIBand{} Pavan(2015)}]{fershtman2015dynamic}
Fershtman D, Pavan A (2015) Dynamic matching: experimentation and cross
  subsidization. Technical report, Citeseer.

\bibitem[{Gai et~al.(2010)Gai, Krishnamachari, \protect\BIBand{}
  Jain}]{gai2010learning}
Gai Y, Krishnamachari B, Jain R (2010) Learning multiuser channel allocations
  in cognitive radio networks: A combinatorial multi-armed bandit formulation.
  \emph{New Frontiers in Dynamic Spectrum, 2010 IEEE Symposium on}, 1--9
  (IEEE).

\bibitem[{Gai et~al.(2012)Gai, Krishnamachari, \protect\BIBand{}
  Jain}]{gai2012combinatorial}
Gai Y, Krishnamachari B, Jain R (2012) Combinatorial network optimization with
  unknown variables: Multi-armed bandits with linear rewards and individual
  observations. \emph{IEEE/ACM Transactions on Networking (TON)}
  20(5):1466--1478.

\bibitem[{Gittins et~al.(2011)Gittins, Glazebrook, \protect\BIBand{}
  Weber}]{gittins2011multi}
Gittins J, Glazebrook K, Weber R (2011) \emph{Multi-armed bandit allocation
  indices} (John Wiley \& Sons).

\bibitem[{Hsu et~al.(2018)Hsu, Xu, Lin, \protect\BIBand{}
  Bell}]{hsu2018integrating}
Hsu WK, Xu J, Lin X, Bell MR (2018) Integrating online learning and adaptive
  control in queueing systems with uncertain payoffs. \emph{2018 Information
  Theory and Applications Workshop (ITA)}, 1--9 (IEEE).

\bibitem[{Hu \protect\BIBand{} Zhou(2015)}]{hu2015dynamic}
Hu M, Zhou Y (2015) Dynamic matching in a two-sided market. \emph{Available at
  SSRN} .

\bibitem[{Kadam \protect\BIBand{} Kotowski(2015)}]{kadam2015multi}
Kadam SV, Kotowski MH (2015) Multi-period matching. Technical report, Harvard
  University, John F. Kennedy School of Government.

\bibitem[{Kurino(2020)}]{kurino2020credibility}
Kurino M (2020) Credibility, efficiency, and stability: A theory of dynamic
  matching markets. \emph{The Japanese Economic Review} 71(1):135--165.

\bibitem[{Kveton et~al.(2015)Kveton, Wen, Ashkan, \protect\BIBand{}
  Szepesvari}]{kveton2015tight}
Kveton B, Wen Z, Ashkan A, Szepesvari C (2015) Tight regret bounds for
  stochastic combinatorial semi-bandits. \emph{Artificial Intelligence and
  Statistics}, 535--543.

\bibitem[{Lai \protect\BIBand{} Robbins(1985)}]{lai1985asymptotically}
Lai TL, Robbins H (1985) Asymptotically efficient adaptive allocation rules.
  \emph{Advances in applied mathematics} 6(1):4--22.

\bibitem[{Liu \protect\BIBand{} Zhao(2012)}]{liu2012adaptive}
Liu K, Zhao Q (2012) Adaptive shortest-path routing under unknown and
  stochastically varying link states. \emph{Modeling and Optimization in
  Mobile, Ad Hoc and Wireless Networks (WiOpt), 2012 10th International
  Symposium on}, 232--237 (IEEE).

\bibitem[{Massouli{\'e} \protect\BIBand{} Xu(2018)}]{massoulie2018capacity}
Massouli{\'e} L, Xu K (2018) On the capacity of information processing systems.
  \emph{Operations Research} 66(2):568--586.

\bibitem[{Mehta(2012)}]{mehta2012online}
Mehta A (2012) Online matching and ad allocation. \emph{Theoretical Computer
  Science} 8(4):265--368.

\bibitem[{Modaresi et~al.(2019)Modaresi, Saure, \protect\BIBand{}
  Vielma}]{modaresi2019learning}
Modaresi S, Saure D, Vielma JP (2019) Learning in combinatorial optimization:
  What and how to explore. \emph{Available at SSRN 3041893} .

\bibitem[{Ozkan \protect\BIBand{} Ward(2017)}]{ozkan2017dynamic}
Ozkan E, Ward A (2017) Dynamic matching for real-time ridesharing.
  \emph{Available at SSRN 2844451} .

\bibitem[{Ross(2008)}]{ross2008stochastic}
Ross S (2008) \emph{Stochastic Processes, 2nd Ed.} Wiley series in probability
  and statistics (Wiley India Pvt. Limited), ISBN 9788126517572,
  \urlprefix\url{https://books.google.com/books?id=HVHqPgAACAAJ}.

\bibitem[{Russo \protect\BIBand{} Van~Roy(2016)}]{russo2016information}
Russo D, Van~Roy B (2016) An information-theoretic analysis of thompson
  sampling. \emph{The Journal of Machine Learning Research} 17(1):2442--2471.

\bibitem[{Saur{\'e} \protect\BIBand{} Zeevi(2013)}]{saure2013optimal}
Saur{\'e} D, Zeevi A (2013) Optimal dynamic assortment planning with demand
  learning. \emph{Manufacturing \& Service Operations Management}
  15(3):387--404.

\bibitem[{Shakkottai et~al.(2008)Shakkottai, Srikant
  et~al.}]{shakkottai2008network}
Shakkottai S, Srikant R, et~al. (2008) Network optimization and control.
  \emph{Foundations and Trends in Networking} 2(3):271--379.

\bibitem[{Shapley \protect\BIBand{} Shubik(1971)}]{shapley1971assignment}
Shapley LS, Shubik M (1971) The assignment game i: The core.
  \emph{International Journal of game theory} 1(1):111--130.

\bibitem[{Srikant(2012)}]{srikant2012mathematics}
Srikant R (2012) \emph{The mathematics of Internet congestion control}
  (Springer Science \& Business Media).

\bibitem[{Sun(2006)}]{sun2006exact}
Sun Y (2006) The exact law of large numbers via fubini extension and
  characterization of insurable risks. \emph{Journal of Economic Theory}
  126(1):31--69.

\bibitem[{Tassiulas \protect\BIBand{}
  Ephremides(1990)}]{tassiulas1990stability}
Tassiulas L, Ephremides A (1990) Stability properties of constrained queueing
  systems and scheduling policies for maximum throughput in multihop radio
  networks. \emph{29th IEEE Conference on Decision and Control}, 2130--2132
  (IEEE).

\bibitem[{Wang et~al.(2014)Wang, Deng, \protect\BIBand{} Ye}]{wang2014close}
Wang Z, Deng S, Ye Y (2014) Close the gaps: A learning-while-doing algorithm
  for single-product revenue management problems. \emph{Operations Research}
  62(2):318--331.

\end{thebibliography}
